\documentclass[twoside,11pt]{article}

\usepackage{epsfig}
\usepackage{amssymb}
\usepackage{natbib}
\usepackage{hyperref}
\usepackage{graphicx}
\usepackage{amsfonts,amstext,amsmath,amsthm}
\usepackage{accents}
\usepackage{color}
\usepackage[utf8]{inputenc}
\usepackage{booktabs}
\usepackage{tikz}
\usepackage{multirow}
\usepackage{float}
\usepackage[labelfont=bf, width=15cm, font=small]{caption}
\usepackage[onehalfspacing]{setspace}
\usepackage[left=3cm, right=3cm, top=2.5cm, bottom=3cm]{geometry}
\setcitestyle{authoryear,open={(},close={)}}
\newcommand*{\doi}[1]{\href{http://dx.doi.org/#1}{doi: #1}}

\usepackage{amsfonts,amssymb,amsthm}
\usepackage{mathtools}
\usepackage{dsfont}
\usepackage{capt-of}
\usepackage{url}
\usepackage{xspace}
\usepackage{placeins}
\usepackage{algorithm}
\usepackage{algpseudocode}
\usepackage{authblk}
\usepackage[page]{appendix} 
\usepackage{tabularx}
\usepackage{longtable}

\usetikzlibrary{shapes.geometric}
\usetikzlibrary{backgrounds}
\usetikzlibrary{arrows}
\usetikzlibrary{arrows.meta}

\definecolor{arrowblue}{RGB}{192,192,192}
\definecolor{dimgrey}{RGB}{105,105,105}

\usepackage[framemethod=tikz]{mdframed}
\newenvironment{mybox}[2]
  {\mdfsetup{
    frametitle={\colorbox{white}{\space#1\space}},
    innertopmargin=10pt,
    frametitleaboveskip=-\ht\strutbox,
    linewidth=1.5pt,
    linecolor=#2,
    nobreak
    }
  \begin{mdframed}%
  }
  {\end{mdframed}}

\usepackage{pgfplots}
\DeclareUnicodeCharacter{2212}{−}
\usepgfplotslibrary{groupplots,dateplot}
\usetikzlibrary{patterns,shapes.arrows}
\pgfplotsset{compat=newest}

\pgfplotsset{compat=1.11,
        /pgfplots/ybar legend/.style={
        /pgfplots/legend image code/.code={%
        \draw [#1] (0cm,-0.1cm) rectangle (0.1cm,0.1cm);}
},}

\usetikzlibrary{cd,arrows,calc,shapes,positioning}
\tikzstyle{obs} = [circle,fill=white,draw=black,inner sep=1pt,minimum size=20pt,font=\fontsize{10}{10}\selectfont,node distance=1,thick]
\tikzstyle{latent} = [obs,dotted]

\newcommand{\bR}{\ensuremath \mathbb{R}}
\newcommand{\bN}{\ensuremath \mathbb{N}}

\newcommand{\bS}{\ensuremath \mathbb{S}}

\newcommand{\cX}{\ensuremath \mathcal{X}}

\newcommand{\cH}{\ensuremath \mathcal{H}}

\newcommand{\bE}{\ensuremath \mathbb{E}}

\DeclareMathOperator*{\argmin}{arg\,min}
\DeclareMathOperator{\E}{\mathbb{E}}

\newcommand{\simp}{\ensuremath \mathbb{S}}

\newcommand{\iid}{\stackrel{iid}{\sim}}

\newcommand{\KernelBiome}{\texttt{KernelBiome}\xspace}

\theoremstyle{definition}
\newtheorem{definition}{Definition}[section]
\newtheorem{proposition}{Proposition}[section]
\newtheorem{theorem}{Theorem}[section]

\newtheorem{lemma}{Lemma}[section]
\newtheorem{example}{Example}[section]

\begin{document}

\title{\bf Supervised Learning and Model Analysis with Compositional Data}
\author[1]{Shimeng Huang}
\author[2]{Elisabeth Ailer}
\author[2,3]{Niki Kilbertus}
\author[1]{Niklas Pfister}
\affil[1]{Department of Mathematical Sciences, University of Copenhagen, Copenhagen, Denmark}
\affil[2]{Helmholtz Munich, Munich, Germany}
\affil[3]{Technical University of Munich, Munich, Germany}

\date{\today}

\maketitle

\begin{abstract}%
  Supervised learning, such as regression and classification, is an essential tool for analyzing modern high-throughput sequencing data, for example in microbiome research. However, due to the compositionality and sparsity, existing techniques are often inadequate. Either they rely on extensions of the linear log-contrast model (which adjust for compositionality but cannot account for complex signals or sparsity) or they are based on black-box machine learning methods (which may capture useful signals, but lack interpretability due to the compositionality). We propose \KernelBiome, a kernel-based nonparametric regression and classification framework for compositional data. It is tailored to sparse compositional data and is able to incorporate prior knowledge, such as phylogenetic structure.

\KernelBiome captures complex signals, including in the zero-structure, while automatically adapting model complexity. We demonstrate on par or improved predictive performance compared with state-of-the-art machine learning methods on $33$ publicly available microbiome datasets. Additionally, our framework provides two key advantages: (i) We propose two novel quantities to interpret contributions of individual components and prove that they consistently estimate average perturbation effects of the conditional mean, extending the interpretability of linear log-contrast coefficients to nonparametric models. (ii) We show that the connection between kernels and distances aids interpretability and provides a data-driven embedding that can augment further analysis. \KernelBiome is available as an open-source Python package at \url{https://github.com/shimenghuang/KernelBiome}.

\end{abstract}


\section{Introduction}

\noindent Compositional data, that is, measurements of parts of a whole, are common in many scientific disciplines. For example, mineral
compositions in geology \citep{buccianti2006compositional},
element concentrations in chemistry \citep{pesenson2015statistical},
species compositions in ecology \citep{jackson1997compositional} and more recently high-throughput sequencing reads in microbiome science \citep{li2015microbiome}.

Mathematically, any $p$-dimensional composition---by appropriate normalization---can be
represented as a point on the simplex
$$\simp^{p-1}\coloneqq\{x\in [0,1]^p\mid \textstyle\sum_{j=1}^px^j=1\}.$$

This complicates the statistical analysis, because the sum-to-one constraint of the simplex induces non-trivial dependencies between the components that may lead to false conclusions, if not appropriately taken into account.

The statistics community has developed a substantial collection of parametric analysis techniques to account for the simplex-structure.
The most basic is the family of Dirichlet distributions. However, as pointed out already by \citet{aitchison1982statistical},
Dirichlet distributions cannot capture non-trivial dependence
structures between the composition
components and are thus too restrictive.
\citet{aitchison1982statistical} therefore introduced the
\emph{log-ratio} approach. It generates a family of distributions by
projecting multivariate normal distributions into $\simp^{p-1}$ via an appropriate log-ratio
transformation (e.g., the additive log-ratio, centered log-ratio
\citep{aitchison1982statistical}, or isometric log-ratio
\citep{egozcue2003isometric}). The resulting family of distributions
results in parametric models on the simplex that are rich enough to
capture non-trivial dependencies between the components (i.e., beyond those induced by the sum-to-one constraint). The log-ratio approach has been extended and adapted to a range of statistical problems \citep[e.g.,][]{aitchison1985general,tsagris2011data, aitchison1983pca, aitchison2002biplot,friedman2012inferring}.

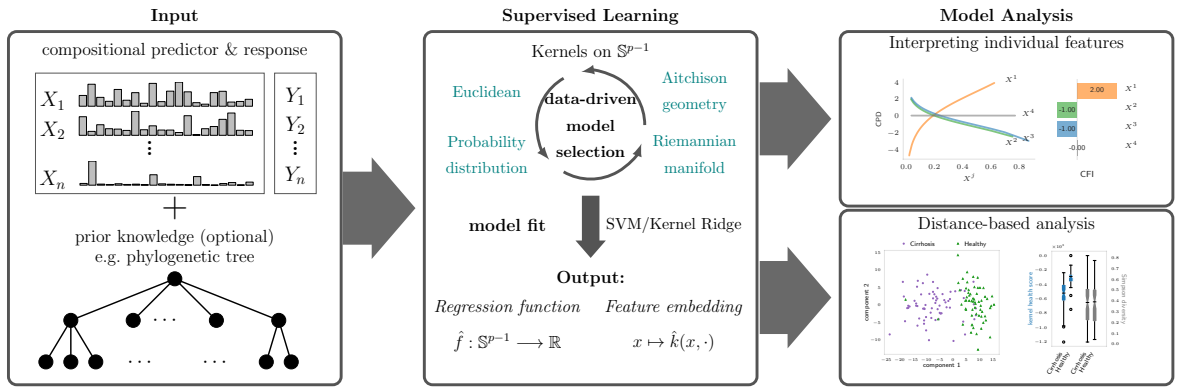
\begin{figure*}[t]
\resizebox{\linewidth}{!}{
\begin{tikzpicture}
  \tikzstyle{every node}=[font=\Large]
  \node at (1,11.8) (Atext){\large \textbf{Input}};
  \node[draw=dimgrey!80!black, minimum width=8cm, minimum height=8.5cm,rounded corners=0.2cm,
  line width=2pt]
  at (1,7.2) (A){};
  \node at (1,11) (Atext){\large compositional predictor \& response};
  \node[draw=dimgrey!80!black, minimum width=5.5cm, minimum height=2.9cm,rounded corners=0cm,
  line width=0.5pt]
  at (0.4,9) (AA){};
  \node at (0.4,9.9) (Apic1)
  {$X_1$ \includegraphics[width=4.5cm]{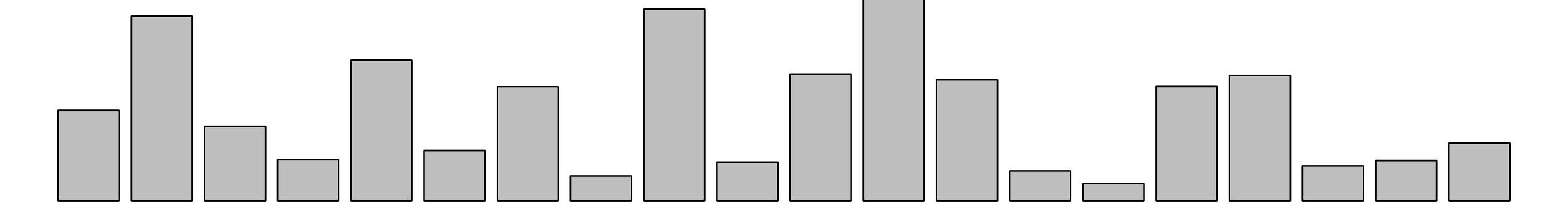}};
  \node at (0.4,9.2) (Apic2)
  {$X_2$ \includegraphics[width=4.5cm]{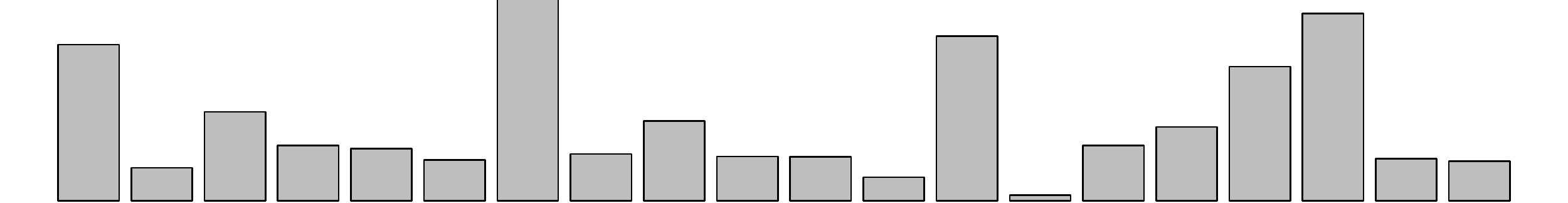}};
  \node at (0.4,8.75) (Apic3)
  {\Huge $\vdots$};
  \node at (0.4,8) (Apic4)
  {$X_n$ \includegraphics[width=4.5cm]{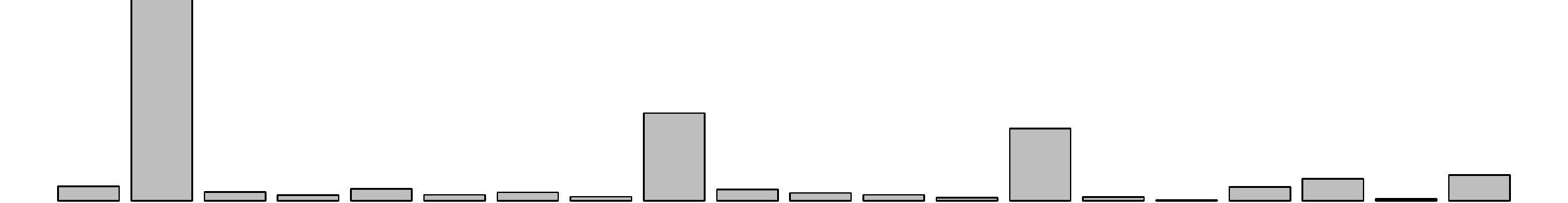}};
  \node[draw=dimgrey!80!black, minimum width=1cm, minimum height=2.9cm,rounded corners=0cm,
  line width=0.5pt]
  at (3.9, 9) (AAA){};
  \node at (3.9,9.9) (AApic1)
  {$Y_1$};
  \node at (3.9,9.2) (AApic2)
  {$Y_2$};
  \node at (3.9,8.75) (AApic3)
  {\Huge $\vdots$};
  \node at (3.9,8) (AApic4)
  {$Y_n$};
  \node at (1,7.2) (Atext){\Huge $+$};
  \node at (1,6.5) (Atext){\large prior knowledge (optional)};
  \node at (1,6) (Atext){\large e.g.\ phylogenetic tree};
  \node[draw, circle, fill=black, inner sep=0, minimum width=1pt] at (1,5.5) (P00){$\bullet$};
  \node[draw, circle, fill=black, inner sep=0, minimum width=1pt] at (-1.5,4.5) (P10){$\bullet$};
  \node[draw, circle, fill=black, inner sep=0, minimum width=1pt] at (0,4.5) (P11){$\bullet$};
  \node at (0.8,4.5) (P12){$\cdots$};
  \node[draw, circle, fill=black, inner sep=0, minimum width=1pt] at (2,4.5) (P13){$\bullet$};
  \node[draw, circle, fill=black, inner sep=0, minimum width=1pt] at
  (3.5,4.5) (P14){$\bullet$};
  \node[draw, circle, fill=black, inner sep=0, minimum width=1pt] at (-2.1,3.5) (P20){$\bullet$};
  \node[draw, circle, fill=black, inner sep=0, minimum width=1pt] at (-1.5,3.5) (P21){$\bullet$};
  \node[draw, circle, fill=black, inner sep=0, minimum width=1pt] at (-0.9,3.5) (P22){$\bullet$};
  \node at (0.2,3.5) (P23){$\cdots$};
  \node at (2,3.5) (P24){$\cdots$};
  \node[draw, circle, fill=black, inner sep=0, minimum width=1pt] at (3.2,3.5) (P25){$\bullet$};
  \node[draw, circle, fill=black, inner sep=0, minimum width=1pt] at (3.8,3.5) (P26){$\bullet$};
  \draw[line width=1pt, shorten <= -6pt, shorten >= -6pt] (P00) --
  (P10);
  \draw[line width=1pt, shorten <= -6pt, shorten >= -6pt] (P00) --
  (P11);
  \draw[line width=1pt, shorten <= -6pt, shorten >= -6pt] (P00) --
  (P13);
  \draw[line width=1pt, shorten <= -6pt, shorten >= -6pt] (P00) --
  (P14);
  \draw[line width=1pt, shorten <= -6pt, shorten >= -6pt] (P10) --
  (P20);
  \draw[line width=1pt, shorten <= -6pt, shorten >= -6pt] (P10) --
  (P21);
  \draw[line width=1pt, shorten <= -6pt, shorten >= -6pt] (P10) --
  (P22);
  \draw[line width=1pt, shorten <= -6pt, shorten >= -6pt] (P14) --
  (P25);
  \draw[line width=1pt, shorten <= -6pt, shorten >= -6pt] (P14) --
  (P26);

  \node at (11,11.8) (Atext){\large \textbf{Supervised Learning}};
  \node[draw=dimgrey!80!black, minimum width=8cm, minimum height=8.5cm,rounded corners=0.2cm,
  line width=2pt]
  at (11,7.2) (B){};

  \node[align=center] at (11, 9.2) (textB1){
    \normalsize \textbf{data-driven}\\
    \normalsize \textbf{model} \\
    \normalsize \textbf{selection}};
  \draw[dimgrey!80!black, line width=2pt,-latex] (11, 9.2) + (1.3, 0) arc[radius=1.3,start
  angle=0,delta angle=100];
  \draw[dimgrey!80!black, line width=2pt,-latex] (11, 9.2) + (-0.65, 1.125833) arc[radius=1.3,start
  angle=120,delta angle=100];
  \draw[dimgrey!80!black, line width=2pt, -latex] (11, 9.2) + (-0.65, -1.125833) arc[radius=1.3,start
  angle=240,delta angle=100];

  \node at (11, 11) (topname) {\large Kernels on $\mathbb{S}^{p-1}$};
  \node[align=center] at (8.5, 10) (method1) {\normalsize \color{teal}{Euclidean}};
  \node[align=center] at (13.5, 10) (method2) {\normalsize \color{teal}{Aitchison}\\
    \normalsize \color{teal}{geometry}};
  \node[align=center] at (8.5, 8.5) (method3) {\normalsize \color{teal}{Probability}\\
    \normalsize \color{teal}{distribution}};
  \node[align=center] at (13.5, 8.5) (method4) {\normalsize \color{teal}{Riemannian}\\
    \normalsize \color{teal}{manifold}};

  \draw[dimgrey!80!black, line width=12pt, -{Triangle[width=20pt,length=12pt]}] (11, 7.5) -- (11, 6);

  \node at (9, 6.8) {\large \textbf{model fit}};
  \node[align=center] at (13, 6.8) {\normalsize SVM/Kernel Ridge};

  \node at (11, 5.5) {\large \textbf{Output:}};
  \node at (9, 4.8) {\large \emph{Regression function}};
  \node at (13, 4.8) {\large \emph{Feature embedding}};
  \node at (9, 4) {\large
    $\hat{f}:\mathbb{S}^{p-1}\longrightarrow\mathbb{R}$};
  \node at (13, 4) {\large $x\mapsto \hat{k}(x,\cdot)$};

  \node at (21,11.8) (Atext){\large \textbf{Model Analysis}};
  \node[draw=dimgrey!80!black, minimum width=8cm, minimum height=4.2cm,rounded corners=0.2cm,
  line width=2pt]
  at (21, 9.35) (C1){};
  \node at (21, 11.15) (textB1){
    \large Interpreting individual features};
  \node at (21, 9.1) (CFIplot)
  {\includegraphics[width=6.7cm]{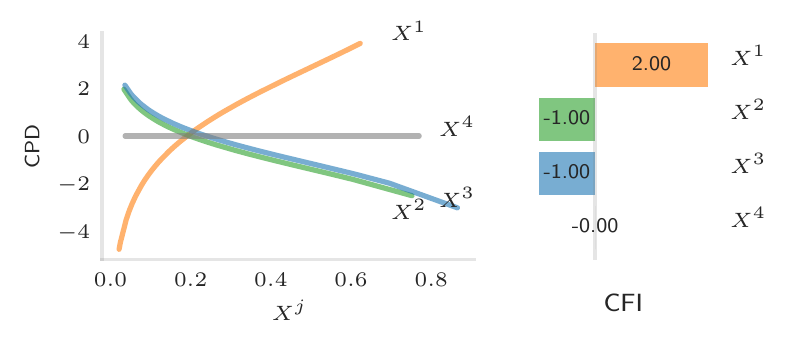}};

  \node[draw=dimgrey!80!black, minimum width=8cm, minimum height=4.2cm,rounded corners=0.2cm,
  line width=2pt]
  at (21, 5.05) (C2){};
  \node at (21, 6.8) (textB1){
    \large Distance-based analysis};
  \node at (19.2, 4.9) (kPCAplot)
  {\includegraphics[height=3.5cm]{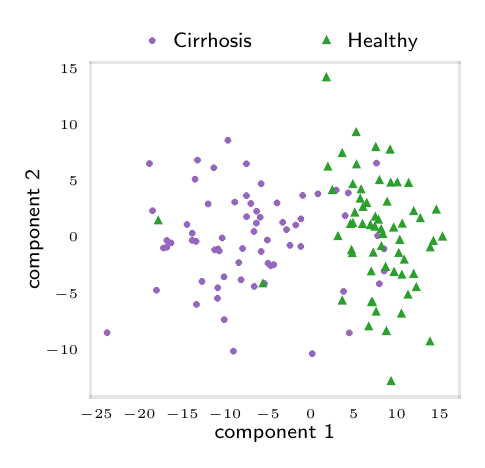}}; 
  \node at (22.7, 4.75) (Distplot)
  {\includegraphics[width=2.75cm]{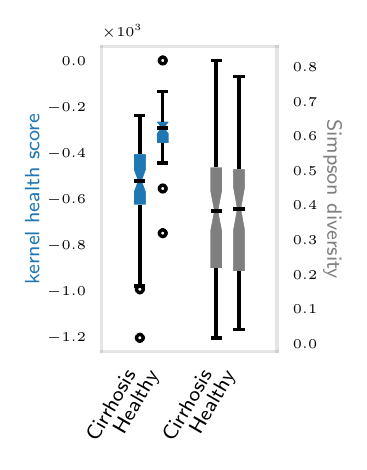}};

  \draw[-{Triangle[width=68pt,length=30pt]}, line width=50pt,
  color=dimgrey, shorten >=5pt] (A) -- (B);
  \node at (14.93, 9.35) (B1) {};
  \node at (14.93, 5.05) (B2) {};
  \draw[-{Triangle[width=68pt,length=30pt]}, line width=50pt,
  color=dimgrey, shorten >=5pt] (B1) -- (C1);
  \draw[-{Triangle[width=68pt,length=30pt]}, line width=50pt,
  color=dimgrey, shorten >=5pt] (B2) -- (C2);
  
\end{tikzpicture}%
}
\caption{Overview of \KernelBiome. We start from a paired dataset with a compositional predictor $X$ and a response $Y$ and optional prior knowledge on the relation between components in the compositions (e.g., via a phylogenetic tree). We then select a model among a large class of kernels which best fits the data. This results in an estimated model $\hat{f}$ and embedding $\hat{k}$. Finally, these can be 
analyzed while accounting for the compositional structure.}\label{fig:workflow}
\end{figure*}

For supervised learning tasks the log-ratio approach leads to the \emph{log-contrast model} \citep{Aitchison1984}.
An attractive property of the log-contrast model is that its coefficients quantify the effect of a multiplicative perturbation (i.e., fractionally increasing one component while adjusting the others) on the response.
While several extensions of the log-contrast model exist \citep[e.g.,][]{10.1093/biomet/asu031,shi2016, Combettes2020, simpson2021classo, ailer2021causal}, its parametric approach to supervised learning has two major shortcomings that become particularly severe when applied to high-dimensional and zero-inflated high-throughput sequencing data \citep{tsilimigras2016compositional,gloor2017microbiome}.
Firstly, since the logarithm is not defined at zero, the log-contrast model cannot be directly applied. A common fix is to add so-called pseudo-counts,
a small non-zero constant, to all (zero) entries
\citep{kaul2017analysis,lin2020analysis}. More sophisticated replacements exist as well \citep[e.g.,][]{martin2003dealing, fernandes2013anova, de2018geometric}, however, they often rely on knowing the nature of the zeros (e.g., whether they are structural or random), which is typically not available in practice and difficult to estimate. In any case, the downstream analysis will strongly depend on the selected zero imputation scheme \citep{park2022kernel}. Secondly, the relationships between individual components (e.g., species) and the response are generally complex. For example, in human microbiome settings, a health outcome may depend on interactions or on the presence or absence of species. Both cannot be captured by the linear structure of the log-contrast model.

We propose to solve the supervised learning task using a nonparametric kernel approach, which is able to handle complex signals and avoid arbitrary zero-imputation. To be of use in biological applications, there are two components to a supervised analysis: (i) estimating a predictive model that accurately captures signals in the data and (ii) extracting meaningful and interpretable information from the estimated model. For (i), it has been shown that modern machine learning methods are capable of creating highly predictive models by using microbiome data as covariates and phenotypes as responses \citep[e.g.,][]{knight2018best, zhou2019review, cammarota2020gut}. In particular, several approaches have been proposed where kernels are used to incorporate prior information \citep{chen2013kernel, randolph2018kernel} or as a way to utilize the compositional structure \citep{ramon2021kernint, di2015non, tsagris2021compositional}. Recently, \citet{park2022kernel, li2022reproducing} used the radial transformation to argue that kernels on the sphere provide a natural way of analyzing compositions with zeros and similar to our work suggest using the kernel embeddings in a subsequent analysis.
Part (ii) is related to the fields of explainable AI \citep{samek2019explainable} and interpretable machine learning \citep{molnar2020interpretable}, which focus on extracting information from predictive models. These types of approaches have also received growing attention in the context of microbiome data \citep{topccuouglu2020framework, gou2021interpretable, ruaud2022interpreting}. However, to the best of our knowledge, none of these procedures have been adjusted to account for the compositional structure. As we show in Sec.~\ref{sec:cfi}, not accounting for the compositionality may invalidate the results.

\KernelBiome{}, see Fig.~\ref{fig:workflow}, addresses both (i) by providing a regression and classification procedure based on kernels targeted to the simplex and (ii) by providing a principled way of analyzing the estimated models. Our contributions are fourfold: (1) We develop a theoretical framework for using kernels on compositional data. While using kernels to analyze various aspects of compositional data is not a new idea, a comprehensive analysis and its connection to existing approaches have been missing. In this work, we provide a range of kernels that each capture different aspects of the simplex structure, many of which have not been previously applied to compositional data. For all kernels, we derive novel, positive-definite weighted versions that allow incorporating prior information between the components.  Additionally, we show that the distance associated with each kernel can be used to define a kernel-based scalar summary statistic.
(2) We propose a theoretically justified analysis of kernel-based models that accounts for compositionality. Firstly, we introduce two novel quantities for measuring the effects of individual features that explicitly take the compositionality into account and prove that these can be consistently estimated. Secondly, we build on known connections between kernels and distance measures to advocate for using the kernel embedding from the estimated model to create visualizations and perform follow-up distance-based analyses that respect the compositionality.
(3) We draw connections between \KernelBiome and log-contrast-based analysis techniques. More specifically, we connect the Aitchison kernel to the log-contrast model, prove that the proposed compositional feature influence in this case reduce to the log-contrast coefficients, and show that our proposed weighted Aitchison kernel is related to the recently proposed tree-aggregation method of log-contrast coefficients \citep{bien2021tree}. Importantly, these connections ensure that \KernelBiome reduces to standard log-contrast analysis techniques whenever simple log-contrast models 
are capable of capturing most of the signal. This is also illustrated by our experimental results.
(4) We propose a data-adaptive selection framework that allows to compare different kernels in a principled fashion.

The paper is structured as follows. In Sec.~\ref{sec:methods}, we introduce the supervised learning task, define
two quantities for analyzing individual components (Sec.~\ref{sec:cfi}),
give a short introduction to kernel
methods and how to apply our methodology (Sec.~\ref{sec:kernel}), and present the full \KernelBiome framework (Sec.~\ref{sec:kernelbiome}).
Finally, we illustrate the advantages of \KernelBiome in the experiments in
Sec.~\ref{sec:experiments}. 


\section{Methods}\label{sec:methods}

We consider the setting in which we observe $n$ independent and identically distributed (i.i.d.) observations $(X_1, Y_1),\ldots,(X_n,Y_n)$ of a random variable $(X, Y)$
with $X\in\simp^{p-1}$ a compositional predictor and $Y\in\bR$ a real-valued response
variable (by which we include categorical responses).
Supervised learning attempts to learn a relationship between the response $Y$ and the dependent predictors $X$. In this work, we focus on conditional mean relationships. More specifically, we are interested in estimating (a version of) the conditional mean of $Y$, that is, the function
\begin{equation}
    \label{eq:conditional_mean}
    f^*: x\mapsto \bE[Y\mid X=x].
\end{equation}
We assume that $f^*\in\mathcal{F}\subseteq\{f\mid f:\simp^{p-1} \to \bR\}$, where $\mathcal{F}$ is a function class determined by the regression (or classification) procedure.

While estimating and analyzing the conditional mean is well established for predictors in Euclidean space, there are two factors that complicate the analysis when the predictors are compositional. (i) While it is possible to directly apply most standard regression procedures designed for $X\in\bR^p$ also for $X\in \simp^{p-1}$, it turns out that many approaches are ill-suited to approximate functions on the simplex. (ii) Even if one accurately estimates the conditional mean function $f^*$, the simplex constraint complicates any direct assessment of the influence and importance of individual components of the compositional predictor. In this work, we address both issues and propose a nonparametric framework for regression and classification analysis for compositional data.

\subsection{Interpreting individual features}\label{sec:cfi}

Our goal when estimating the conditional mean $f^*$ given in \eqref{eq:conditional_mean} is to gain insight into the relationship between the response $Y$ and predictors $X$. For example, when fitting a log-contrast model (see Example~\ref{ex:log-contrast}), the estimated coefficients provide a useful tool to generate hypotheses about which features affect the response and thereby inform follow-up experiments. For more complex models, such as the nonparametric methods proposed in this work, direct interpretation of a fitted model $\hat{f}$ is difficult. Two widely applicable measures due to \citet{friedman2001greedy} are the following: (i) Relative influence, which assigns each coordinate $j$ a scalar influence value given by the expected partial derivative $\bE[\tfrac{d}{dx^j}\hat{f}(X)]$ and (ii) partial dependence plots, which are constructed by plotting, for each coordinate $j$, the function $z\mapsto \bE[\hat{f}(X^1,\ldots,X^{j-1}, z, X^{j+1},\ldots,X^p)]$. However, directly applying these measures on the simplex is not possible as we illustrate in App.~\ref{apd:cfi_cpd_pi_pdp}. The intuition is that both measures evaluate the function $\hat{f}$ outside the simplex. An adaption of the relative influence (or elasticity in the econometrics literature) to compositions based on the Aitchison geometry has recently been proposed by \citet{morais2021impact}. We adapt the relative influence without relying on the log-ratio transform and hence allow for more general function classes.

Our approach is based on two coordinate-wise perturbations on the simplex. For any $j\in\{1,\ldots,p\}$ and $x \in \simp^{p-1}$, define (i) for $c\in [0,\infty)$ the function $\psi_j(x, c) \in \simp^{p-1}$ to be the composition resulting from multiplying the $j$-th component by $c$ and then scaling the entire vector back into the simplex, and (ii) $c\in [0,1]$ the function $\phi_j(x, c)\in \simp^{p-1}$ to be the composition that consists of fixing the $j$-th coordinate to $c$ and then rescaling all remaining coordinates such that the resulting vector lies in the simplex. Each perturbation can be seen as a different way of intervening on a single coordinate while preserving the simplex structure. More details are given in App.~\ref{sec:add_defs}. Based on the first perturbation, we define the \emph{compositional feature influence} (CFI) of component $j\in\{1,\ldots,p\}$ for any differentiable function $f:\simp^{p-1}\rightarrow\bR$ 
by
\begin{equation}
    {\color{gray}\text{(CFI)}}\qquad I_f^j\coloneqq \bE\Big[\tfrac{d}{dc}f(\psi_{j}(X, c))\big\vert_{c=1}\Big].
\end{equation}
Similarly, we adapt partial dependence plots using the second perturbation. Define the \emph{compositional feature dependence} (CPD) of component $j\in\{1,\ldots,p\}$ for any function $f:\simp^{p-1}\rightarrow\bR$ by
\begin{equation}
    {\color{gray}\text{(CPD)}}\qquad S_f^j: z\mapsto \bE[f(\phi_j(X,z))]-\bE[f(X)].
\end{equation}
In practice, we can compute Monte Carlo estimates of both quantities by replacing expectations with empirical means. We denote the corresponding estimators by $\hat{I}^j_f$ and $\hat{S}^j_f$, respectively (see App.~\ref{sec:add_defs} for details).

The following proposition connects the CFI and CPD to the coefficients in a log-contrast function.
\begin{proposition}[CFI and CPD in the log-contrast model]
\label{prop:cfi_cpd_log_contrast_model}
Let $f: x \mapsto \beta^T\log(x)$  with $\sum_{j=1}^p \beta_j = 0$, 
then the CFI and CPD are given by
\begin{equation*}
    I_f^j=\beta_j
    \quad\text{and}\quad
    S_f^j: z\mapsto \beta_j \log\left(\tfrac{z^j}{1-z^j}\right) + c,
\end{equation*}
respectively,
where $c\in\bR$ is a constant depending on the distribution of $X$ but not on $z$ and satisfies $c=0$ if $\beta_j=0$.
\end{proposition}
A proof is given in App.~\ref{apd:cfi_cpd_log_contrast_model_proof}. The proposition shows that the CFI and CPD are generalizations of the $\beta$-coefficients in the log-contrast model. The following example provides further intuition.

\begin{example}[CFI and CPD in a log-contrast model]
\label{ex:log-contrast}
Consider a log-contrast model $Y=f(X) + \epsilon$ with $f: x\mapsto 2\log(x^1) - \log(x^2) - \log(x^3)$.

The CFI and CPD for the true function $f$ --- estimated based on $n=100$ i.i.d.\ samples $(X_1,Y_1),\ldots,(X_n,Y_n)$ with $X_i$ compositional log-normal --- are shown in Fig.~\ref{fig:cfi_example}.
\end{example}

The following theorem highlights the usefulness of the CFI and CPD by establishing when they can be consistently estimated from data.
\begin{theorem}[Consistency]\label{thm:consistency}
Assume $\hat{f}_n$ is an estimator of the conditional mean $f^*$ given in \eqref{eq:conditional_mean} based on $(X_1,Y_1),\ldots,(X_n,Y_n)$ i.i.d..
\begin{itemize}
    \item[(i)] If $\frac{1}{n}\sum_{i=1}^n \big\|\nabla \hat{f}_n(X_i)-\nabla f^*(X_i)\big\|_2\overset{P}{\longrightarrow}0$ as $n\rightarrow\infty$ and $\bE[(\nabla f^*(X_i))^2]<\infty$, then it holds for all $j\in\{1,\ldots,p\}$ that
    $$\hat{I}_{\hat{f}_n}^j\overset{P}{\longrightarrow} I_{f^*}^j
    \quad\text{as }n\rightarrow\infty.$$
    \item[(ii)] If $\sup_{x\in\operatorname{supp}(X)}|\hat{f}_n(x)- f^*(x)|\overset{P}{\longrightarrow}0$ as $n\rightarrow\infty$ and $\operatorname{supp}(X)=\{w/(\sum_j w^j)\mid w\in \operatorname{supp}(X^1)\times\cdots\times \operatorname{supp}(X^p)\}$, then it holds for all $j\in\{1,\ldots,p\}$ and all $z\in[0,1]$ with $z/(1-z)\in\operatorname{supp}(X^j/\sum_{\ell\neq j}X^{\ell})$ that
    $$\hat{S}_{\hat{f}_n}^j(z)\overset{P}{\longrightarrow} S_{f^*}^j(z)
    \quad\text{as }n\rightarrow\infty.$$
\end{itemize}
\end{theorem}
A proof is given in App.~\ref{apd:conistency_proof} and the result is demonstrated on simulated data in App.~\ref{apd:consistency_simulation}. The theorem shows that the CFI is consistently estimated as long as the derivative of $f^*$ is consistently estimated, which can be ensured for example for the kernel methods discussed in Sec.~\ref{sec:kernel}. In contrast, the CPD only requires the function $f^*$ itself to be consistently estimated. The additional assumption on the support ensures that the perturbation $\phi_j$ used in the CPD remains within the support. If this assumption is not satisfied one needs to ensure that the estimated function extrapolates beyond the sample support. Interpreting the CPD therefore requires caution.

\subsection{Kernel methods for compositional data analysis}\label{sec:kernel}

\begin{figure}[t]
    \centering
    \includegraphics{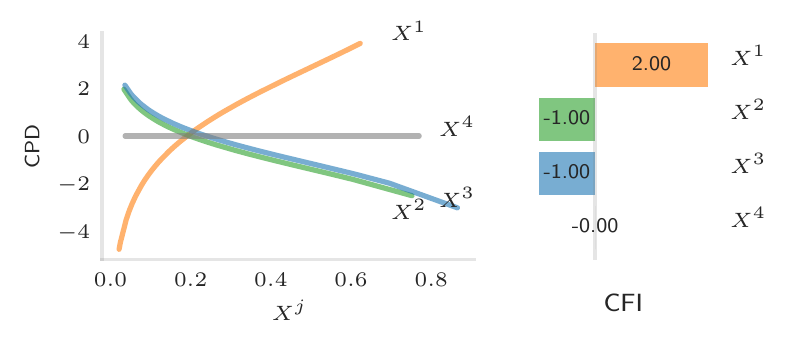}
    \caption{Visualization of the CPD (left) and CFI (right) based on $n=100$ samples and the true function $f$. Since $\beta_4=0$ in this example the $4$-th component has no effect on the value of $f$ resulting in a CFI of zero and a flat CPD. Since we are not estimating $f$, the CFI values exactly correspond to the $\beta$-coefficients in this example.}
    \label{fig:cfi_example}
\end{figure}

Before presenting our proposed weighted and unweighted kernels, we briefly review the necessary background on kernels and their connection to distances. Kernel methods are a powerful class of nonparametric statistical methods that are particularly useful for data from non-standard (i.e., non-Euclidean) domains $\mathcal{X}$. The starting point is a symmetric, positive definite function  $k:\mathcal{X}\times\mathcal{X}\rightarrow\bR$, called kernel. Kernels encode similarities between points in $\mathcal{X}$, i.e., large values of $k$ correspond to points that are similar and small values to points that are less similar. Instead of directly analyzing the data on $\mathcal{X}$, kernel methods map it into a well-behaved feature space $\mathcal{H}_{k}\subseteq\{f\mid f:\mathcal{X}\rightarrow\bR\}$ called reproducing kernel Hilbert space (RKHS), whose inner product preserves the kernel induced similarity.

Here, we consider kernels on the simplex, that is, $\cX=\simp^{p-1}$. The conditional mean function $f^*$ given in \eqref{eq:conditional_mean} can then be estimated by optimizing a loss over $\cH_k$, for an appropriate kernel $k$ for which $\cH_k$ is sufficiently rich, i.e., $f^*\in\cH_k$. The representer theorem \citep[e.g.,][]{scholkopf2002learning} states that such an optimization over $\cH_k$ can be performed efficiently. Formally, it states that the minimizer of an arbitrary convex loss function $L:\bR^n \times \bR^n \rightarrow [0,\infty)$ of the form
\begin{equation*}
    \hat{f} = \argmin_{f\in\cH_k} L\big((Y_1, f(X_1)),\ldots, (Y_n,f(X_n))\big) + \lambda||f||^2_{\cH_k},
\end{equation*}
with $\lambda>0$ a penalty parameter, has the form $\hat{f}(\cdot) = \sum_{i=1}^n \hat\alpha_i k(X_i,\cdot)$ for some $\hat{\alpha}\in\bR^n$. This means that instead of optimizing over a potentially infinite-dimensional space $\cH_k$, it is sufficient to optimize over the $n$-dimensional parameter $\hat{\alpha}$. Depending on the loss function, this allows to construct efficient regression and classification procedures, such as kernel ridge regression and support vector machines \citep[e.g.,][]{scholkopf2002learning}.

The performance of the resulting prediction model depends on the choice of kernel as this determines the function space $\mathcal{H}_k$. A useful way of thinking about kernels is via their connection to distances. In short, any kernel $k$ induces a unique semi-metric $d_k$ and vice versa. More details are given in App.~\ref{app:short_connection_dist_kernel}. This connection has two important implications. Firstly, it provides a natural way for constructing kernels based on established distances on the simplex. The intuition being that a distance, which is large for observations with vastly different responses and small otherwise, leads to an informative feature space $\cH_k$. Secondly, it motivates using the kernel-induced distance, see Sec.~\ref{sec:model_analysis}.

\subsubsection{Kernels on the simplex}\label{sec:kernsl_for_comp}

We consider four types of kernels on the simplex, each related to different types of distances. A full list with all kernels and induced
distances is provided in App.~\ref{apd:list_kernels}. While most kernels have previously appeared in the literature, we have adapted many of the kernels to fit into the framework provided here, e.g., added zero-imputation for Aitchison kernels and updated the parametrization for the probability distribution kernels.

\textbf{Euclidean:} These are kernels that are constructed by restricting kernels on $\bR^p$ to the simplex. Any such restriction immediately guarantees that the restricted kernel is again a kernel. However, the induced distances are not targeted to the simplex and therefore can be unnatural choices. In \KernelBiome, we have included the linear kernel and the radial basis function (RBF) kernel. The RBF kernel is $L^p$-universal \citep[e.g.,][]{sriperumbudur2011universality} which means that it can approximate any integrable function (in the large sample limit). However, this does not necessarily imply good performance for finite sample sizes.

\textbf{Aitchison geometry:} One way of incorporating the simplex structure is to use the Aitchison geometry. Essentially, this corresponds to mapping points from the interior of the simplex via the centered log-ratio transform into $\bR^p$ and then using the Euclidean geometry. This results in the Aitchison kernel for which the induced RKHS is equal to the log-contrast functions. In particular, applying kernel ridge regression with an Aitchison kernel corresponds to fitting a log-contrast model with a penalty on the coefficients. As the centered log-ratio transform is only defined for interior points in the simplex, we add a hyperparameter to the kernels that shift them away from zero.
From this perspective, the commonly added pseudo-count constant added to all components becomes a tuneable hyperparameter of our method, rather than a fixed ad-hoc choice during data pre-processing. Thereby, our modified Aitchison kernel respects the fact that current approaches to zero-replacement or imputation are often not biologically justified, yet may impact predictive performance. Our proposed zero-imputed Aitchison kernel comes with two advantages over standard log-contrast modelling: (1) A principled adjustment for zeros and (2) an efficient form of high-dimensional regularization that performs well across a large range of our experiments. In \KernelBiome, we include the Aitchison kernel and the Aitchison-RBF kernel which combines the Aitchison and RBF kernels.

\textbf{Probability distributions:} Another approach to incorporate the simplex structure into the kernel is to view points in the simplex as discrete probability distributions. This allows us to make use of the extensive literature on distances between probability distributions to construct kernels. In \KernelBiome, we have adapted two classes of such kernels: (1) A parametric class based on generalized Jensen-Shannon distances due to \citet{topsoe2003jenson}, which we call generalized-JS, and (2) a parametric class based on the work by \citet{hein2005hilbertian}, which we call Hilbertian. Together they contain many well-established distances such as the total variation, Hellinger, Chi-squared, and Jensen-Shannon distance. All resulting kernels allow for zeros in the components of compositions.

\textbf{Riemannian manifold:} Finally, the simplex structure can be incorporated by using a multinomial distribution which has a parameter in the simplex. \citet{lafferty2005diffusion} show that the geometry of multinomial statistical models can be exploited by using kernels based on the heat equation on a Riemannian manifold. The resulting kernel is known as the heat-diffusion kernel and has been observed to work well with sparse data.

\subsubsection{Including prior information into kernels}\label{sec:prior_info}

All kernels introduced in the previous section (and described in detail in App.~\ref{apd:list_kernels}) are invariant under permutations of the compositional components. They therefore do not take into account any relation between the components. In many applications, one may however have prior knowledge about the relation between the components. For example, if the compositional predictor consists of relative abundances of microbial species, information about the genetic relation between different species encoded in a phylogenetic tree may be available. Therefore, we provide the following way to incorporate such relations. Assume the prior information has been expressed as a positive semi-definite weight matrix $W\in\bR^{p\times p}$ with non-negative entries (e.g., using the UniFrac-Distance \citep{lozupone2005unifrac} as shown in App.~\ref{apd:unifrac}), where the $ij$-th entry corresponds to the strength of the relation between components $i$ and $j$. We can then incorporate $W$ directly into our kernels. To see how this works, consider the special case where the kernel $k$ can be written as
$k(x,y) = \textstyle\sum_{i=1}^p k_0(x^i, y^i)$ for a positive definite kernel $k_0: [0,1] \times [0,1] \to \bR$. Then, the weighted kernel 
\begin{equation}
 k_W(x,y) \coloneqq \textstyle\sum_{i,j=1}^p W_{i,j} \cdot k_0(x^i, y^j)
\end{equation}
is positive definite and incorporates the prior information in a natural way. If two components $i$ and $j$ are known to be related (corresponding to large values of $W_{i,j}$), the kernel $k_W$ takes the similarity across these components into account. In App.~\ref{apd:weighted_kernels_details}, we show that the probability distribution kernels and the linear kernel can be expressed in this way
and propose similar weighted versions for the remaining kernels.

An advantage of our framework is that it defaults to the log-contrast model when more complex models fail to improve the prediction.\footnote{Due to the zero-shift in our proposed Aitchison kernel and the kernel-based regularization, this correspondence is however not exact.} We now show that for the weighted Aitchison kernel, the RKHS consists of log-contrast functions with equal coefficients across the weighted blocks, this is similar to how \citet{bien2021tree} incorporate prior information into log-contrast models.

\begin{proposition}[weighted Aitchison kernel RKHS]
\label{prop:weighted_aitchison}
Let $P_1,\ldots,P_m\subseteq\{1,\ldots,p\}$ be a disjoint partition and $W\in\mathbb{R}^{p\times p}$ the weight matrix defined for all $i,j\in\{1,\ldots,p\}$ by $W_{i,j}\coloneqq \sum_{\ell=1}^m \frac{1}{|P_{\ell}|}\mathds{1}_{\{i,j\in P_{\ell}\}}$. Let $k_W$ be the weighted Aitchison kernel given in App.~\ref{apd:list_weighted_kernels} (but without zero imputation and on the open simplex). Then, it holds that
\begin{equation*}
    f\in \mathcal{H}_{k_W} \quad\Leftrightarrow\quad
    f= \beta^{\top}\log(\cdot)
\end{equation*}
for some $\beta\in\mathbb{R}^p$ satisfying (1) $\sum_{j=1}^p\beta_j=0$ and (2) for all $\ell\in\{1,\ldots,m\}$ and $i,j\in P_{\ell}$ it holds $\beta_i=\beta_j$
\end{proposition}

A proof is given in App.~\ref{apd:prop:weighted_aitchison}. Combined with Proposition~\ref{prop:cfi_cpd_log_contrast_model}, this implies that the CFI values are equal across the equally weighted blocks $P_1,\ldots,P_m$, which is demonstrated on data in Sec.~\ref{sec:weighting_exp}.

\subsection{KernelBiome framework}\label{sec:kernelbiome}

For a given i.i.d.\ dataset $(X_1,Y_1),\ldots,(X_n,Y_n)$, \KernelBiome first runs a data-driven model selection, resulting in an estimated regression function $\hat{f}$ and a specific kernel $\hat{k}$ (see Fig.~\ref{fig:workflow}). Then, the feature influence properties (CFI, CPD) and embedding induced by $\hat{k}$ are analyzed in a way that respects compositionality.

\subsubsection{Model selection}

We propose the following two step data-driven selection procedure.
\begin{enumerate}
    \setlength{\itemsep}{-7pt}
    \item Select the best kernel $\hat{k}$ with the following hierarchical CV:
    \vspace{-7pt}
    \begin{itemize}
        \item Fix a kernel $\tilde{k}$, i.e., a type of kernel and its kernel parameters.
        \item Split the sample into $N_{\text{out}}$ random (or stratified) folds.
        \item For each fold, use all other folds to perform a $N_{\text{in}}$-CV to select the best hyperparameter $\tilde{\lambda}$ and compute a CV score based on $\tilde{k}$ and $\tilde{\lambda}$ on the left-out fold.
        \item Select the kernel $\hat{k}$ with the best average CV score.
    \end{itemize}
    \item Select the best hyperparameter $\hat{\lambda}$ for $\hat{k}$ using a $N_{\text{in}}$-fold CV on the full data. The final estimator $\hat{f}$ is then given by the kernel predictor based on $\hat{k}$ and $\hat{\lambda}$.
\end{enumerate}
We provide default parameter grids for the kernel parameters of each kernel type (see Table~\ref{tab:default_param_grid} in App.~\ref{apd:summary_kernels_and_metrics}). The kernel and the hyperparameters are treated separately, because the kernel is itself of interest and is used in subsequent analyses as discussed in the next section.

\subsubsection{Model analysis}\label{sec:model_analysis}

Firstly, as discussed in Sec.~\ref{sec:cfi}, we propose to analyze the fitted model $\hat{f}$ with the CPD and CFI. Other methods developed for functions on $\bR^p$ do not account for compositionality and can be misleading. Secondly, the kernel embedding $\hat{k}$ can be used for the following two types of analyses.

\textbf{Distance-based analysis:} A key advantage of using kernels is that the fitted kernel $\hat{k}$ is itself helpful in the analysis. As discussed in Sec.~\ref{sec:kernel}, $\hat{k}$ induces a distance on the simplex that is well-suited to separate observations with different response values. We therefore suggest to utilize this distance to investigate the data further. Essentially, any statistical method based on distances can be applied.
We specifically suggest using kernel PCA to project the samples into a two-dimensional space. As we illustrate in Sec.~\ref{sec:postanalysis_exp}, such a projection can be used to detect specific groups or outliers in the samples and can also help understand how the predictors are used by the prediction model $\hat{f}$. As we are working with compositional data we need to be careful when looking at how individual components contribute to each principle component. Fortunately, the perturbation $\psi$ defined in Sec.~\ref{sec:cfi} can again be used to construct informative component-wise measures. All details on kernel PCA and how to compute component-wise contributions for each principle component are provided in App.~\ref{app:kernel_mds}. 

\textbf{Data-driven scalar summary statistics:} 
Practitioners often work with scalar summaries of the data as these are easy to communicate.
A commonly used summary statistic in ecology is $\alpha$-diversity which measures the variation within a community. The connection between kernels and distances provides a useful tool to construct informative scalar summary statistics by considering distances to a reference point $u$ in the simplex. Formally, for a fixed reference point $u\in\simp^{p-1}$ define for all $x\in\simp^{p-1}$ a corresponding closeness measure by $D^k(x)\coloneqq -d^2_k(x, u)$, where $d_k$ is the distance induced by the kernel $k$. This provides an easily interpretable scalar quantity. For example, if $Y$ is a binary indicator taking values \emph{healthy} and \emph{sick}, we could select $u$ to be the geometric median\footnote{The geometric median is the observation that has the smallest total distance to all the other observations based on the pairwise kernel distance.} of all $X_i$ with $Y_i=\text{\emph{healthy}}$. Then, $D^k$ corresponds to a very simple health score (see Sec.~\ref{sec:postanalysis_exp} for a concrete example). A further example is given by selecting $u=(1/p,\ldots,1/p)$ and considering points on the simplex as communities. Then, $u$ can be interpreted as the most diverse point in the simplex and $D^k$ corresponds to a data-adaptive $\alpha$-diversity measure. While such a definition of diversity does not necessarily satisfy all desirable properties for diversities \citep[see e.g.,][]{Cobbold2012}, it is (1) symmetric with respect to switching of coordinates, (2) has an intuitive interpretation and (3) is well-behaved when combined with weighted kernels. Connections to established diversities also exist, for example, the linear kernel corresponds to a shifted version of the Gini-Simpson diversity (i.e., $\text{Gini-Simpson}(x)\coloneqq1-\sum_{j=1}^{p}(x^j)^2=D^{k}(x)+\tfrac{p-2}{p}$).

\subsubsection{Implementation}

\KernelBiome is implemented as a Python \citep{vanrossum2009python} package that takes advantage of the high-performance \texttt{JAX} \citep{jax2018github} and \texttt{scikit-learn} \citep{scikit-learn} libraries. All kernels introduced are implemented with \texttt{JAX}'s just-in-time compilation and automatically leverage accelerators such as GPU and TPU whenever available. \KernelBiome provides fast computation of all kernels and distance metrics as well as easy-to-use procedures for model selection and comparison and procedures to estimate CPD and CFI, compute kernel PCA, and estimate scalar summary statistics.


\section{Results}\label{sec:experiments}

We evaluated \KernelBiome on a total of $33$ microbiome datasets. All datasets have been previously published and a full list including details on the pre-processing, prediction task and references is provided in Table~\ref{tab:datasets} in App.~\ref{apd:details_on_data}. First, in Sec.~\ref{sec:prediction_exp}, we show that \KernelBiome performs on par or better than existing supervised learning procedures for compositional data, while reducing to a powerfully regularized version of the log-contrast model if the prediction task is simple. In Sec.~\ref{sec:weighting_exp}, we demonstrate how \KernelBiome can incorporate prior knowledge, while preserving a theoretically justified interpretation. Finally, in  Sec.~\ref{sec:postanalysis_exp}, we illustrate the advantages of a full analysis with \KernelBiome.

\subsection{State-of-the-art prediction performance}\label{sec:prediction_exp}

We compare the predictive performance of \KernelBiome on all datasets with the following competitors: (i) \texttt{Baseline}, a naive baseline that predicts the majority class for classification and the mean for regression, (ii) \texttt{SVM-RBF}, a support vector machine with the RBF kernel, (iii) \texttt{Lin/Log-L1}, a linear/logistic regression with $\ell^1$-penalty (iv) \texttt{LogCont-L1}, a log-contrast regression with $\ell^1$ penalty with a half of the minimum non-zero relative abundance added as pseudo-count to remove zeros, and (v) \texttt{RF}, a random forest with $500$ trees. For \texttt{SVM-RBF}, \texttt{Lin/Log-L1} and \texttt{RF} we use the \texttt{scikit-learn} implementations \citep{scikit-learn} and choose the hyperparameters (bandwidth, max depth and all penalty parameters) based on a 5-fold CV. For \texttt{LogCont-L1}, we use the \texttt{c-lasso} package \citep{simpson2021classo} and the default CV scheme to chose the penalty parameter. We apply two versions of \KernelBiome: (1) The standard version that adaptively chooses the kernel using $N_{\text{in}}=5$, $N_{\text{out}}=10$ (denoted \KernelBiome), and (2) a version with fixed Aitchison kernel with $c$ equal to half of the minimum non-zero relative abundance (denoted \texttt{KB-Aitchison}). Both methods use a default hyperparameter grid of size $40$.

For the comparison we perform $20$ random (stratified) $10$-fold train/test splits and record the predictive performance (accuracy for classification and mean-squared error (MSE) for regression) on each test set. The results for a subselection of the datasets are shown in Fig.~\ref{fig:real_datasets_scores}. The results for the remaining datasets as well as ROC-curves for all classification tasks are provided in App.~\ref{apd:remaining_results}. 
On all datasets \KernelBiome achieves the best or close to best performance, indicating that the proposed procedure is well-adapted to microbiome data. There are several further interesting observations: (1) Even though \KernelBiome selects mostly the Aitchison kernel on rmp, it outperforms \texttt{KB-Aitchison}, we attribute this to the advantage of the data-driven zero-imputation. (2) On datasets were the top kernel is selected consistently (e.g., americangut, hiv and tara) \KernelBiome generally performs very well and in these cases strongly outperformed both log-contrast based methods 
\texttt{KB-Aitchison} and \texttt{LogCont-L1}. (3) The predictive performance is substantially different between \texttt{KB-Aitchison} and \texttt{LogCont-L1} which we see as an indication that the regularization is crucial in microbiome datasets.

\begin{figure}[!tpb]
  \centering
  \includegraphics[width=0.75\textwidth]{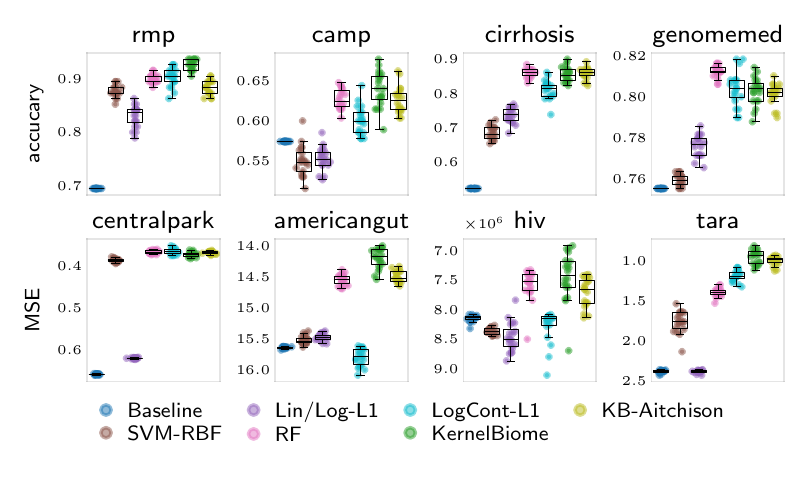}
  \caption{Comparison of predictive performance on 8 public datasets (4 classification and 4 regression tasks) based on a 10-fold train/test split. The frequency and kernel which was selected most often by \KernelBiome is $58\%$ Aitchison for rmp, $50\%$ Aitchison-RBF for camp, $73\%$ Aichtison-RBF for cirrhosis, $67\%$ Aitchison-RBF for genomemed, $45\%$ generalized-JS for centralpark, $100\%$ Aitchison-RBF for americangut, $97\%$ Aitchison-RBF for hiv and $93\%$ Aitchison-RBF for tara.}
  \label{fig:real_datasets_scores}
\end{figure}

\subsection{Model analysis with KernelBiome}\label{sec:postanalysis_exp}

\begin{figure*}[t]
\resizebox{\linewidth}{!}{
    \begin{tikzpicture}
    \node (A) at (0, 0) {\includegraphics[width=0.5\linewidth]{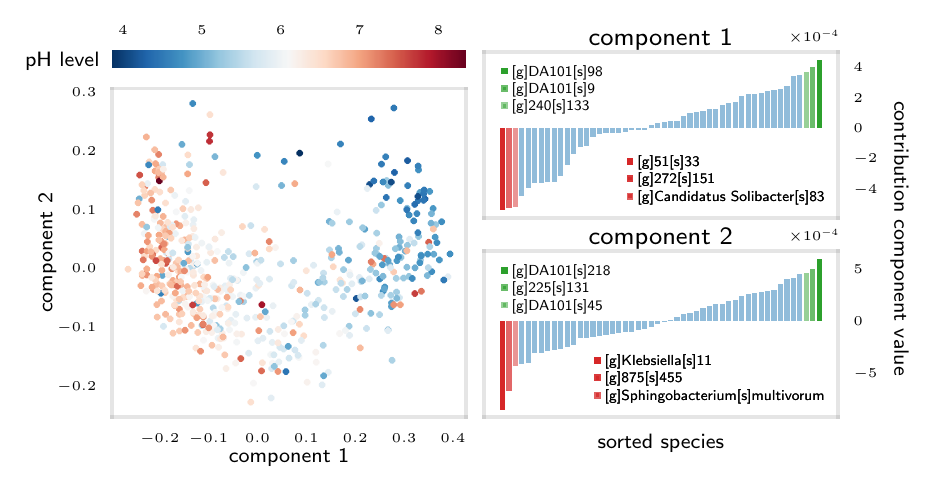}};
    \node (B) at (6.9, 0.95) {\includegraphics{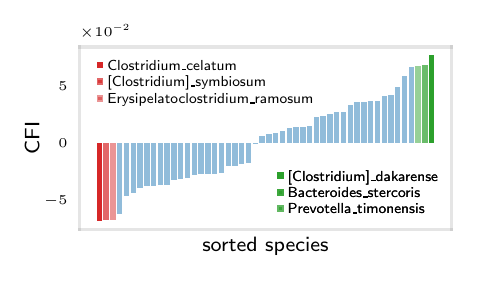}};
    \node (C) at (11, 0) {\includegraphics{figures/sum_stat_and_sim_div_cirrhosis.pdf}};
    \node (E) at (6.9, -1.6) {\includegraphics[width=0.28\linewidth]{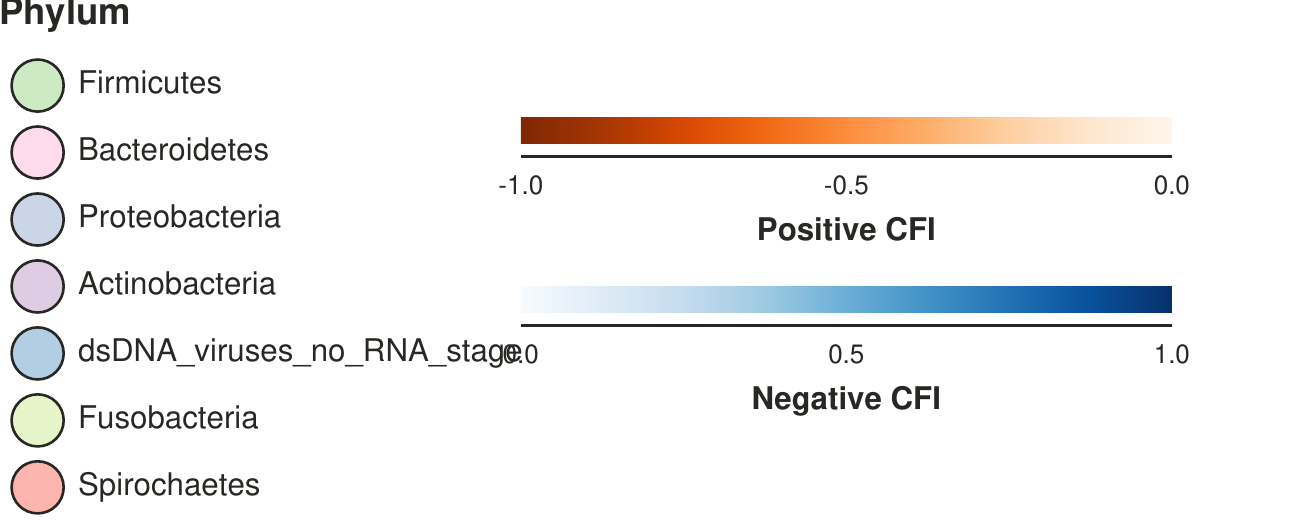}};
    \node (D) at (4, -5) {
    \includegraphics[width=.22\linewidth, trim={0 -6cm 0 1cm},clip]{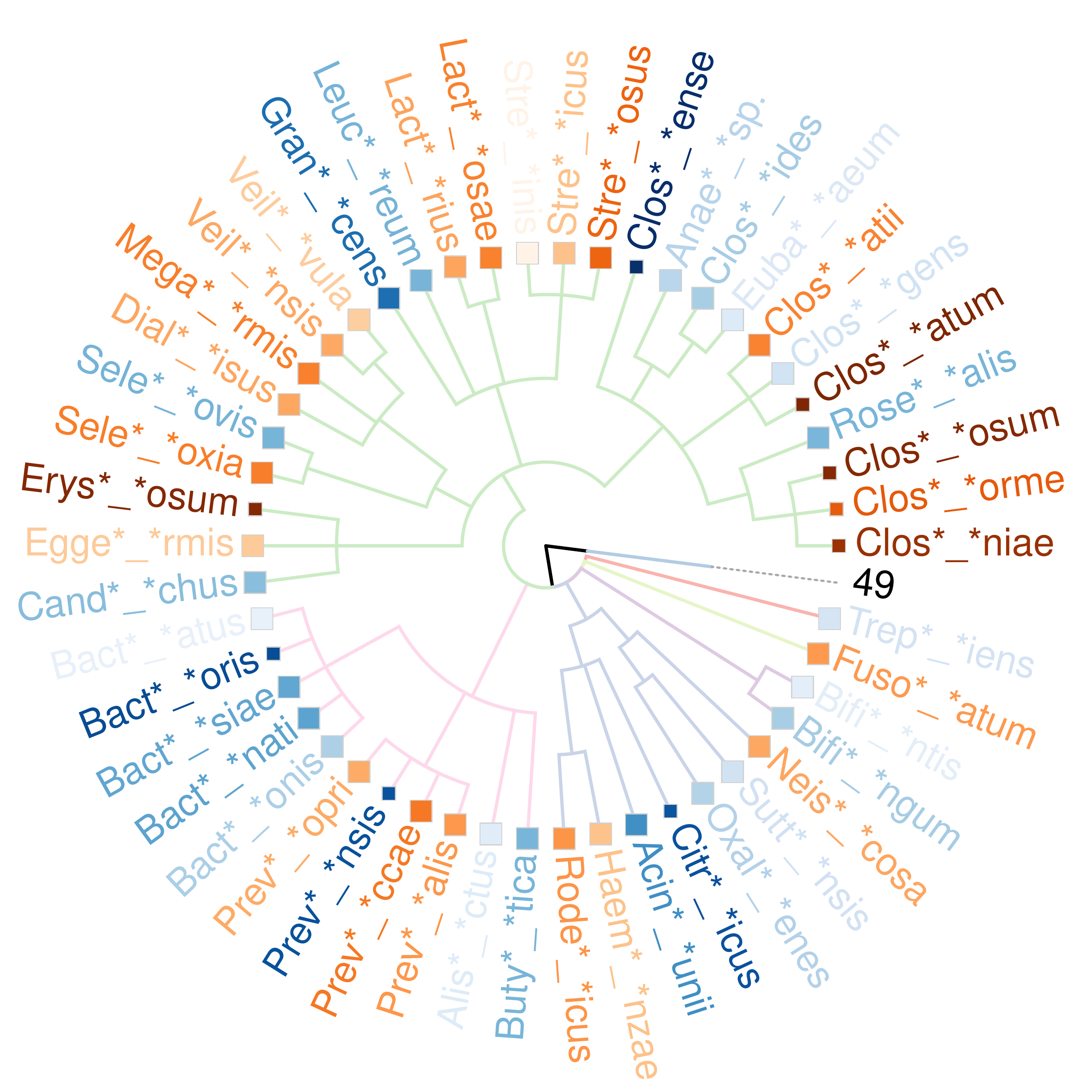} 
    \hspace{1cm}
    \includegraphics[width=.22\linewidth, trim={0 -6cm 0 1cm},clip]{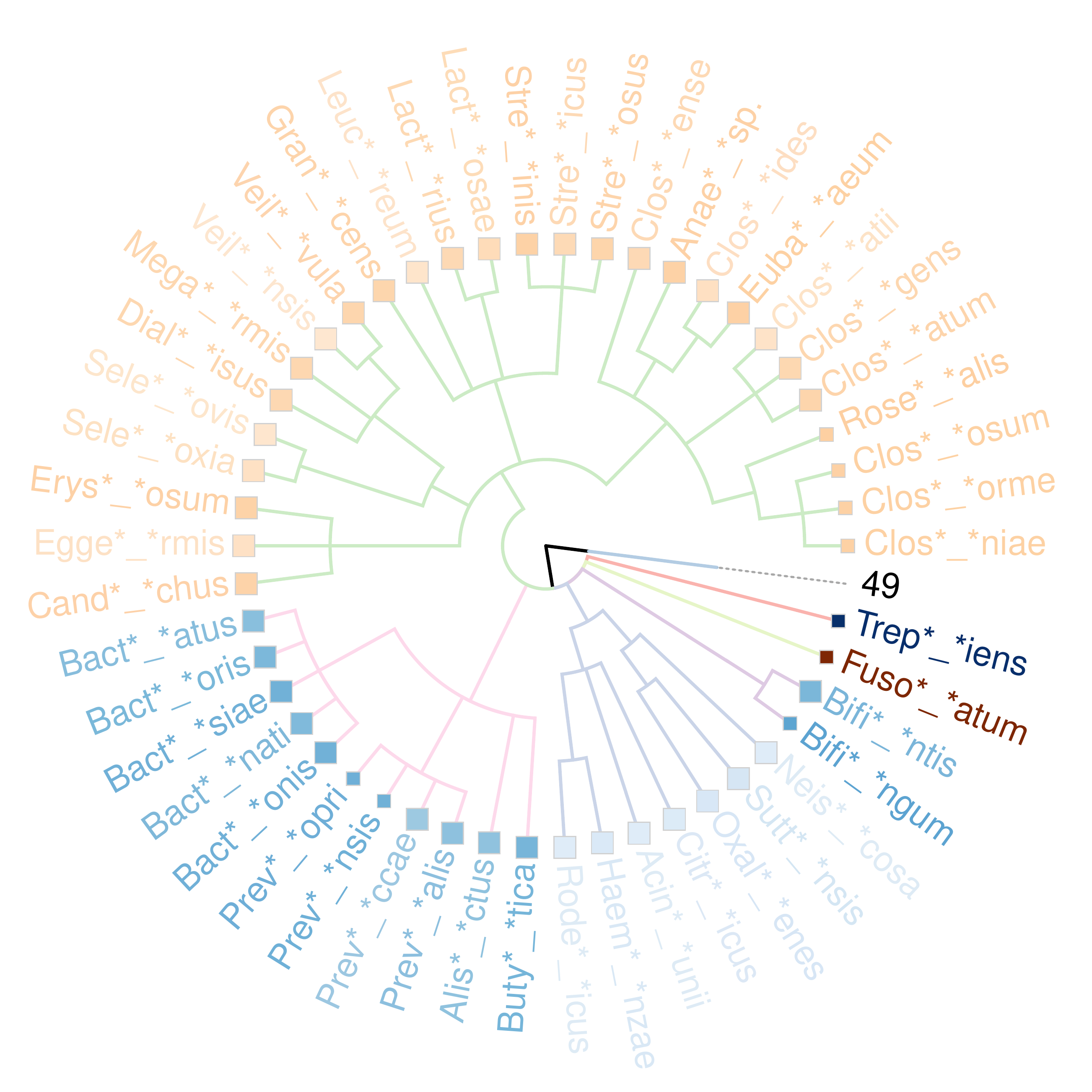}
    \hspace{1cm}
    \includegraphics[width=.22\linewidth, trim={0 -6cm 0 1cm},clip]{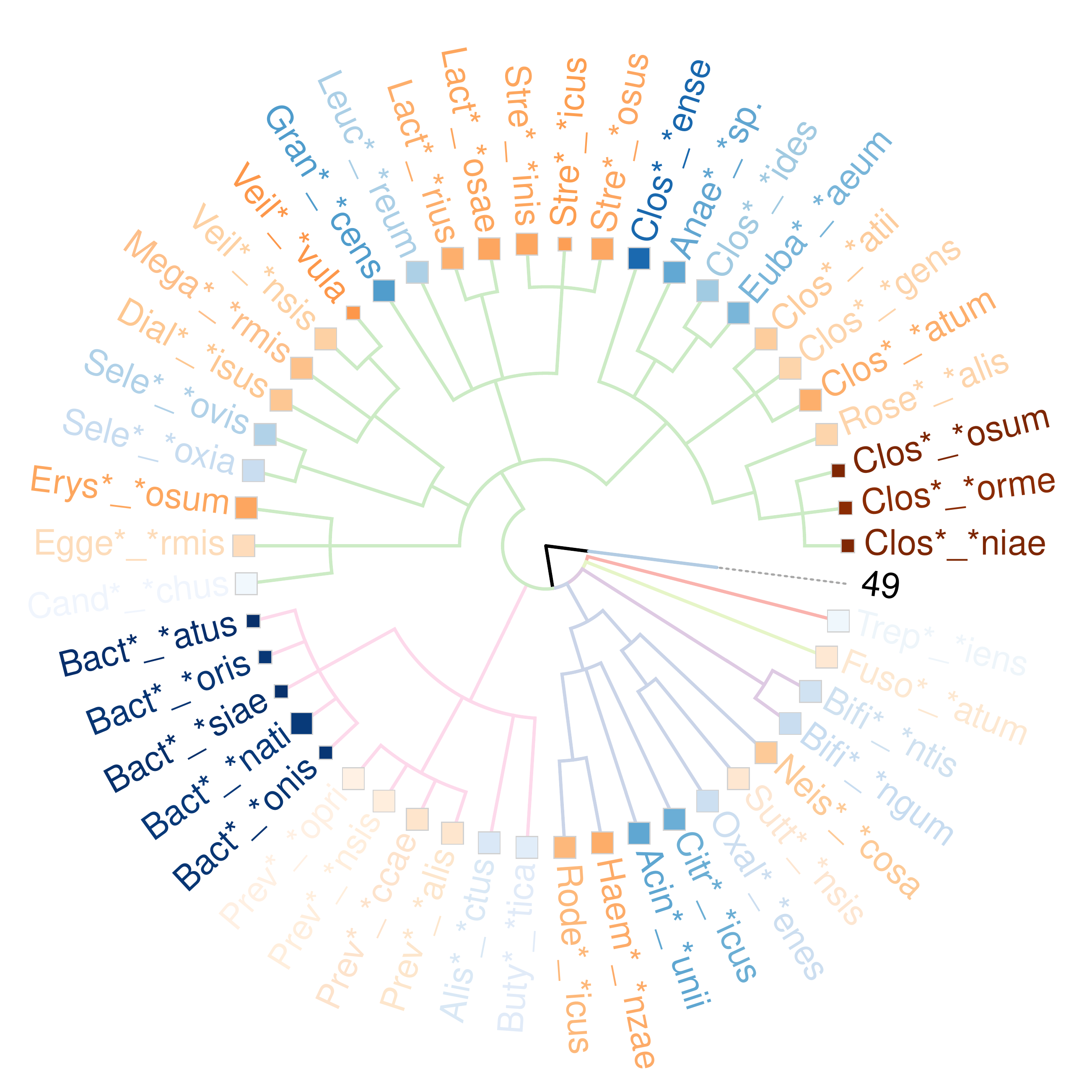}
    };
    \draw[dotted, rounded corners=4pt] (-4.2, 2.3) -- (4.2, 2.3) -- (4.2,-2.2) -- (-4.2,-2.2) -- cycle;
    \node[fill=white, draw] (aa) at (-3.8, 2.3) {\textbf{a}};
    \draw[dotted, rounded corners=4pt] (4.4, 2.3) -- (12.7, 2.3) -- (12.7,-2.2) -- (9.3,-2.2) -- (9.3, -0.4) -- (4.4, -0.4) -- cycle;
    \node[fill=white, draw] (bb) at (4.8, 2.3) {\textbf{b}};
    \draw[dotted, rounded corners=4pt] (-4.2, -2.4) -- (4.4, -2.4) -- (4.4, -0.6) -- (9.1, -0.6) -- (9.1, -2.4) -- (12.7, -2.4) -- (12.7, -6.5) -- (-4.2, -6.5) -- cycle;
    \node[fill=white, draw] (cc) at (-3.8, -2.4) {\textbf{c}};
    \node (clab1) at (-3,-2.8) {\small unweighted};
    \node (clab2) at (1.8,-2.8) {\small phylum-weighted};
    \node (clab3) at (6.7,-2.8) {\small UniFrac-weighted};
    \end{tikzpicture}}
\caption{(a) shows a kernel PCA for the centralpark dataset with 2 principle components. On the right, the contribution of the species to each of the two components is given (see App.~\ref{app:kernel_mds} for details). (b) and (c) are both based on the cirrhosis dataset. In (b) the CFI values are shown on the left and the right plot compares the proposed kernel health score with Simpson diversity.
In (c) the scaled CFI values for are illustrated for different weightings. A darker color shade of the (shortened) name of the microbiota signifies a stronger (positive resp. negative) CFI.}
\label{fig:model_analysis}
\end{figure*}

As shown in the previous section, \KernelBiome results in fitted models with state-of-the-art prediction performance. This is useful because supervised learning procedures can be used in two types of applications: (1) To learn a prediction model that has a direct application, e.g., as a diagnostic tool, or (2) to learn a predictive model as an intermediate step of an exploratory analysis to find out what factors could be driving the response. As discussed above (2) requires us to take the compositional nature of the predictors into account to avoid misleading conclusions. We show how the \KernelBiome framework can be used to achieve this based on two datasets: (i) \emph{cirrhosis}, based on a study analyzing the differences in microbial compositions between $n=130$ healthy and cirrhotic patients \citep{qin2014alterations} and (ii) \emph{centralpark}, based on a study analyzing the pH concentration using microbial compositions from $n=580$ soil samples \citep{ramirez2014biogeographic}. Our aim is not do draw novel biological conclusions, but rather to showcase how \KernelBiome can be used in this type of analysis.

To reduce the complexity, we screen the data using \KernelBiome with the Aitchison kernel and only keep the $50$ taxa with the highest absolute CFIs. We then fit \KernelBiome with default parameter grid. For cirrhosis this results in the Aitchison kernel and for centralpark in the Aitchison-RBF kernel. As outlined in Sec.~\ref{sec:model_analysis}, we can then apply a kernel PCA with a compositionally adjusted component influence. The result for centralpark is given in Fig.~\ref{fig:model_analysis} (a) (for cirrhosis see Fig.~\ref{fig:kpca_cirrhosis} in App.~\ref{apd:remaining_results}). This provides some direct information on which perturbations affects each principle component (e.g., ``[g]DA101[s]98" affects the first component the most positively and ``Sphingobacterium[s]multivorum" affects the second component the most negatively). Moreover, it also directly provides a tool to detect groupings or outliers of the samples. For example, the samples in the top middle (i.e., center of the U-shape) in Fig.~\ref{fig:model_analysis} (a) could be investigated further as they behave different to the rest.

A further useful quantity is the CFI, which for cirrhosis is given Fig.~\ref{fig:model_analysis} (b, left)  (for centralpark see Fig.~\ref{fig:cfi_centralparksoil} in App.~\ref{apd:remaining_results}). They explicitly take the compositional structure into account and have an easy interpretation. For example, ``Prevotella\_timonensis" has a CFI of 0.07 which implies that on average solely increasing ``Prevotella\_timonensis"  will lead to a larger predicted response. We therefore believe that CFIs are more trustworthy than relying on for example Gini-importance for random forests, which does not have a clear interpretation due to the compositional constraint.

Lastly, one can also use the connection between kernels and distances to construct useful scalar summary statistics. In Fig.~\ref{fig:model_analysis} (b, right), we use the kernel-distance to the geometric median in the healthy subpopulation as a scalar indicator for the healthiness of the microbiome. In comparison with more standard scalar summary statistics such as the Simpson diversity, it is targeted to distinguish the two groups.

\subsection{Incorporating prior information}\label{sec:weighting_exp}

In many applications, in particular in biology, prior information about a system is available and should be incorporated into the data analysis. As we show in Sec.~\ref{sec:prior_info}, \KernelBiome allows for incorporating prior knowledge on the relation between individual components (e.g., taxa). To illustrate how this works in practice, we again consider the screened cirrhosis dataset. We apply KernelBiome with an Aitchison-kernel and $c$ equal to half the minimum non-zero relative abundance without weighting, with a phylum-weighting and with a UniFrac-weighting. The resulting scaled CFI values for each are visualized in Fig.~\ref{fig:model_analysis} (c). The phylum-weighting corresponds to giving all taxa within the same phylum the same weights and the UniFrac-weighting is a weighting that incorporates the phylogenetic structure based on the UniFrac-distance and is described in App.~\ref{apd:unifrac}. As can bee seen in Fig.~\ref{fig:model_analysis} (c), the phylum weighting assign approximately the same CFI to each variable in same phylum, this is expected given that the phylum weighting has exactly the structure given in Proposition~\ref{prop:weighted_aitchison}. 
Moreover, the UniFrac-weighting leads to CFI values that lie 
in-between the unweighted and phylum-weighted versions. Similar effects are seen for different kernels as well. The same plots for the generalized-JS kernel are provided in App.~\ref{apd:details_on_data}.

\section{Discussion and conclusions}\label{sec:discussion}

In this work, we propose the \KernelBiome framework for supervised learning with compositional covariates consisting of two main ingredients: data-driven model selection and model interpretation. Our approach is based on a flexible family of kernels targeting the structure of microbiome data, and is able to work with different kernel-based algorithms such as SVM and kernel ridge regression. It is also possible to incorporate prior knowledge, which is crucial in microbiome data analysis. We compare \KernelBiome with other state-of-the-art approaches on $33$ microbiome datasets and show that \KernelBiome achieves improved or comparable results. Moreover, \KernelBiome provides multiple ways to extract interpretable information from the fitted model. Two novel measures, CFI and CPD, can be used to analyze how each component affects the response. We prove the consistency of these two measures and demonstrate them via simulated and real datasets. \KernelBiome also leverages the connection between kernels and distances to conduct distance-based analysis in a lower-dimensional space. 

\section*{Acknowledgements}
The authors would like to thank Christian M\"uller for detailed feedback and suggestions on this work, Johannes Ostner for help creating the circle plots, Jeroen Raes and Doris Vandeputte for making their raw data available and the anonymous reviewer for helpful feedback.

\section*{Funding}

SH and NP are supported by a research grant (0069071) from Novo Nordisk Fonden. EA is supported by the Helmholtz Association under the joint research school ``Munich School for Data Science - MUDS''.\vspace*{-12pt}

\clearpage

\begin{appendices}
\defcitealias{human2012structure}{HMP (2012)}


\noindent The supplementary material is divided into the following sections.

\begin{enumerate}
    \item[\ref{sec:add_defs}.] Details on CFI and CPD\\
        \textit{Formal definitions of perturbations and estimators related to CFI and CPD.}
    \item[\ref{apd:summary_kernels_and_metrics}] Details on kernels included in \KernelBiome\\
        \textit{Overview of different kernel types, details on how they connect to distances and description of weighted kernels.}
    \item[\ref{apd:experiment_details}] Details and additional results for experiments in Sec.~\ref{sec:experiments}\\
        \textit{Datasets pre-processing, parameter setup, construction of the weighting matrices with UniFrac-distance and further experiment results based on the cirrhosis and centralpark datasets.}
    \item[\ref{apd:additional_experiments}] Additional experiments with simulated data\\
        \textit{Consistency of CFI and CPD and comparison of CFI and CPD with their non-simplex counterparts.}
    \item[\ref{app:kernels}] Background on kernels\\
        \textit{Mathematical background on kernels and details on dimensionality and visualization with kernels.}
    \item[\ref{apd:proofs}] Proofs
    \item[\ref{apd:list_kernels}] List of kernels implemented in \KernelBiome
\end{enumerate}

\newpage

\section{Details on CFI and CPD}\label{sec:add_defs}

\subsection{Perturbations}
Formally, the multiplicative perturbation $\psi$ and the fixed coordinate perturbation $\phi$ are defined as follows.
\begin{itemize}
    \item For all $j\in\{1,\ldots,p\}$, $x \in \simp^{p-1}$ with $x^j\neq 1$ and $c\in [0,\infty)$, define 
    \begin{equation*}
        \psi_j(x, c) \coloneqq s_c (x^1,\cdots,x^{j-1}, cx^j, x^{j+1},\cdots,x^p)\in\simp^{p-1},
    \end{equation*}
    where $s_c = 1/(\sum_{\ell\neq j}^p x^{\ell}+cx^j)$.
    \item For all $j\in\{1,\ldots,p\}$, $x \in \simp^{p-1}$ with $\sum_{\ell\neq j}^p x^{\ell}> 0$ and $c\in [0,1]$,     define the intervened composition by
    \begin{equation*}
        \phi_j(x, c) \coloneqq (s x^1,\cdots,sx^{j-1}, c, s x^{j+1},\cdots,s x^p)\in\simp^{p-1},
    \end{equation*}
    where $s = (1-c)/(\sum_{\ell\neq j}^p x^{\ell})$.
\end{itemize}

\subsection{Estimators}
We propose to estimate CFI and CPD with the following two estimators.
\begin{itemize}
    \item For i.i.d.\ observations $X_1,\ldots,X_n$ and a differentiable function $f:\simp^{p-1}\rightarrow\bR$, we estimate the CFI for all $j\in\{1,\ldots,p\}$ as
    \begin{equation*}
        \hat{I}_f^j=\frac{1}{n}\sum_{i=1}^{n}\tfrac{d}{dc}f(\psi(X_i,c))\big\vert_{c=1}.
    \end{equation*}
    \item For i.i.d.\ observations $X_1,\ldots,X_n$ and a function function $f:\simp^{p-1}\rightarrow\bR$, we estimate the CPD for all $j\in\{1,\ldots,p\}$ and $z\in[0,1]$ as
    \begin{equation*}
        \hat{S}_f^j(z)=\frac{1}{n}\sum_{i=1}^{n}f(\phi(X_i,z))-\frac{1}{n}\sum_{i=1}^n f(X_i).
    \end{equation*}
\end{itemize}


\section{Details on kernels included in KernelBiome} \label{apd:summary_kernels_and_metrics}

\subsection{Overview of kernels}

In this section, we give additional details on the kernels used in
\KernelBiome. A full list of all kernels and their corresponding
metrics together with a visualization on $\simp^2$ is given in
App.~\ref{apd:list_kernels}.

As discussed in the main paper, we consider four types of kernels.
\begin{itemize}
\item \textbf{Euclidean} These are kernels that are used on Euclidean
  space but restricted to the simplex. This includes the \emph{linear kernel}
  and the \emph{RBF kernel}.
\item \textbf{Probability distribution} These are kernels that are
  constructed from metrics between probability
  distributions. \KernelBiome includes two parametric classes of
  kernels, the \emph{Hilbertian kernel} and the \emph{generalized-JS
    kernel}. These kernels correspond to multiple well-known metrics
  on probabilities such as the \emph{chi-squared metric}, the
  \emph{total-variation metric}, the \emph{Hellinger metric} and the
  \emph{Jensen-Shannon metric}.
\item \textbf{Aitchison geometry} These are kernels that are
  constructed by using the centered log-ratio transform to project
  data on the simplex into Euclidean space and then combining it with
  a Euclidean kernel. \KernelBiome includes the \emph{Aitchison kernel} and the
  \emph{Aitchison RBF kernel}. In order to allow for zeros, a small positive number $c$ is added to each coordinate for all observations before applying the centered log-ratio transformation.
\item \textbf{Riemannian manifold} These kernels are connected to the simplex
  via multinomial distributions and have been shown to
  empirically perform well on sparse text data mapped into the
  simplex. \KernelBiome contains the \emph{heat-diffusion kernels}.
\end{itemize}
For each type of kernel there are multiple parameter settings. Although users of the \KernelBiome package can freely change the parameters, the
default settings for \KernelBiome for each type of kernel are provided by the package and are given in Table~\ref{tab:default_param_grid}.

\begin{table}[h]
\begin{tabular}{p{0.12\textwidth}p{0.12\textwidth}p{0.5\textwidth}p{0.1\textwidth}}
    \toprule
    \textbf{Geometry} & \textbf{Kernel} & \textbf{Parameters} & \textbf{Number of}\newline \textbf{kernels}\\ \midrule
    Euclidean & linear & none & $1$ \\
    & RBF & $\sigma^2\in\{10^{-2}\cdot m_1, 10^{-1}\cdot m_1, m_1, 10\cdot m_1, $ \newline \hspace*{1.05cm} $10^{2}\cdot m_1,10^3 \cdot m_1,10^{4}\cdot m_1\}$ & $7$ \\ \midrule
    Probability & generalized-JS & $(a, b) \in \{(1, 0.5),(1, 1),(10, 0.5), (10, 1),(10, 10),$ \newline \hspace*{1.3cm} $(\infty, 0.5),(\infty, 1),(\infty, 10),(\infty, \infty)\}$ & $9$ \\
    distributions & Hilbertian &  $(a,b) \in \{(1, -1), (1, -10), (1, -\infty), (10, -1), $ \newline \hspace*{1.3cm} $(10, -10), (10, -\infty), (\infty, -1), (\infty, -10)\}$ & $8$ \\\midrule
    Aitchison &  Aitchison & $c\in\{\mu_X/2\cdot 10^{-4},\ldots, \min(\mu_X/2\cdot 10^4, 10^{-2})\}$ & $9$
    \\
    geometry &  Aitchison-RBF & $c\in\{\mu_X/2\cdot 10^{-4},\ldots, \min(\mu_X/2\cdot 10^4, 10^{-2})\}, $ \newline $\sigma\in\{c\cdot m_2\cdot 10^{-1}, c\cdot m_2, c\cdot m_2 \cdot 10\}$ & $15$ 
    \\ \midrule
    Riemannian manifold &  heat-diffusion & $t=x^{\frac{2}{n-1}}\frac{1}{4\pi}$ for $x\in\{10^{-20},\ldots, 10\}$ & $6$ \\ \bottomrule
\end{tabular}
\caption{Default parameter grid in \KernelBiome. $m_1$ and $m_2$ are the median heuristic for the RBF and Aitchison-RBF kernel, respectively, which depend on the data. $\mu_X$ is the minimal non-zero value in $X$. The zero grids for the Aitchison geometry kernels have an even logarithmic spacing and contain $9$ and $5$ parameters for the Aitchison and Aitchison-RBF, respectively. Similarly, the grid for $x$ for the heat-diffusion kernel has an even logarithmic spacing with $6$ values. There are a total of $55$ kernels.} \label{tab:default_param_grid}
\end{table}

\subsubsection{Connecting positive definite kernels to metrics}\label{app:short_connection_dist_kernel}

Any fixed kernel $k$ on $\cX$ induces a semi-metric\footnote{A semi-metric $d$ satisfies all properties of a metric, except that $d(x, y)=0$ does not imply $x=y$. This can happen because a kernel can map two different points in $\mathcal{X}$ to the same point in
$\mathcal{H}_k$.} $d_k$ on $\cX$ defined for all $x,y\in\cX$ by
\begin{equation}
    \label{eq:kernel_distance}
     d^2_k(x,y) = k(x,x) + k(y,y) - 2k(x,y).
\end{equation}
This holds for all positive-definite kernels by Theorem~\ref{thm:kernel_and_metric}.
In particular, this corresponds to the distance between the embedded points in the RKHS $\cH_k$, that is,
\begin{equation*}
  \|k(x,\cdot) -k(y, \cdot)\|_{\cH_k}=d_k(x,y).
\end{equation*}
The feature embedding $x\mapsto k(x,\cdot)$ induced by a kernel
therefore preserves the distances $d_k$. A useful aspect of kernel
methods, is that they allow a post-analysis based on the embedded
features, see also App.~\ref{app:kernel_mds}.

A partial reverse implication is also true. For a particular type of semi-metric $d$ on $\cX$ (these metrics are called Hilbertian, see App.~\ref{apd:background_kernels_metrics}) it is possible to construct a kernel $k$ on $\cX$ defined for all $x,y\in\cX$ by 
\begin{equation*}
    k(x,y) = -\tfrac{1}{2}d^2(x, y) + \tfrac{1}{2}d^2(x,x_0) + \tfrac{1}{2}d^2(x_0,y),
\end{equation*}
where $x_0\in\cX$ is an arbitrary reference point, such that the distance in the corresponding RKHS $\cH_k$ is $d$.

Kernels can be shifted in such a way that the origin in the induced
RKHS changes but the metric in \eqref{eq:kernel_distance} remains
fixed (see Lemma~\ref{lemma:pd_and_cpd}). A natural origin in the
simplex is given by the point $u=(\frac{1}{p},\ldots,\frac{1}{p})$,
therefore we have shifted all kernels such that $k(u, \cdot)\equiv 0$
and hence correspond to the origin in $\cH_k$. In
App.~\ref{apd:background_kernels_metrics}, we provide a short
overview of the mathematical results that connect kernels and metrics.

\subsection{Weighted kernels - including prior information}\label{apd:weighted_kernels_details}

In this section, we discuss how to include prior knowledge, e.g.\ phylogenetic information, into the simplex kernels. We assume the information is encoded in a matrix $W\in\bR^{p\times p}$ where each element corresponds to a measure of similarity between components. That is, $W_{i,j}$ is large if components $i$ and $j$ are similar (or related) and small otherwise. We assume that $W$ is symmetric, positive semi-definite and all entries in $W$ are non-negative.

The linear kernel and all kernels based on probability distributions have the form 
\begin{equation}
\label{eq:linear_expansion_kernel}
    k(x,y) = \textstyle\sum_{i=1}^p k_0(x^i, y^i)
\end{equation}
and we therefore define the weighted version by
\begin{equation}\label{eq:weightedkernel}
    k_W(x,y) = \sum_{j,\ell=1}^p W_{j,\ell} \cdot k_0(x^j, y^\ell).
\end{equation}
The weighted versions of the remaining kernels are defined individually. A full list of the weighted kernels is given in App.~\ref{apd:list_weighted_kernels}.

\subsubsection{Validity of weighted kernels}\label{apd:validity_weighted_kernels}
In order to use the proposed weighted kernels, we need to ensure that they are indeed positive definite. In the following, we prove this for the weighted versions of the \textit{linear kernel}, the \textit{Hilbertian kernel}, the \textit{Generalized-JS kernel}, the \textit{RBF kernel} and the \textit{Aitchison kernel}. We do not prove it for the \textit{Aitchison RBF kernel} and the \textit{Heat Diffusion kernel} and only note that they appear to be positive definite from our empirical evaluations.

We begin by showing that the kernel defined in \eqref{eq:weightedkernel} is positive definite whenever $k_0:[0,1]\times [0,1]\rightarrow \bR$ is positive definite. To see this, fix $x_1,\ldots,x_n\in\simp^{p-1}$ and $\alpha\in\bR^n$ and denote by $K_W\in\bR^{n\times n}$ the kernel Gram-matrix based on $x_1,\ldots,x_n$ and kernel $k_W$. Then,
\begin{align*}
    \alpha^{\top}K_W\alpha
    &=\sum_{i,r=1}^n\sum_{j,\ell=1}^p\alpha_{i}\alpha_{r}W_{j,\ell}k_0(x_{i}^{j},x_{r}^{\ell})\\
    &=\sum_{j,\ell=1}^p W_{j,\ell}\left(\sum_{i,r=1}^n\alpha_{i}\alpha_{r}k_0(x_{i}^{j},x_{r}^{\ell})\right).
\end{align*}
Since $k_0$ is positive definite, it holds that $\sum_{i,r}\alpha_{i}\alpha_{r}k_0(x_{i}^{j},x_{r}^{\ell})\geq 0$ and hence $\alpha^{\top}K_W\alpha\geq 0$ since all entries in $W$ are non-negative.

We now go over the individual weighted kernels and argue that they are positive definite.

\begin{itemize}
\item\textbf{Linear kernel }
Since $\bR$ is a Hilbert space with the inner product $xy$ which induces the $|x-y|$ it follows that the squared distance $d^2_{\text{Linear}}(x, y) \coloneqq (x-y)^2$ is Hilbertian as well. Applying Theorem~\ref{thm:hilbert_d_negativetype} we know that the distance is of negative type.
Thus, based on the one-dimensional squared linear distance $d^2_{\text{Linear}}$, we apply Theorem~\ref{thm:kernel_and_metric} with $x_0 = \frac 1 p$ to construct the following positive definite kernel $k_0$ defined for all $x,y \in [0,1]$ by
\begin{align*}
    k_0(x, y) \coloneqq - \tfrac{1}{2}(x-y)^2 + \tfrac{1}{2} (x - \tfrac 1 p)^2 + \tfrac{1}{2} (\tfrac 1 p -y)^2 = xy - \tfrac{x}{p} - \tfrac{y}{p} + \tfrac{1}{p^2}.
\end{align*}
Comparing this with our weighted linear kernel in App.~\ref{apd:list_weighted_kernels}, we see that the weighted linear kernel has the form \eqref{eq:weightedkernel} and is therefore positive definite by the above argument.

\item\textbf{Hilbertian kernel } As shown by \citet{hein2005hilbertian} the distance $d_{\text{Hilbert}}: \bR_+ \times \bR_+ \to \bR$ defined for all $x,y\in\bR_+$ by
\begin{equation*}
    d^2_{\text{Hilbert}}(x,y) = \frac{2^{\frac{1}{b}}\Big[x^{a} +  y^{a}\Big]^{\frac{1}{a}} - 2^{\frac{1}{a}}\Big[x^{b} + y^{b}\Big]^{\frac{1}{b}}}{2^{\frac{1}{a}} - 2^{\frac{1}{b}}}
\end{equation*}
is a Hilbertian metric on $\bR_+$. Applying Theorem~\ref{thm:kernel_and_metric} with $x_0 = \frac{1}{p}$ results in a positive definite kernel $k_0$ that when combined as in \eqref{eq:weightedkernel} results in the proposed weighted Hilbertian kernels in App.~\ref{apd:list_weighted_kernels}. Therefore, we have shown that the weighted Hilbertian kernels are positive definite as long as $W$ has non-negative entries.

\item\textbf{Generalized-JS kernel } Similarly the weighted Generalized-JS kernels in App.~\ref{apd:list_weighted_kernels} can all be decomposed as in \eqref{eq:weightedkernel} with a one-dimensional kernels $k_0$ on $[0,1]$. \citet{topsoe2003jenson} show that all these $k_0$ can be generated using Theorem~\ref{thm:kernel_and_metric} with $x_0 = \frac{1}{p}$ based on Hilbertian metrics. Hence, all weighted Generalized-JS kernels are positive definite as long as $W$ has non-negative entries.

\item \textbf{Aitchison kernel } To show that the weighted Aitchison kernel (defined in App.~\ref{apd:list_weighted_kernels} is positive definite, we first define the mapping $\Phi:\simp^{p-1}\rightarrow\bR^p$ by $\Phi(x)\coloneqq \frac{x+c}{g(x+c)}$. Then, the weighted Aitchison kernel is given by
\begin{equation*}
    k(x,y)=\Phi(x)^{\top}W\Phi(y).
\end{equation*}
Since $W$ is symmetric and positive semi-definite there exists $M\in\bR^{p\times p}$ such that $W=M^{\top}M$. Therefore, for any $\alpha\in\bR^{n}$ and $x_1,\ldots,x_n\in\simp^{p-1}$ it holds that
\begin{equation*}
    \sum_{i,r}\alpha_i\alpha_r k(x_i, x_r)=\sum_{i,r}\alpha_i\alpha_r(M\Phi(x_i))^{\top}M\Phi(x_r)\geq 0.
\end{equation*}
Hence, $k$ is positive definite.
\item \textbf{RBF kernel } Using the symmetry of $W$ the weighted RBF kernel can be expressed as follows
\begin{align*}
    k(x,y) &=
    \exp\Big(-\frac{1}{\sigma^2}\sum_{j, \ell=1}^p W_{j, \ell}(x^j-y^\ell)^2\Big) \\
    &=\underbrace{\exp\Big(-\frac{1}{\sigma^2}\sum_{j, \ell=1}^p W_{j, \ell}(x^j)^2\Big)\exp\Big(-\frac{1}{\sigma^2}\sum_{j, \ell=1}^p W_{j, \ell}(y^j)^2\Big)}_{\eqqcolon A(x,y)}\\
    &\quad \cdot\underbrace{\exp\Big(\frac{2}{\sigma^2}\sum_{j, \ell=1}^p W_{j, \ell}x^j y^\ell\Big)}_{\eqqcolon B(x,y)}
\end{align*}
The function $A$ is a positive definite kernel since it is the inner-product of a feature mapping. The function $B$ can be shown to be a kernel by considering the Taylor expansion of the exponential function and using that sums and limits of positive definite kernels are again positive definite together with the fact that $W$ is positive semi-definite. Therefore, the weighted RBF kernel is positive definite.
\end{itemize}

\subsection{UniFrac-Weighting}\label{apd:unifrac}

In this section, we show how prior information based on the UniFrac-Distance \citep{lozupone2005unifrac} can be encoded into a weight matrix $W\in\bR^{p\times p}$. Depending on the application at hand different distances can be used in a similar way. The UniFrac-Distance is a $\beta$-diversity measure that uses phylogenetic information to compare two compositional samples $x,y\in\simp^{p-1}$. Each element of the sample is hereby placed on a phylogenetic tree. The distance between both samples is computed via quantification of overlapping branch length, that is,
\begin{equation*}
    \text{UniFrac-Distance}(x,y) = \frac{\text{sum of unshared branch length of $x$ and $y$}}{\text{sum of all tree branch length of $x$ and $y$}}\in[0,1].
\end{equation*}
Based on the UniFrac-Distance, we define two similarity matrices $M^A,M^B\in[0,1]^{p\times p}$ for all $i,j\in\{1,\ldots,p\}$ by
\begin{align*}
    M^{A}_{i,j}&\coloneqq 1-\text{UniFrac-Distance}(e_i, e_j),\\
    M^{B}_{i,j}&\coloneqq \sum_{\ell=1}^{p}\text{UniFrac-Distance}(e_i, e_\ell)\cdot\text{UniFrac-Distance}(e_j, e_\ell),
\end{align*}
where $e_i,e_j\in\simp^{p-1}$ with $1$ on the $i$-th and $j$-th coordinate, respectively. $M^A$ and $M^B$ are two options of encoding the UniFrac-Distance as a similarity. $M^B$ is positive semi-definite by construction, while this is not true for $M^A$ and should be checked empirically. We recommend using $M^A$ whenever it is positive semi-definite.

We then construct the weight matrix $W^{\text{UniFrac}}\in\bR^{p\times p}$ by scaling $M^{*}$ such that the diagonal entries are one, that is,
\begin{equation*}
    W^{\text{UniFrac}}\coloneqq D M^{*} D,
\end{equation*}
where $D=\text{diag}(\sigma_1,\ldots,\sigma_p)$, with $\sigma_i=1/\sqrt{M^*_{i,i}}$. Since by construction the matrix $M^*$ has its largest values on the diagonal, this weight matrix takes values in $[0,1]$. Moreoever, by construction it remains symmetric and positive semi-definite.


\section{Details and additional results for experiments in Sec.~\ref{sec:experiments}} \label{apd:experiment_details}

\subsection{List of datasets}\label{apd:details_on_data}

\FloatBarrier

\begin{longtable}{p{0.2\textwidth}p{0.15\textwidth}p{0.2\textwidth}p{0.12\textwidth}p{0.2\textwidth}}
    \toprule
    \textbf{Name} & \textbf{Reference} & \textbf{Prediction tasks} & \textbf{Dim}\newline ($n\times d$) & \textbf{Additional}\newline \textbf{preprocessing} \\ \midrule
    americangut & \cite{mcdonald2018american} & classification:\newline - diabetes vs healthy & $882\times 327$ & UK subpopulation \& prev./abun. \newline filtering\\
    camp & \cite{berry2020natural} & classification:\newline - parasite infected vs healthy & $270\times 622$ & none\\
    centralpark & \cite{ramirez2014biogeographic} & regression:\newline - ph level of soil & $580\times 1498$ & prev./abun. \newline filtering \\
    cirrhosis & \cite{qin2014alterations} & classification:\newline - cirrhosis vs healthy & $130\times 444$ & aggregated to \newline species \& \newline prev./abun. \newline filtering \\
    genomemed & \cite{baxter2016microbiota} & classification:\newline - cancer vs non-cancer & $490\times 335$ & none\\
    hiv & \cite{rivera2018balances} & regression:\newline - CD4+ cell counts & $152\times 282$ & none\\
    rmp & \cite{vandeputte2017quantitative} & classification:\newline - Crohn's disease vs healthy & $95\times 351$ & none\\
    tara & \cite{sunagawa2020tara} & regression:\newline - ocean salinity & $136\times 2407$ & prev./abun. \newline filtering \\
    \midrule
    gevers\_ileum & \cite{gevers2014treatment} & classification:\newline - Crohn's disease vs healthy & $140\times 446$ & none\\
    gevers\_pcdai-ileum & \cite{gevers2014treatment} & regression: \newline - PCDAI scores & $67\times 446$ & none\\
    gevers\_pcdai-rectum & \cite{gevers2014treatment} & regression: \newline - PCDAI scores & $51\times 446$ & none\\
    hmp\_gastro-oral & \citetalias{human2012structure} & classification: \newline  - gastrointestinal vs oral & $2070\times 1218$ & none\\
    hmp\_sex & \citetalias{human2012structure} & classification: \newline - female vs male & $180\times 1218$ & none\\
    hmp\_stool-tongue-paired & \citetalias{human2012structure} & classification: \newline - stool vs tongue dorsum & $404\times 1218$ & none\\
    hmp\_sub-supra-paired & \citetalias{human2012structure} & classification: \newline - sub vs supragingival plaque & $408\times 1218$ & none\\
    karlsson\_impaired-diabetes & \cite{karlsson2013gut} & classification: \newline - impaired vs type 2 diabetes & $101\times 3758$ & none\\
    karlsson\_normal-diabetes & \cite{karlsson2013gut} & classification: \newline - normal vs type 2 diabetes & $96\times 3758$ & none\\
    kostic & \cite{kostic2012genomic} & classification: \newline - healthy vs tumor & $172\times 409$ & none\\
    qin2012 & \cite{qin2012metagenome} & classification: \newline - healthy vs type 2 diabetes & $124\times 2526$ & none\\
    qin2014 & \cite{qin2014alterations} & classification: \newline - cirrhosis vs healthy & $130\times 2579$ & none\\
    ravel\_black-hispanic & \cite{ravel2011vaginal} & classification: \newline - black vs Hispanic & $199\times 305$ & none\\
    ravel\_nugent-category & \cite{ravel2011vaginal} & classification: \newline - nugent score high vs low & $342\times 305$ & none\\
    ravel\_nugent-score & \cite{ravel2011vaginal} & regression: \newline - nugent score & $388\times 305$ & none\\
    ravel\_ph & \cite{ravel2011vaginal} & regression: \newline - vaginal pH & $388\times 305$ & none\\
    ravel\_white-black & \cite{ravel2011vaginal} & classification: \newline - white vs black & $200\times 305$ & none\\
    sokol\_healthy-cd & \cite{morgan2012dysfunction} & classification: \newline - healthy vs Crohn's disease & $74\times 367$ & none\\
    sokol\_healthy-uc & \cite{morgan2012dysfunction} & classification: \newline - healthy vs Ulcerative colitis  & $59\times 367$ & none\\
    turnbaugh & \cite{turnbaugh2007human} & classification: \newline - lean vs obese & $142\times 232$ & none\\
    yatsunenko\_baby-age & \cite{yatsunenko2012human} & regression: \newline - infant age & $49\times 1544$ & none\\
    yatsunenko\_malawi-venezuela & \cite{yatsunenko2012human} & classification: \newline - Malawi vs Venezuela & $54\times 1544$ & none\\
    yatsunenko\_sex & \cite{yatsunenko2012human} & classification: \newline male vs female & $129\times 1544$ & none\\
    yatsunenko\_usa-malawi & \cite{yatsunenko2012human} & classification: \newline - US vs Malawi & $150\times 1544$ & none
    \\
    \bottomrule
    \caption{List of microbiome datasets used to benchmark \KernelBiome. Datasets in below vertical line are taken from the MLRepo \citep{vangay2019microbiome} without further processing. The cirrhosis dataset is a processed version of qin2014. Whenever prevalence/abundance filtering (prev./abun. filtering) is applied it means that only taxa that appear in $25\%$ of the samples and with a median non-zero count of $5$.}\label{tab:datasets}
\end{longtable}

\FloatBarrier

\subsection{Weighting matrix for weighted \KernelBiome}

The weight matrix $W^{\text{UniFrac}}$ for the cirrhosis dataset \citep{qin2014alterations} and the centralpark dataset \citep{ramirez2014biogeographic} are presented as heatmaps in Fig.~\ref{fig:unifrac_visualization}. Using our proposed weighted kernels (see App.~\ref{apd:list_weighted_kernels}) with the UniFrac-based weight matrix $W^{\text{UniFrac}}$ is different from incorporating the UniFrac-distance via kernel convolution as proposed by \citet{Zhao2015}.

\begin{figure}[hbt!]
\centering
\includegraphics[width=\textwidth]{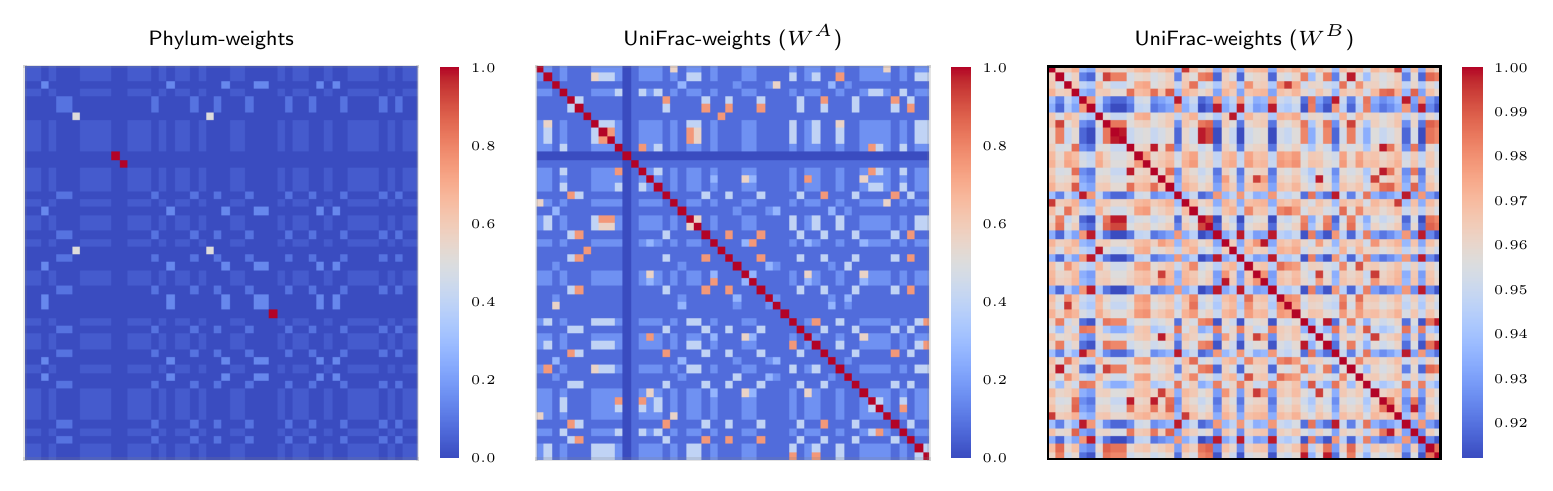}
\includegraphics[width=\textwidth]{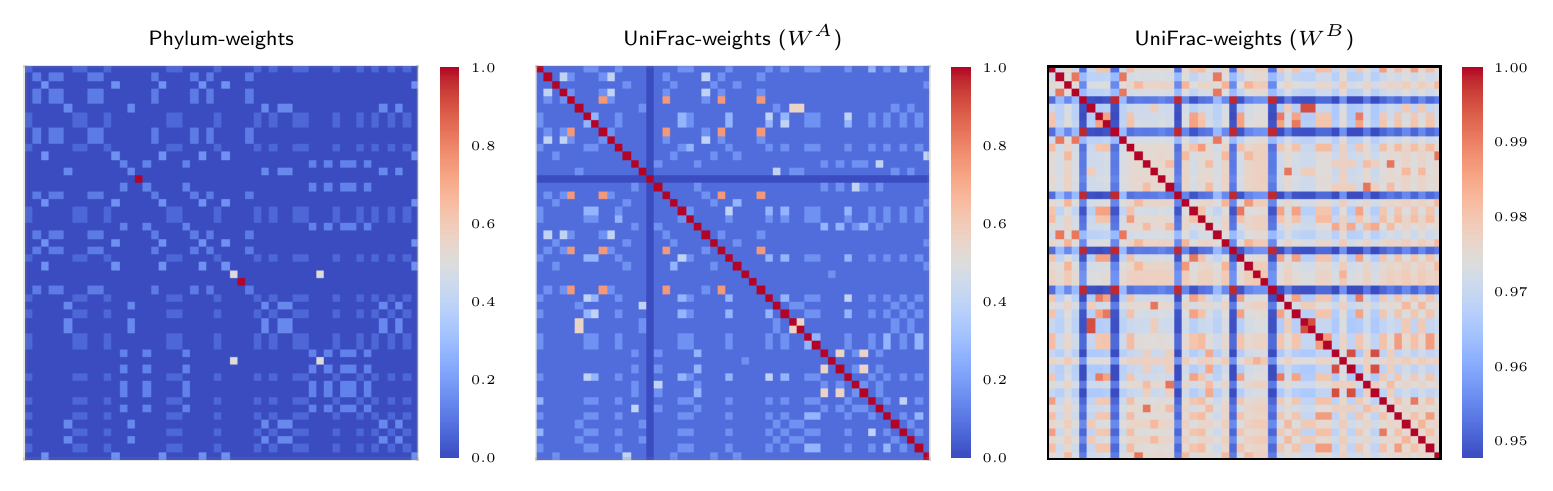}
\caption{Visualization of the phylum-weights and the two UniFrac-weights $W^{A} = DM^AD$ for $W^{B} = DM^BD$, based on the $50$ pre-screened species (see App.~\ref{apd:experiment_details}). Upper panel: cirrhosis dataset. Lower panel: centralpark dataset.}
\label{fig:unifrac_visualization}
\end{figure}

\FloatBarrier

\subsection{Remaining experiment results on public datasets} \label{apd:remaining_results}

\subsubsection{Prediction performance}

Here we provide the remaining results for the prediction experiments. ROC curves for the classification tasks accompanying the boxplot results are given in Fig.~\ref{fig:real_datasets_scores}. Prediction scores for the MLRepo datasets in Fig.~\ref{fig:mlrepo_scores} are given in Fig.~\ref{fig:curated_roc}, and ROC curves accompanying the boxplot results can be found in Fig.~\ref{fig:mlrepo_roc}. We can see that \KernelBiome achieves the best results for most of the tasks. For classification tasks, \KernelBiome performs competitively both in terms of accuracy and ROC curves.\footnote{For \texttt{SVM-RBF}, \texttt{KB-Aitchison} and \KernelBiome, the ROC curves are based on the estimated probabilities computed in the sklearn-package. We observed a slight mismatch between these predicted probabilities and predicted classes in some of the examples, which is due to a bug \url{https://github.com/scikit-learn/scikit-learn/issues/13211}. We therefore recommend putting more emphasis on the accuracy plots.}

\begin{figure}[htb]
    \centering
    \includegraphics[width=\textwidth]{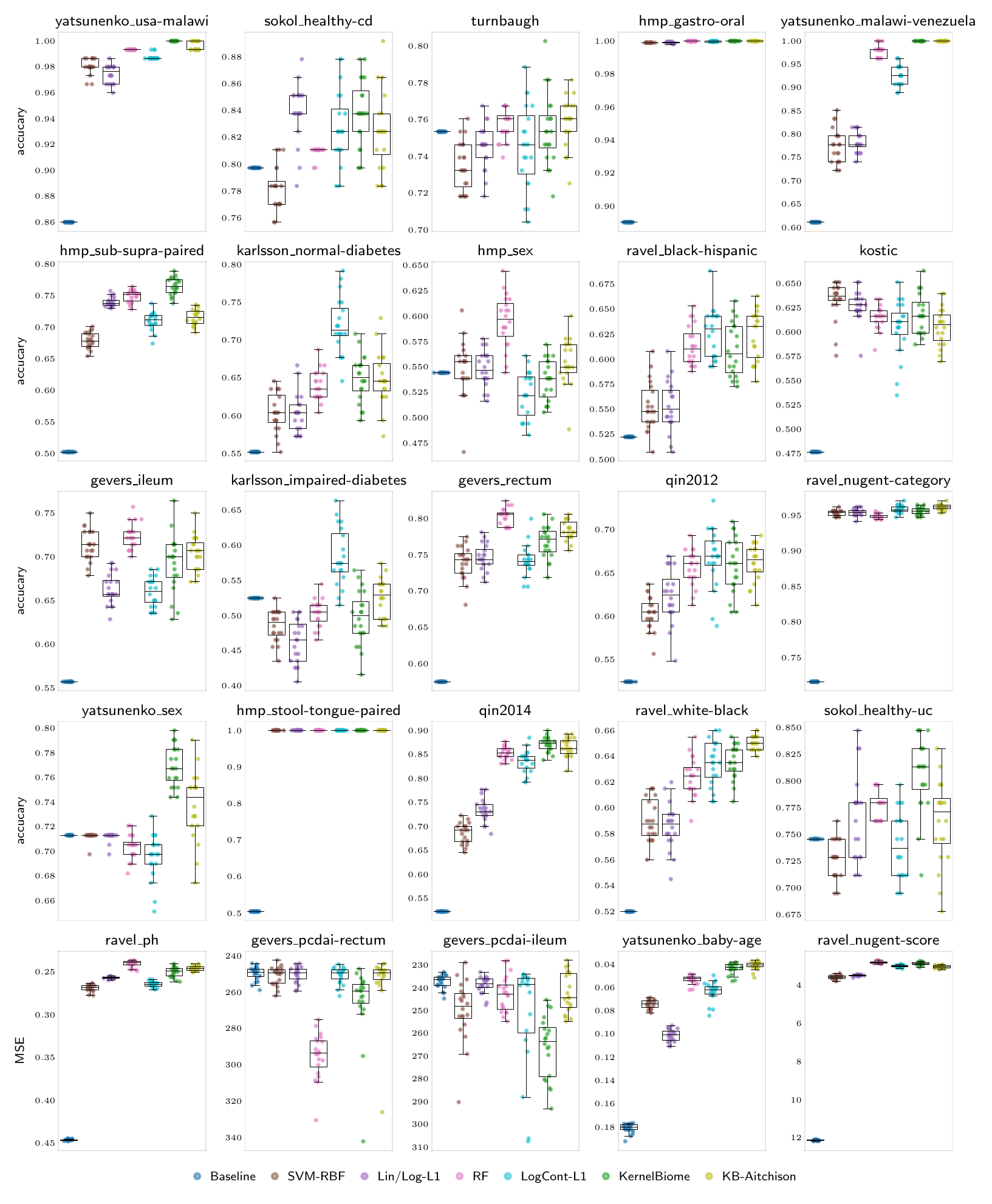}
    \caption{Comparison of predictive performance on the MLRepo datasets (20 classification tasks and 5 regression tasks) based on a 10-fold train/test split.}
    \label{fig:mlrepo_scores}
\end{figure}

\begin{figure}[htb]
    \centering
    \includegraphics[width=\textwidth]{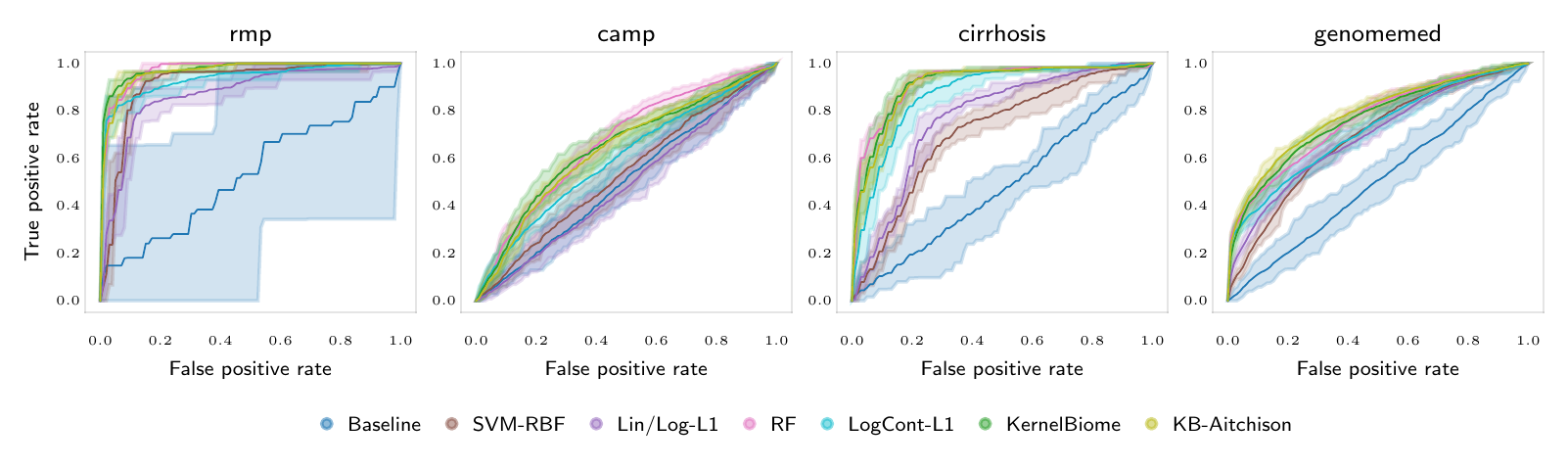}
    \caption{ROC curves for the 4 public classification datasets shown in the main paper. The solid curve is the average curve from the $20$ different 10-fold CV, and the shaded area is the $95\%$ confidence band.}
    \label{fig:curated_roc}
\end{figure}

\begin{figure}[htb]
    \centering
    \includegraphics[width=\textwidth]{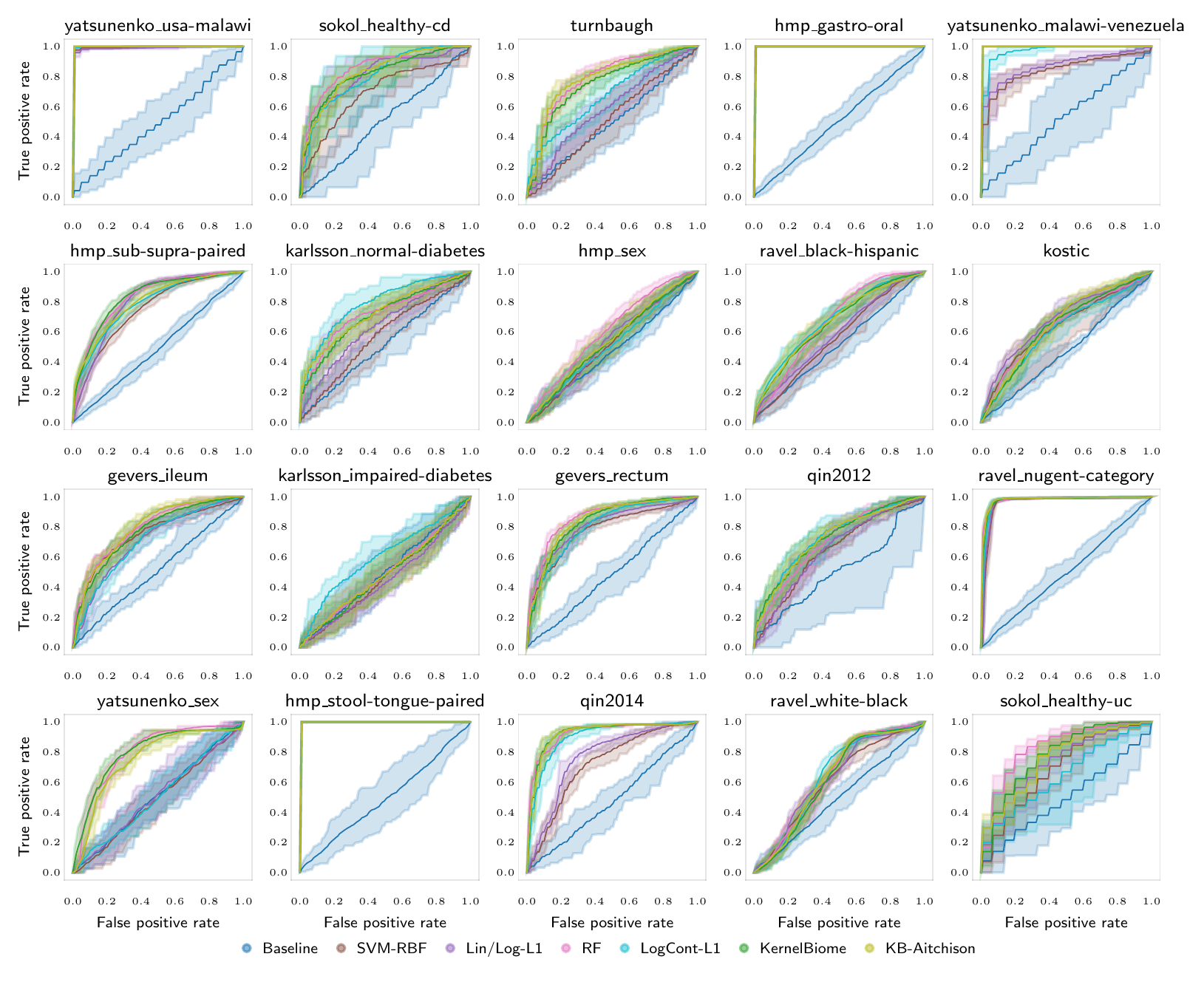}
    \caption{ROC curves for the 20 classification datasets from MLRepo. The solid curve is the average curve from the $20$ different 10-fold CV, and the shaded area is the $95\%$ confidence band.}
    \label{fig:mlrepo_roc}
\end{figure}

\FloatBarrier

\subsubsection{Model analysis}

Here we include the remaining model analysis results for the cirrhosis and centralpark datasets. As in the main paper, we screened both data sets to only include the $50$ taxa with the highest absolute CFIs by \KernelBiome with Aitchison kernel. Kernel PCA plots for the cirrhosis dataset is given in Fig.~\ref{fig:kpca_cirrhosis} and the CFI values for the centralpark dataset are given in Fig.~\ref{fig:cfi_centralparksoil}.

Furthermore, we also provide the missing circle plots here.
The extended version of Fig.~\ref{fig:model_analysis} (c) in the main text with long labels is given in Fig.~\ref{fig:cfi_cirrhosis_circle_aitchison}. Fig.~\ref{fig:cfi_cirrhosis_circle_generalizedJS} is the circle plot for the cirrhosis dataset based on the generlized-JS kernel. Fig.~\ref{fig:cfi_centralparksoil_circle_aitchison} is the circle plot for the centralpark dataset based on the Aitchison kernel.

\begin{figure}[htb]
    \centering
    \includegraphics[width=\textwidth]{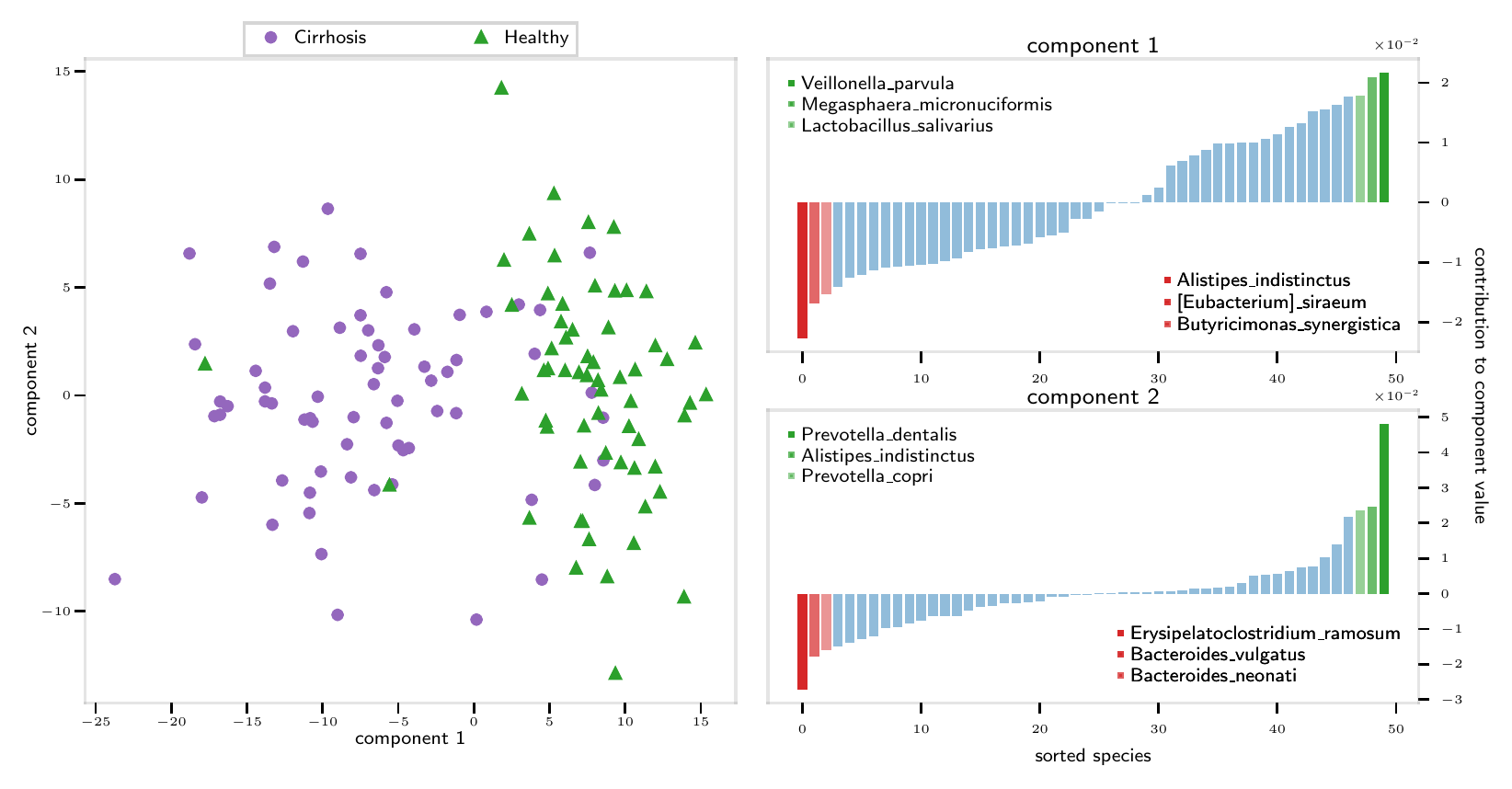}
    \caption{Kernel PCA plot and contributions of the $50$ taxa to component 1 and 2 sorted from the most negative contribution to the most positive contribution.}
    \label{fig:kpca_cirrhosis}
\end{figure}

\begin{figure}[htb]
    \centering
    \includegraphics{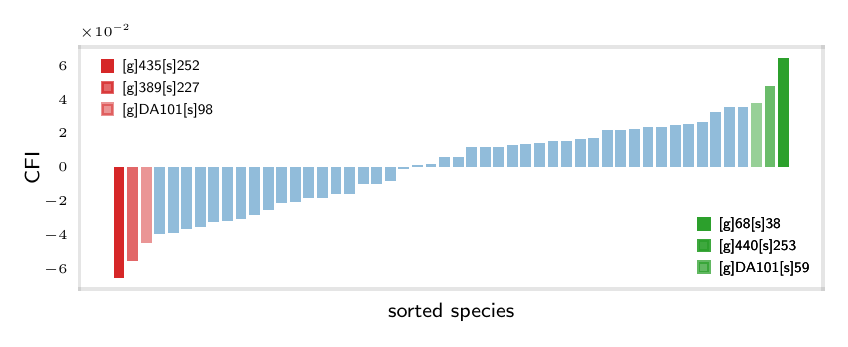}
    \caption{CFI values for the $50$ taxa sorted from the most negative contribution to the most positive contribution to the response.}
    \label{fig:cfi_centralparksoil}
\end{figure}

\begin{figure}[htb]
\resizebox{\textwidth}{!}{
    \begin{tikzpicture}
    \node (E) at (4.5, -10.5) {\includegraphics[width=0.3\linewidth]{figures/legend_cirrhosis.pdf}};
    \node (D) at (4, -7) {
    \includegraphics[width=.37\linewidth]{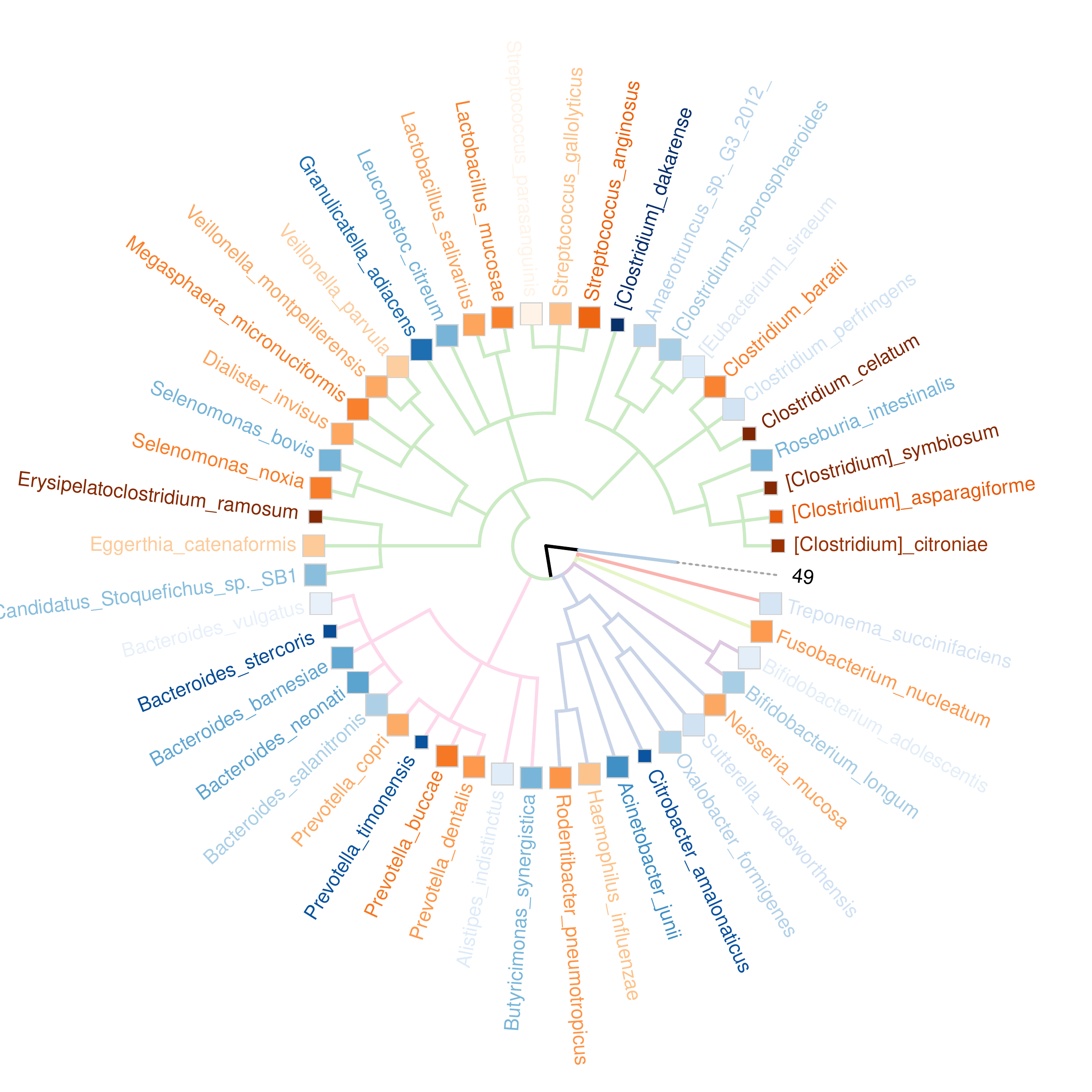} 
    \hspace{1cm}
    \includegraphics[width=.37\linewidth]{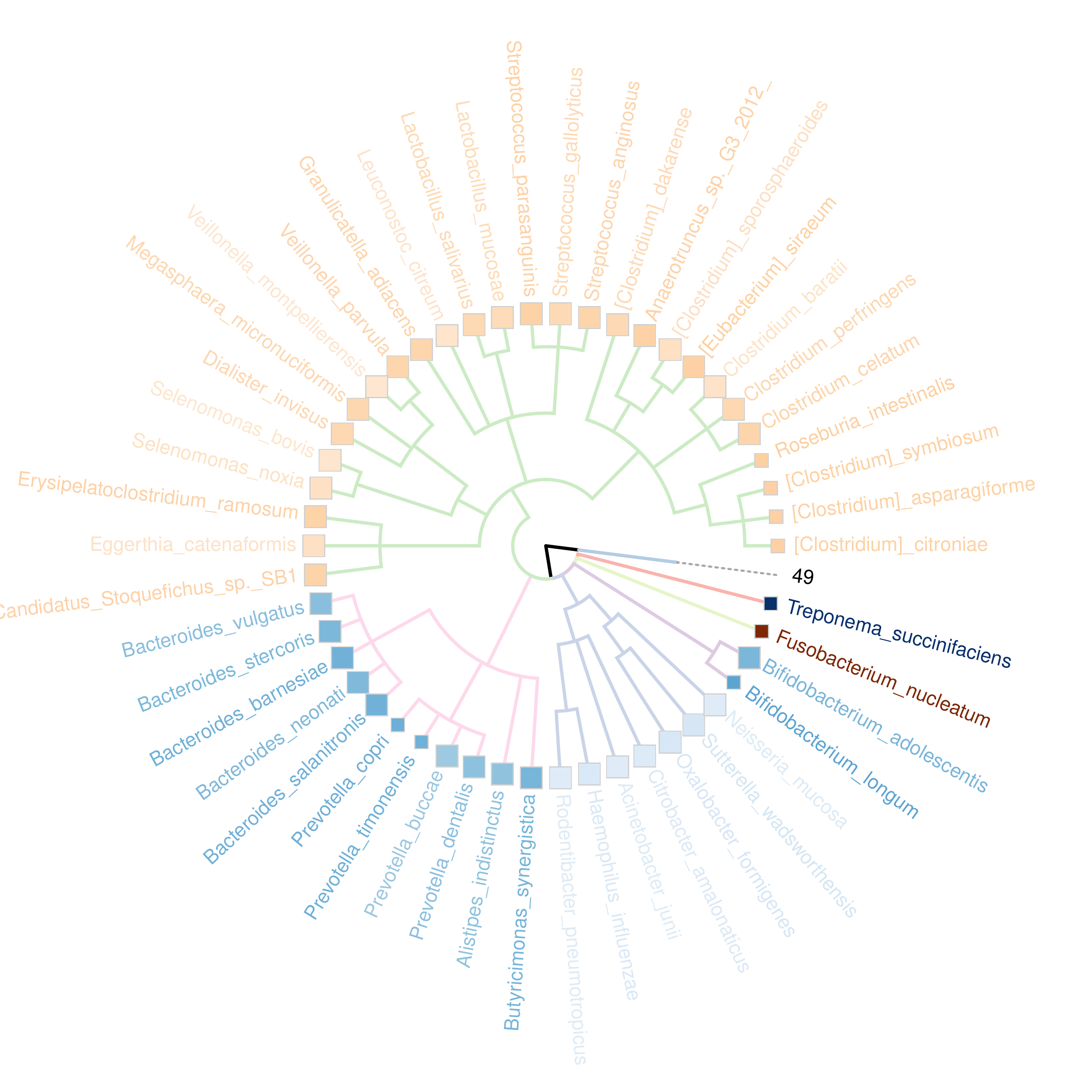}
    };
    \node (F) at (4, -14) {
    \includegraphics[width=.37\linewidth]{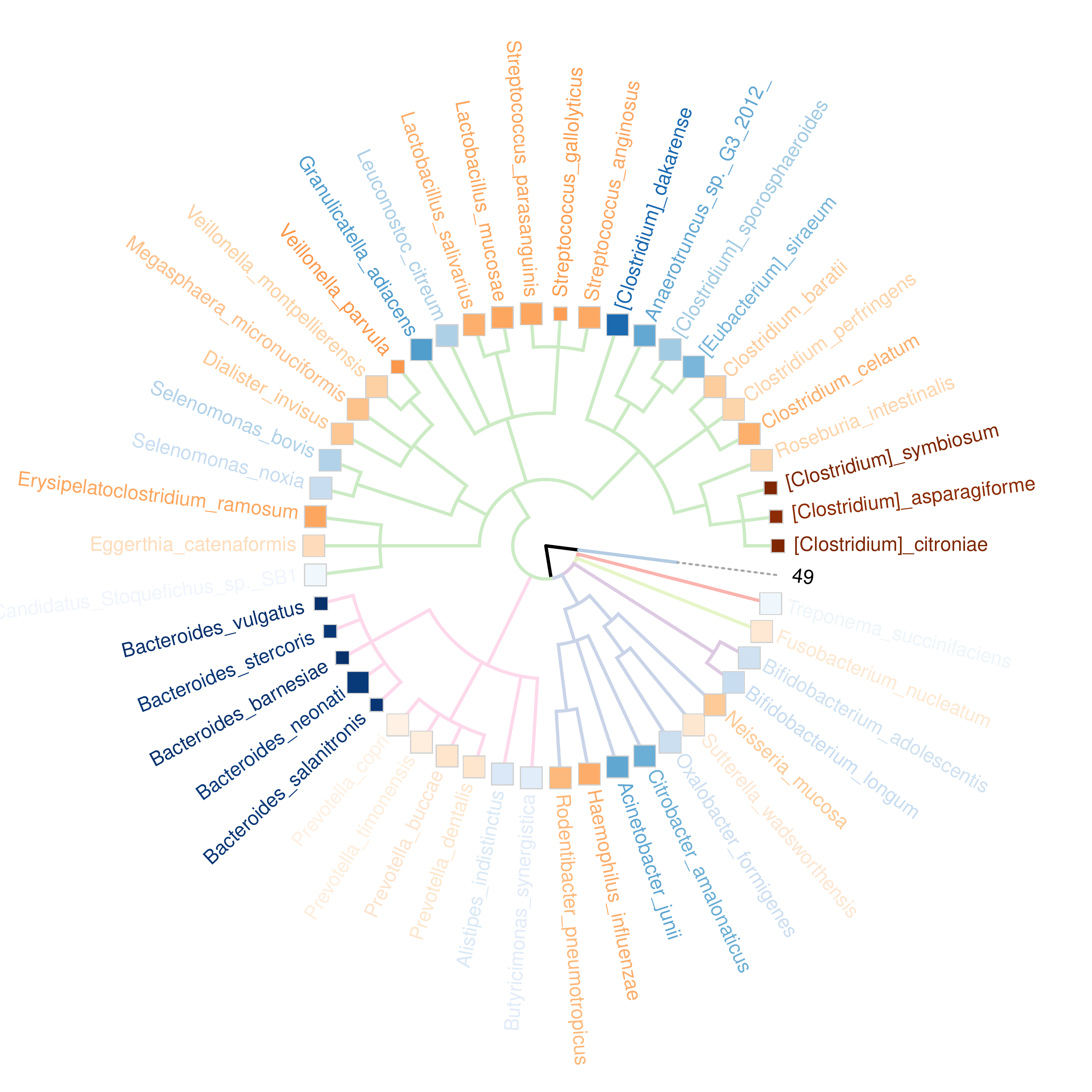}
    \hspace{1cm}
    \includegraphics[width=.37\linewidth]{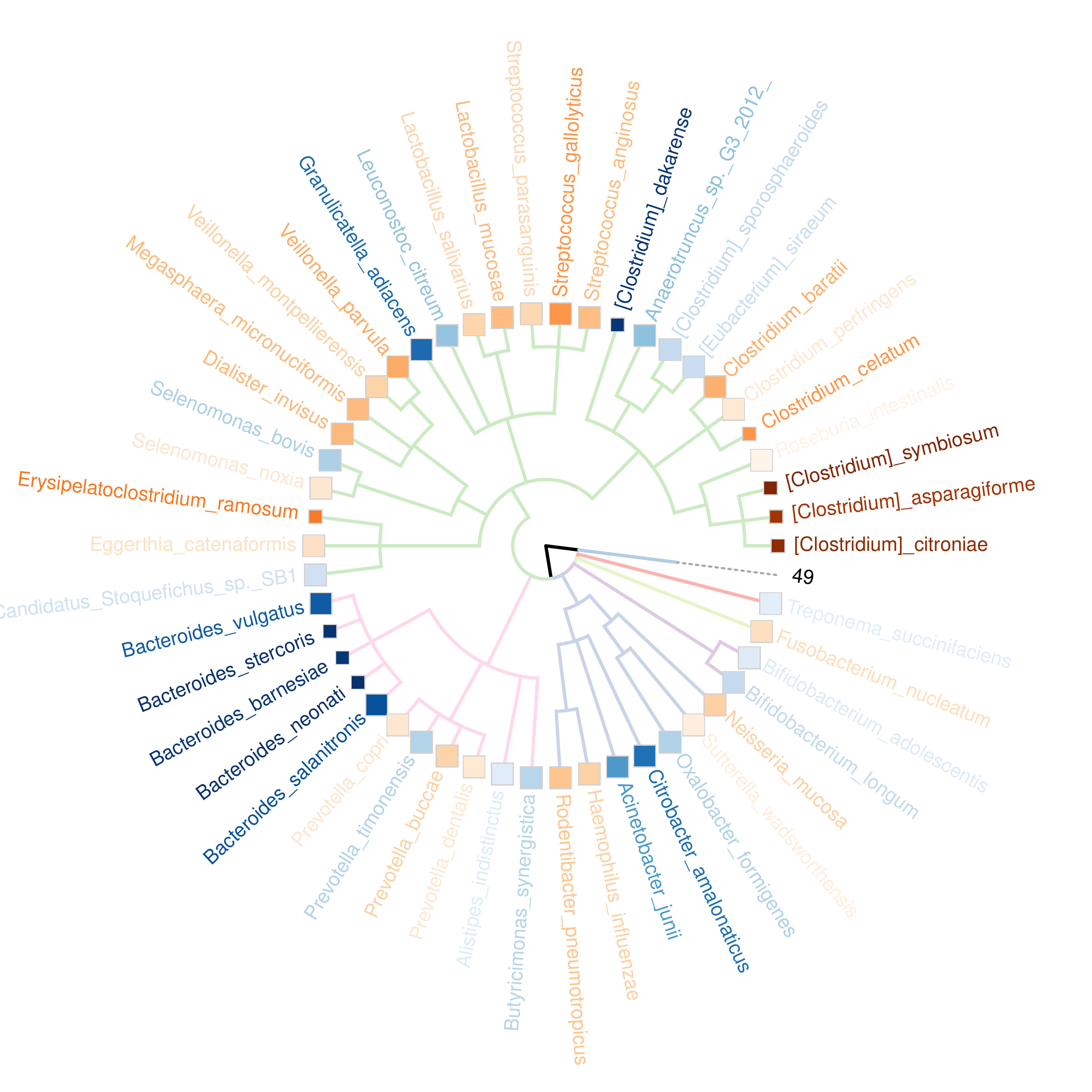}
    };
    \node (clab1) at (-2,-4) {\small unweighted};
    \node (clab2) at (11,-4) {\small phylum-weighted};
    \node (clab3) at (-2,-11) {\small UniFrac-weighted W\textsuperscript{B}};
    \node (clab3) at (11,-11) {\small UniFrac-weighted W\textsuperscript{A}};
    \end{tikzpicture}}
\caption{Scaled CFI values for cirrhosis dataset where a darker color shade of the name of the microbiota signifies a stronger (positive resp. negative) feature influence (Aitchison Kernel).}
\label{fig:cfi_cirrhosis_circle_aitchison}
\end{figure}

\begin{figure}[htb]
\resizebox{\textwidth}{!}{
\begin{tikzpicture}
    \node (E) at (4.5, -10.5) {\includegraphics[width=0.3\linewidth]{figures/legend_cirrhosis.pdf}};
    \node (D) at (4, -7) {
    \includegraphics[width=.37\linewidth]{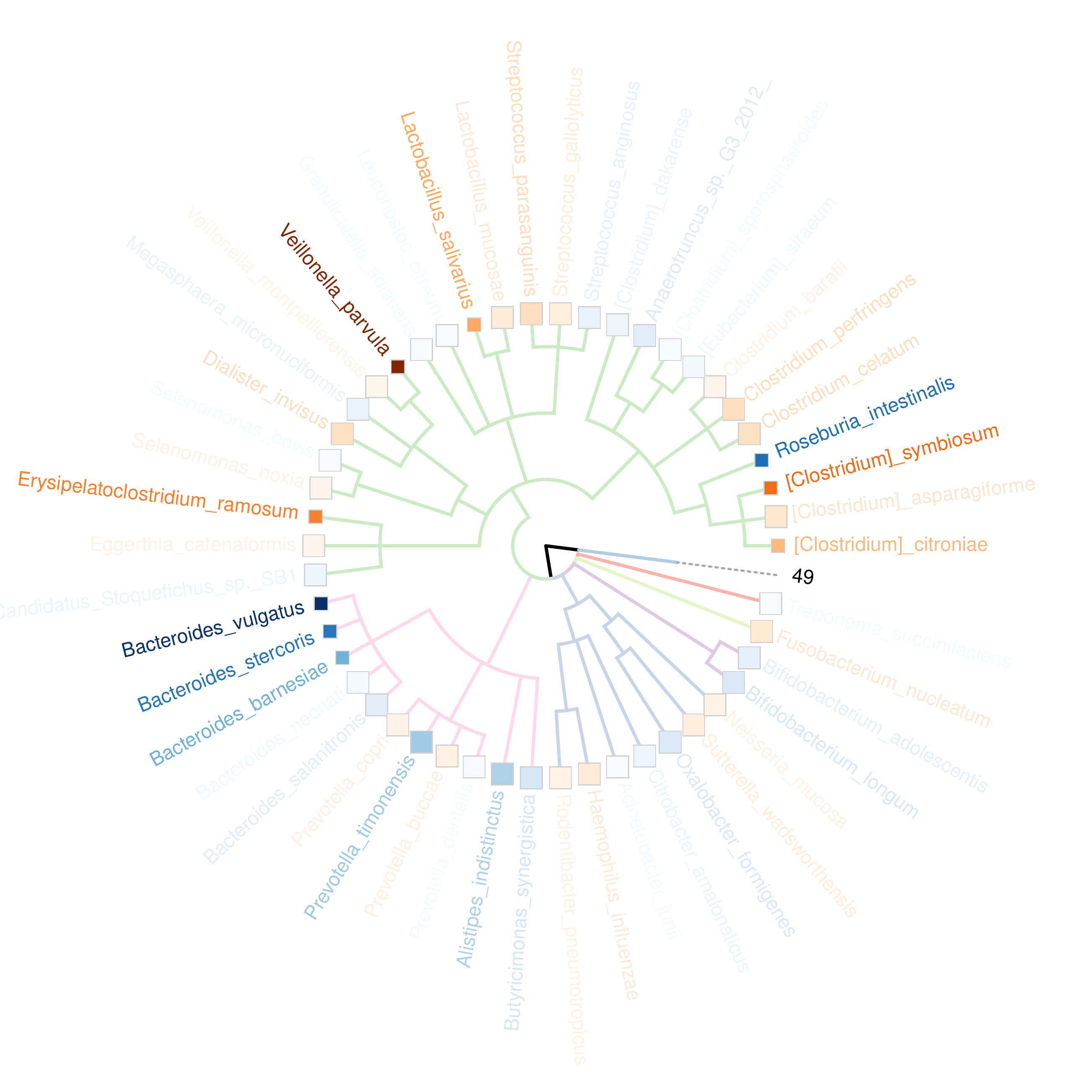} 
    \hspace{1cm}
    \includegraphics[width=.37\linewidth]{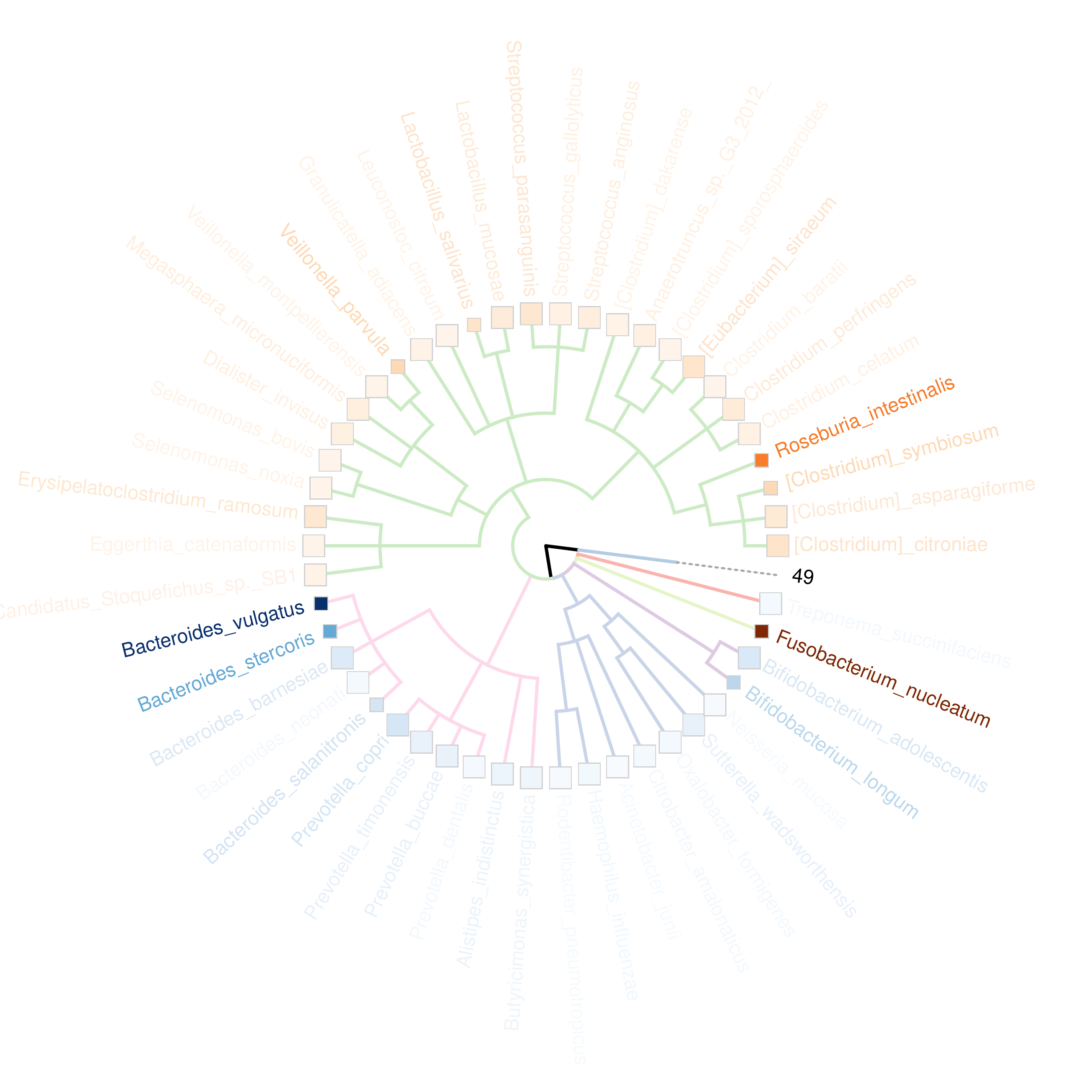}
    };
    \node (F) at (4, -14) {
    \includegraphics[width=.37\linewidth]{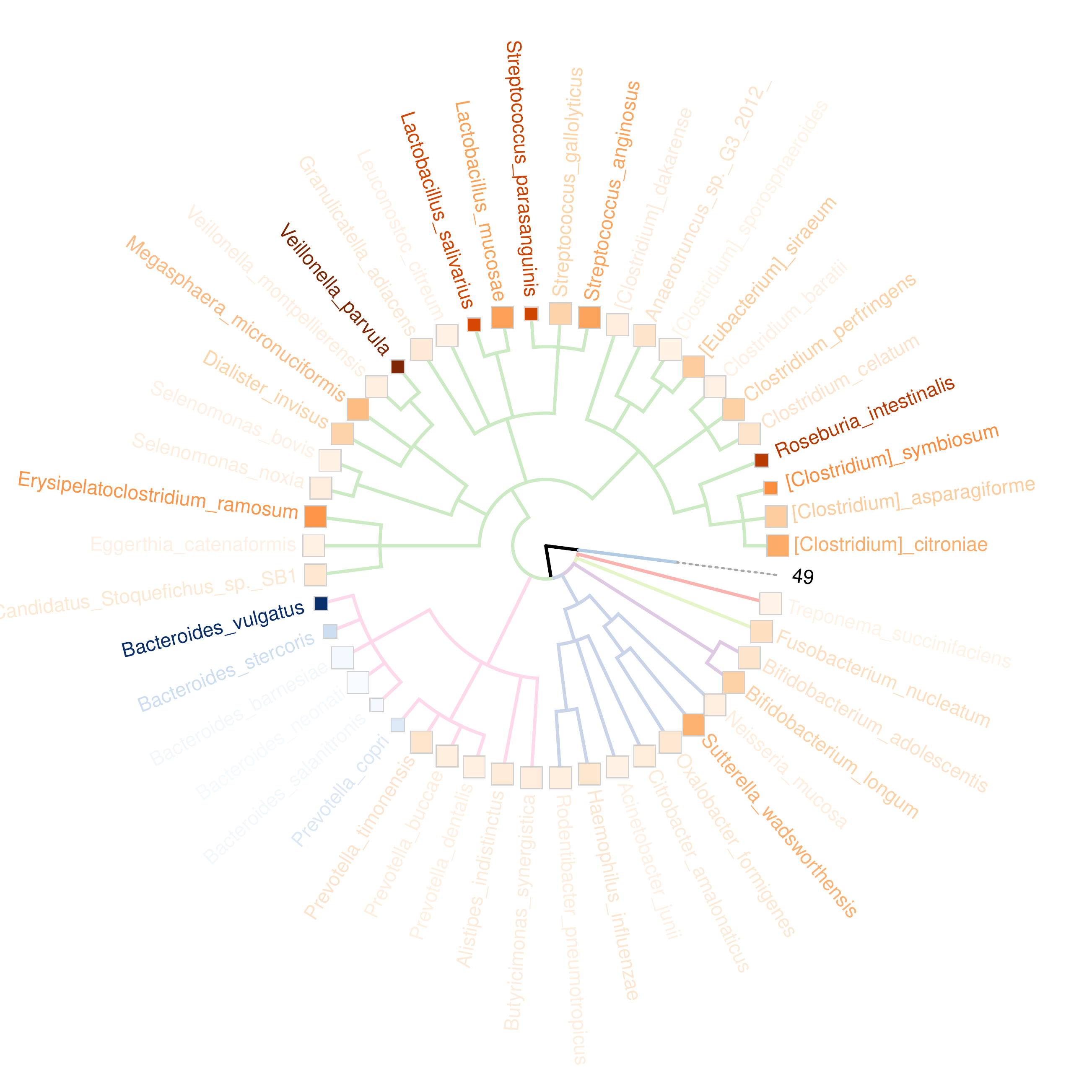}
    \hspace{1cm}
    \includegraphics[width=.37\linewidth]{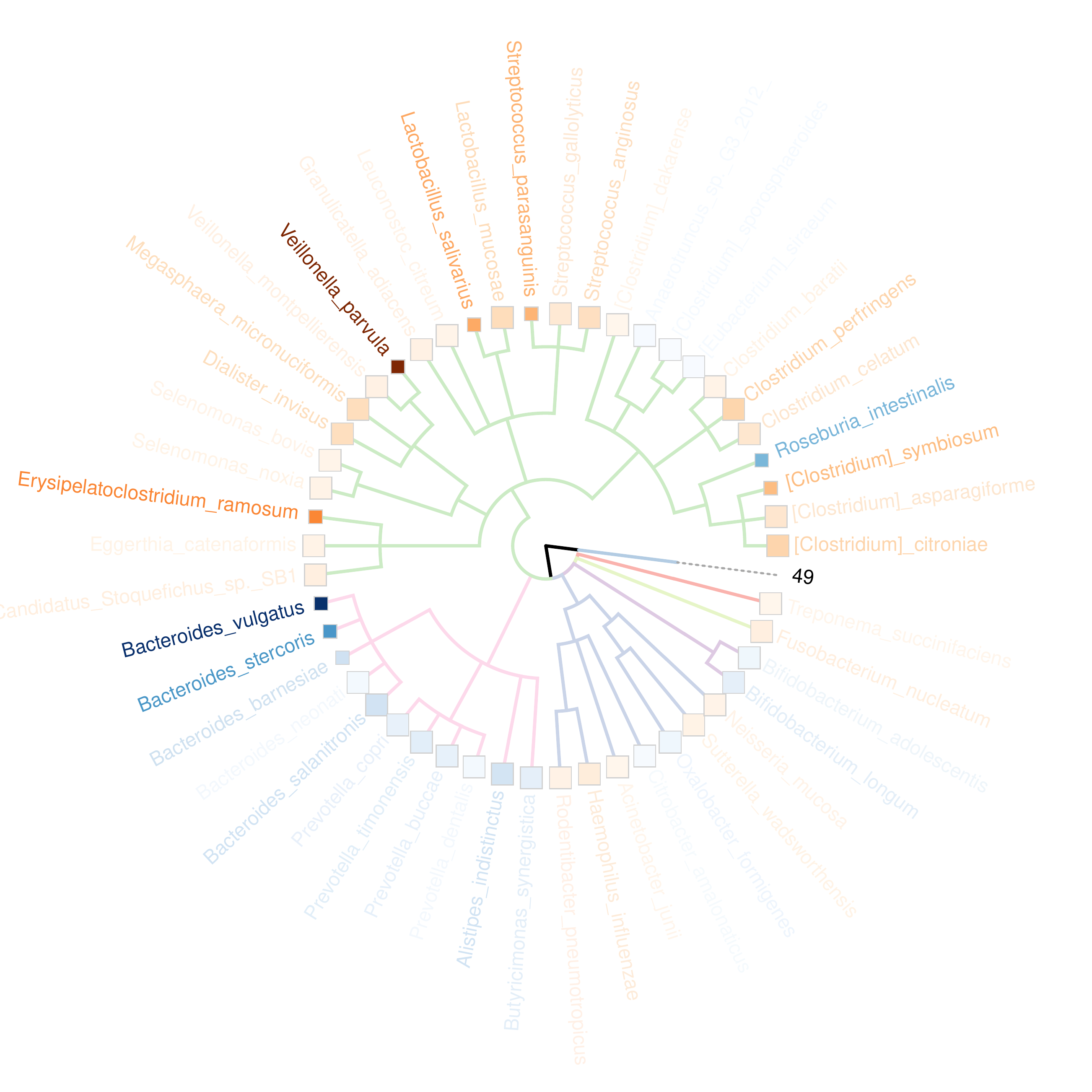}
    };
    \node (clab1) at (-2,-4) {\small unweighted};
    \node (clab2) at (11,-4) {\small phylum-weighted};
    \node (clab3) at (-2,-11) {\small UniFrac-weighted W\textsuperscript{B}};
    \node (clab3) at (11,-11) {\small UniFrac-weighted W\textsuperscript{A}};
\end{tikzpicture}}
\caption{Scaled CFI values for cirrhosis dataset where a darker color shade of the name of the microbiota signifies a stronger (positive resp. negative) feature influence (Generalized JS Kernel).}\label{fig:cfi_cirrhosis_circle_generalizedJS}
\end{figure}

\begin{figure}[htb]
\resizebox{\textwidth}{!}{
    \begin{tikzpicture}
    \node (E) at (4.5, -10.5) {\includegraphics[width=0.3\linewidth]{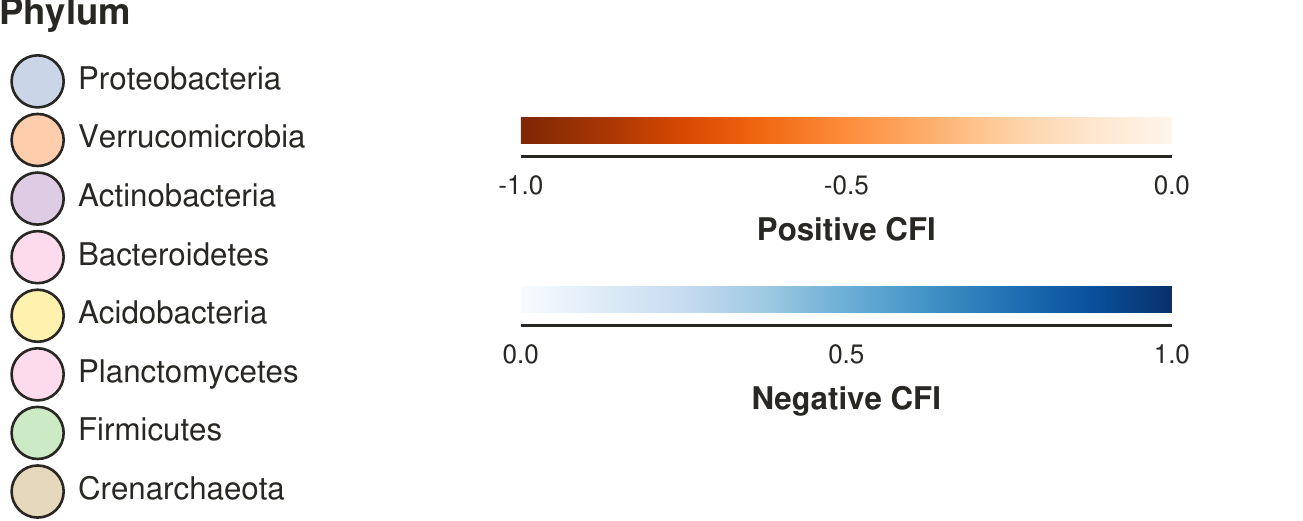}};
    \node (D) at (4, -7) {
    \includegraphics[width=.37\linewidth]{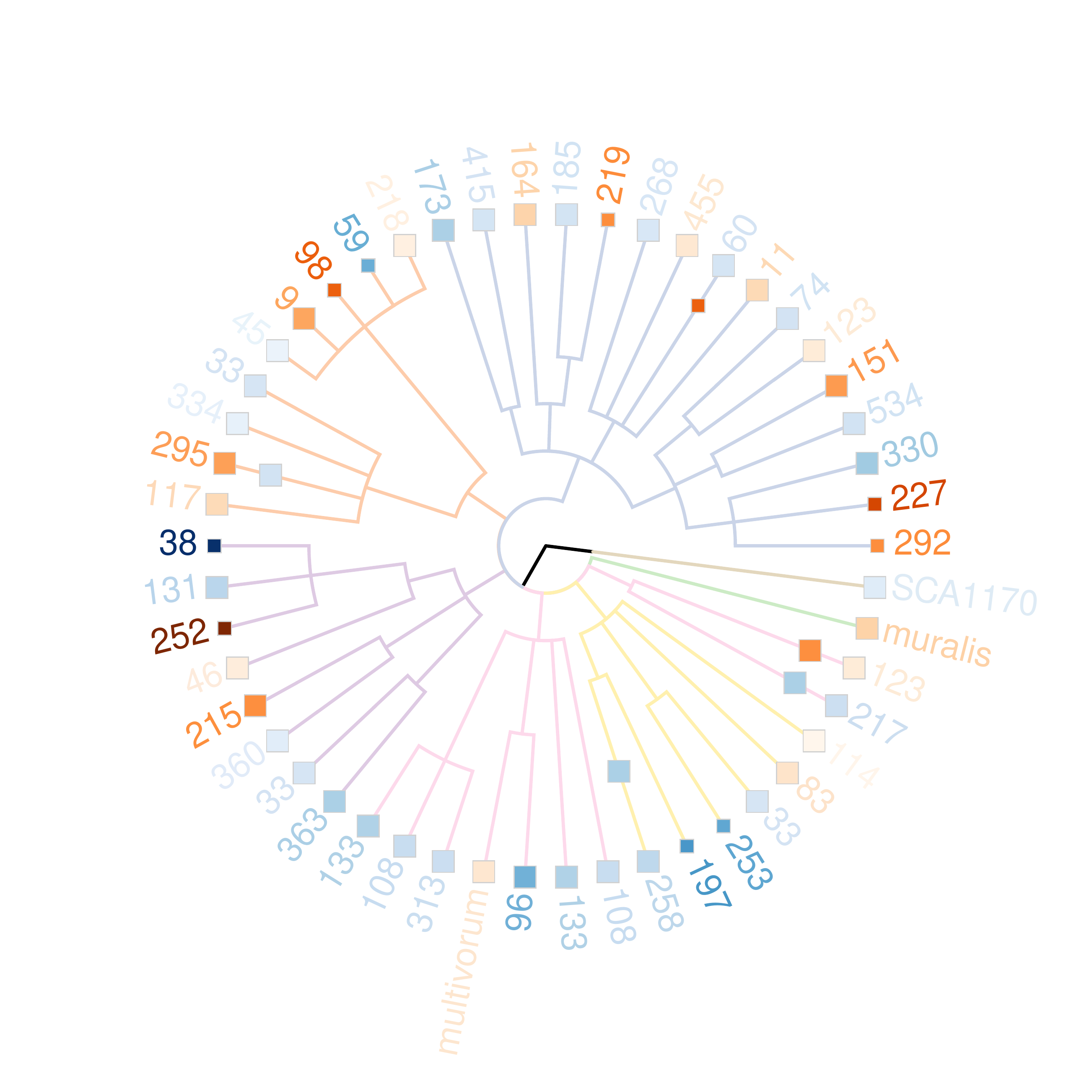} 
    \hspace{1cm}
    \includegraphics[width=.37\linewidth]{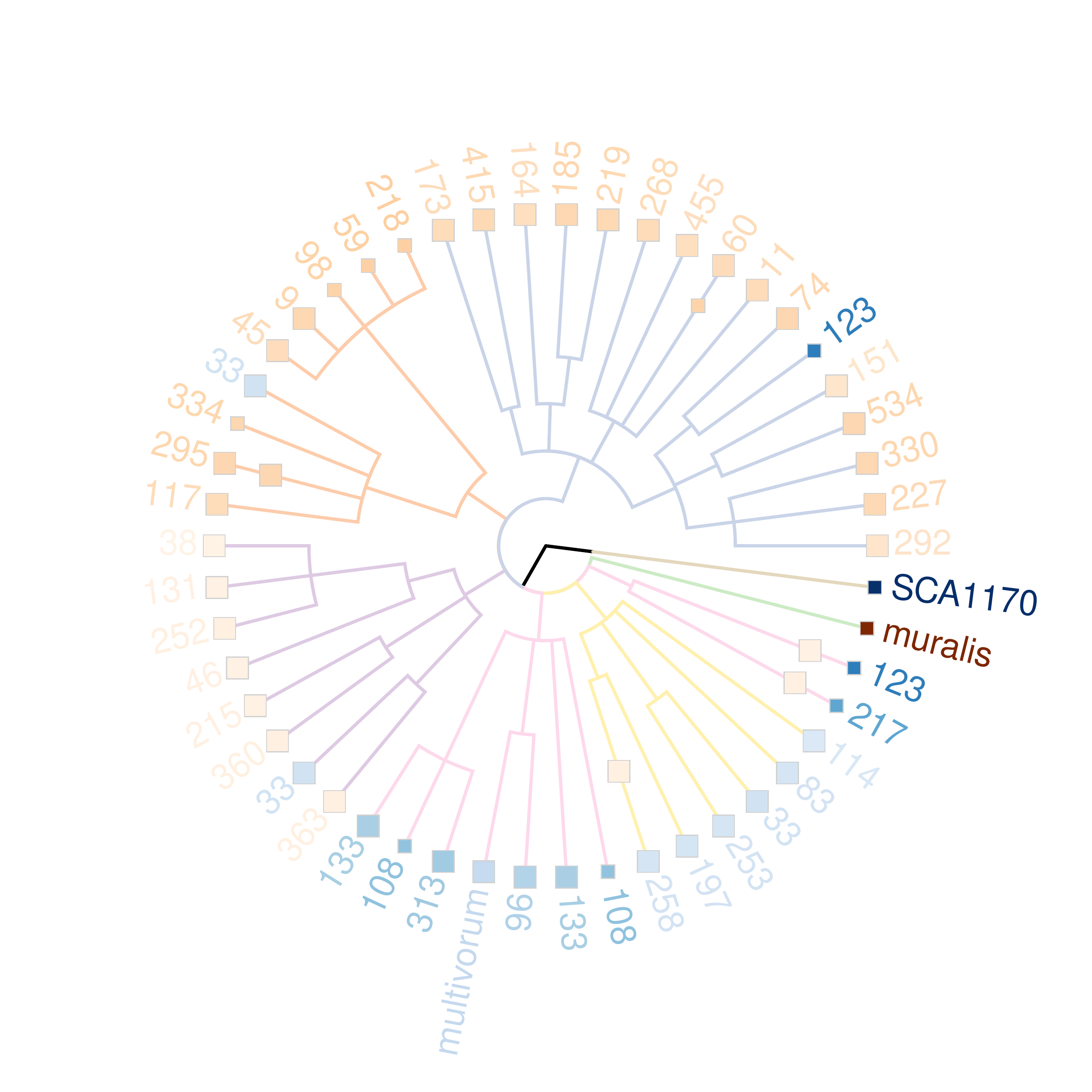}
    };
    \node (F) at (4, -14) {
    \includegraphics[width=.37\linewidth]{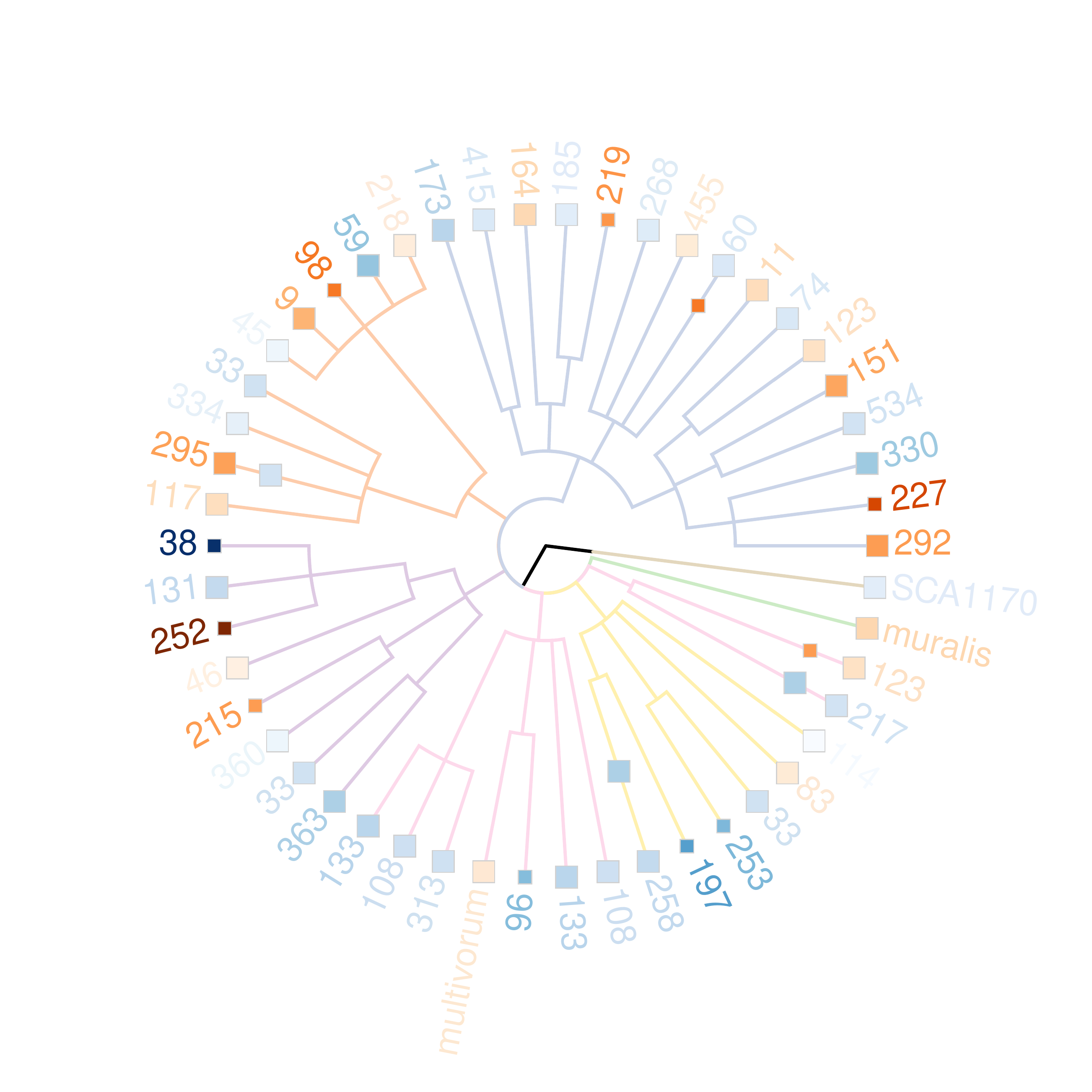}
    \hspace{1cm}
    \includegraphics[width=.37\linewidth]{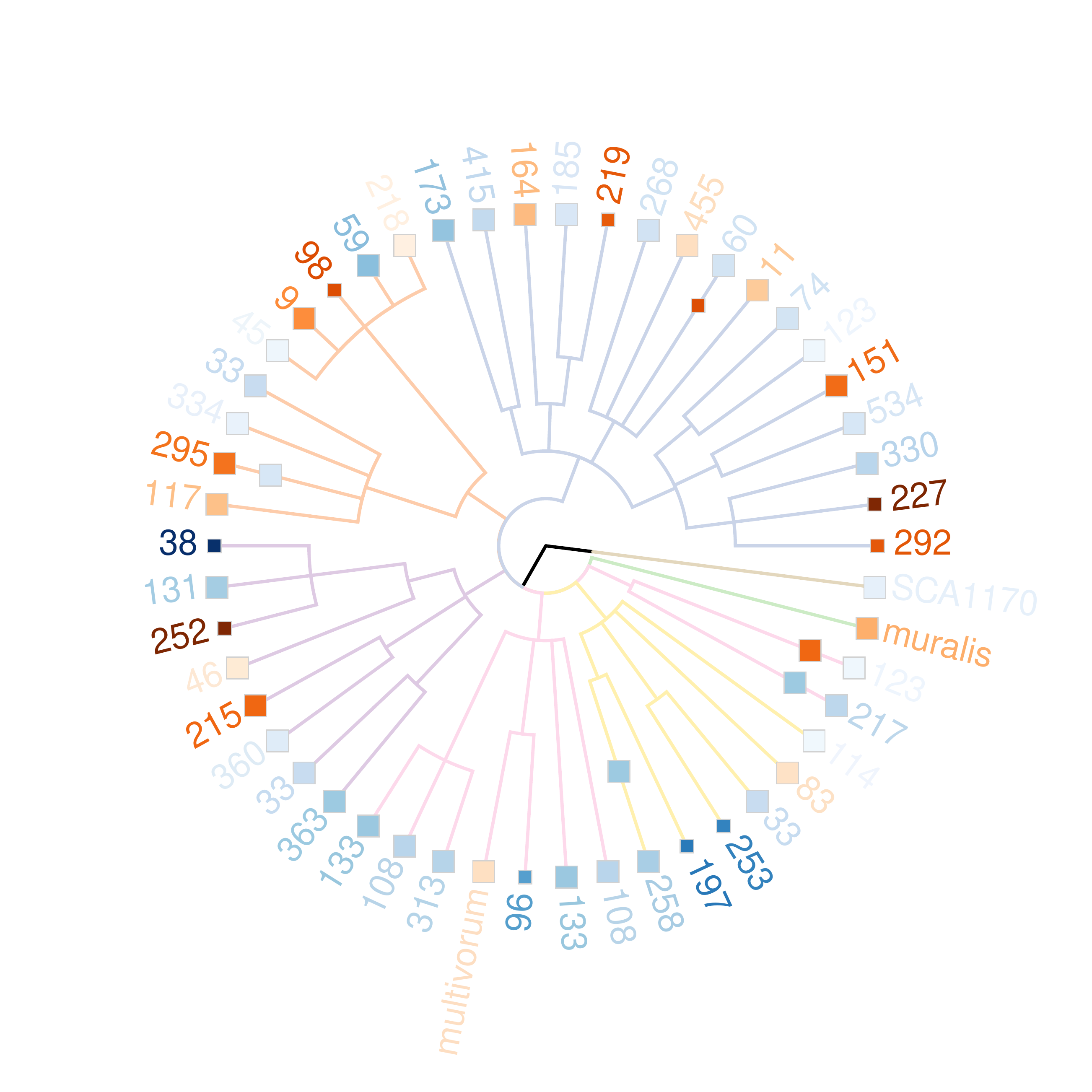}
    };
    \node (clab1) at (-2,-4) {\small unweighted};
    \node (clab2) at (11,-4) {\small phylum-weighted};
    \node (clab3) at (-2,-11) {\small UniFrac-weighted W\textsuperscript{B}};
    \node (clab3) at (11,-11) {\small UniFrac-weighted W\textsuperscript{A}};
    \end{tikzpicture}
}
\caption{Scaled CFI values for centralpark dataset where a darker color shade of the name of the microbiota signifies a stronger (positive resp. negative) feature influence (Aitchison Kernel).}\label{fig:cfi_centralparksoil_circle_aitchison}
\end{figure}

\FloatBarrier

\section{Additional experiments with simulated data}\label{apd:additional_experiments}

\subsection{Consistency of CPD and CFI}\label{apd:consistency_simulation}

We illustrate the consistency of CPD and CFI from Theorem~\ref{thm:consistency} based on \KernelBiome with the following example. Let $k_{\text{tv}}$ be the total variation kernel and consider the function
$$f: x\mapsto 100 \cdot k_{\text{tv}}(z, x)$$ with 
\begin{align*}
    z &=(0.06544714, 0.08760064, 0.17203408, 0.07502236, 0.1642615, \\
    &\qquad 0.03761901, 0.18255478, 0.13099514, 0.08446536) \in\simp^{8}
\end{align*}
being a fixed and randomly selected point. 
Furthermore, we generate an i.i.d.\ dataset $(X_1,Y_1),$ $\ldots, (X_n,Y_n)$ based on the following 2 step generative model.
\begin{itemize}
    \item \textbf{Step 1:} Generate a random variable $\tilde{X}=(\tilde{X}^1,\ldots,\tilde{X}^9)$ such that the three blocks $(\tilde{X}^1, \tilde{X}^2, \tilde{X}^3)$, $(\tilde{X}^4, \tilde{X}^5, \tilde{X}^6)$, and $(\tilde{X}^7, \tilde{X}^8, \tilde{X}^9)$ are i.i.d.\ from $\text{LogNormal}(0, \Sigma)$, where 
    $\Sigma = \begin{pmatrix}
    1 & 0.25 & -0.25 \\
    0.25 & 1 & 0.25 \\
    -0.25 & 0.25 & 1
    \end{pmatrix}$. Then, $X_i$ is constructed by normalizing $\tilde{X}$, that is, $X_i=\tilde{X}/\sum_{j=1}^{9}\tilde{X}^j$. The block structure adds non-trivial correlation structure between the compositional components.
    \item \textbf{Step 2:} Generate $Y_i$ based on $X_i$ by
    $$Y_i = f(X_i) + \epsilon_i$$ with $\epsilon_i \iid \mathcal{N}(0,1)$.
\end{itemize}
Based on one such dataset, we estimate the CFI and CPD for a fitted \KernelBiome estimator (using kernel ridge regression and default settings), and compare the estimates against the population CFI and CPD calculated from the true function $f$. In Fig.~\ref{fig:consistency}, we report the mean squared deviations (MSD) for both CFI and CPD based on $100$ such datasets for each sample size.

\begin{figure}[hbt!]
    \centering
    \includegraphics{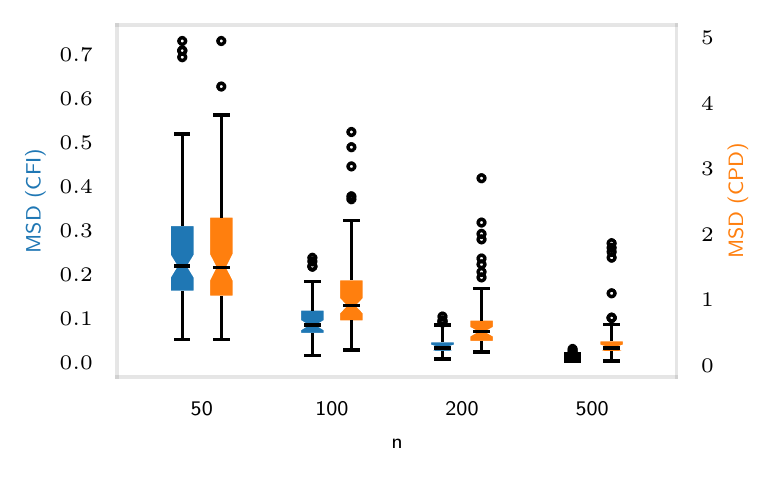}
    \caption{MSD of estimated CFI and CPD using \KernelBiome estimator based on $100$ random datasets for each sample size. For CPD, we calculate the true and estimated CPD based on $100$ evenly spaced grid points within the range of $[0.001, 0.999]$ and the reported MSD is the average MSD over the $9$ components. As the sample size $n$ increases the CFI and CPD estimates based on \KernelBiome converge to the true population quantities.}
    \label{fig:consistency}
\end{figure}

\subsection{Comparing CFI and CPD with permutation importance and partial dependence plots}\label{apd:cfi_cpd_pi_pdp}

Two common approaches to assess the importance of individual features are permutation importance (PI) and partial dependency plot (PDP). PI of the $j$-th feature is defined as the mean difference between the baseline mean squared error of a fitted model and the average mean squared error after permuting the $j$-th feature column a certain number of times. PDP is used to describe how individual features contribute to a fitted model. For the $j$-th feature, it describes its contribution by the function $z\mapsto \bE[\hat{f}(X^1,\ldots,X^{j-1}, z, X^j,\ldots,X^p)]$, where $\hat{f}$ is the fitted model. Both PI and PDP can be misleading when used with compositional covariates. 

In this section, we illustrate this based on two examples. In both cases, the proposed adjusted measures CFI and CPD remain correct, while the PI and PDP are incorrect. Consider the two functions
\begin{equation}\label{eq:twoexamples}
\begin{split}
    &f_1: x\mapsto 10 x^1 + 10 x^2 \\
    &f_2: x\mapsto \frac{1-x^2-x^3}{1-x^3}.
\end{split}
\end{equation}
For $f_1$, changes in all coordinates affect the function value due to the simplex constraint. For $f_2$, only changes in $x^1$ and $x^2$ affect the function value but not changes in $x^3$. This is because on the simplex $f_2(x)=\frac{x^1}{x^1+x^2}$. An importance measure should therefore associate a non-zero value to $x^3$ for $f^1$ and zero to $x^3$ for $f^2$.

We generate $200$ i.i.d.\ observations  $X_1\ldots,X_{100}$ with $X_i\overset{d}{=}\tilde{X}_i/\sum_{j=1}^3\tilde{X}_i^{j}$ for $\tilde{X}_i\overset{\text{i.i.d.}}{\sim}\text{LogNormal}(0, \text{Id}_3)$\footnote{$\text{LogNormal}(\mu, \Sigma)$ denotes the log-normal distribution with location parameter $\mu$ and scale parameter $\Sigma$. $\text{Id}_3$ denotes the 3-dimensional identity matrix.} and compute PI, PDP, CFI and CPD for each of the two functions. The results are given in Table~\ref{tab:cfi_vs_pi} and Fig.~\ref{fig:cpd_vs_pdp}.

As expected, the CFI and also CPD correctly capture the behavior of the two functions. However, PI and PDP are incorrect in both cases: For $f_1$ the variable $x^3$ shows no effect both with PI and PDP and for $f_2$ the variable $x^3$ is falsely assigned a strong negative effect even though it does not affect the function value at all.
In Table~\ref{tab:cfi_vs_pi}, we have additionally computed the relative influence (RI) given by $\bE[\tfrac{d}{dx^j}\hat{f}(X)]$ due to \citet{friedman2001greedy}. It has the same problems as PI as it does not take into account the simplex structure.

\begin{table}
\centering
\begin{tabular}{crrrrrr}
\toprule
    & \multicolumn{3}{c}{$f_1$} & \multicolumn{3}{c}{$f_2$} \\ \cmidrule(lr){2-4}\cmidrule(lr){5-7} 
    & $x^1$ & $x^2$ & $x^3$ & $x^1$ & $x^2$ & $x^3$ \\ \midrule
CFI & 0.85 & 0.87 & \textbf{-1.72} & 1.94 & -1.94 & \textbf{0.00} \\ 
RI & 3.76 & 2.99 & \textbf{0.00} & 0.00 & -4.72 & \textbf{-4.40} \\
PI  & 11.66 & 5.76 & \textbf{0.00} & 0.00 & 28.98 & \textbf{24.72} \\ \bottomrule
\end{tabular}
\caption{CFI, RI and PI for the two functions $f_1$ and $f_2$ defined in \eqref{eq:twoexamples}. Only CFI correctly attributes the effect of $x^3$ (marked in bold).} \label{tab:cfi_vs_pi} 
\end{table}

\begin{figure}
\centering
\includegraphics[width=0.6\textwidth]{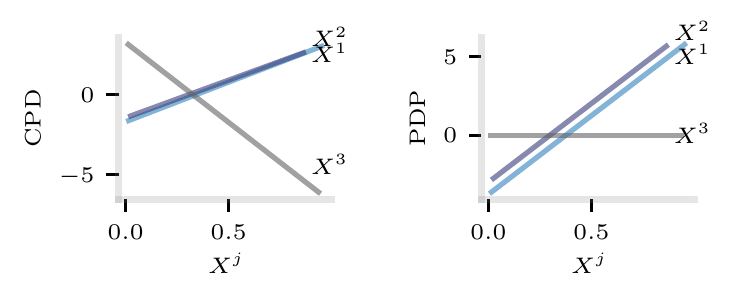}
\hfill
\includegraphics[width=0.6\textwidth]{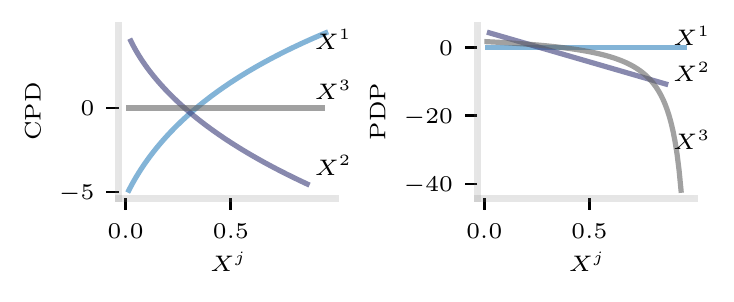}
\captionof{figure}{Top row: CPD and PDP plot based on $f_1$. Bottom row: CPD and PDP plot based on $f_2$. CPD reflects the true feature importance on the simplex while PDP does not.}\label{fig:cpd_vs_pdp}
\end{figure}

\FloatBarrier


\section{Background on kernels}\label{app:kernels}

\subsection{Connection between metrics and kernels}\label{apd:background_kernels_metrics}

\begin{definition}[Metric, semi-metric, and quasi-metric] 
A function $d: \cX\times \cX\to\bR$ is called a \textit{metric} if it satisfies
\begin{enumerate}
    \item[(a)] $d(x,x) = 0$,
    \item[(b)] $d(x,y) = d(y,x) \geq 0$,
    \item[(c)] $d(x,y)\leq d(x,z) + d(y,z)$,
    \item[(d)] $d(x,y) = 0 \Rightarrow x=y$.
\end{enumerate} 
It is called a \textit{semi-metric} if it satisfies (a)-(c), and a \textit{quasi-metric} if it satisfies (a)-(b).
\end{definition}

\begin{definition}[Function of negative type and Hilbertian metric] 
A quasi-metric $d: \cX\times\cX\to\bR$ is called of \textit{negative-type} if for all $n\in\bN$, all $x_1,\cdots, x_n \in \cX$, and all $c_1,\cdots,c_n\in\bR$ with $\sum_{i=1}^n c_i = 0$, it holds that
\begin{equation}\label{cond:embed}
    \sum_{i,j=1}^n c_ic_jd^2(x_i,x_j)\leq 0.
\end{equation}
If $d$ is a (semi-)metric, then $d$ is also called \textit{Hilbertian}.
\end{definition}

\begin{theorem}[Sufficient and necessary conditions for isometric embeddings] \label{thm:hilbert_d_negativetype}
A quasi-metric space $(X,d)$ can be isometrically embedded in a Hilbert space if and only if $d$ is of negative type.
\end{theorem}

\begin{proof}
See \citet[Theorem 2.4]{wells2012embeddings}.
\end{proof}

\begin{definition}[(conditionally) positive definite kernels]
A symmetric function $k: \cX \times \cX \to \bR$ (i.e., $\forall x,y \in \cX, k(x,y) = k(y,x)$) is called a \textit{positive definite kernel} if and only if for all $n\in\bN$, all $x_1,\cdots,x_n \in\cX$, and all $c_1,\cdots, c_n\in\bR$, it holds that 
\begin{equation} \label{cond:cpd}
    \sum_{i,j=1}^n c_ic_j k(x_i,x_j)\geq 0
\end{equation}
It is called a \textit{conditional positive definite kernel} if instead of for all $c_1,\cdots, c_n\in\bR$ condition \eqref{cond:cpd} only holds for all $c_1,\cdots, c_n\in\bR$ with $\sum_{i=1}^n c_i = 0$.
\end{definition}

\begin{lemma} \label{lemma:pd_and_cpd}
Let $\cX$ be a non-empty set, fix $x_0 \in \cX$ and let $k, \tilde{k}: \cX \times \cX \to\bR$ be symmetric functions satisfying for all $x,y\in\cX$ that
\begin{equation}
    k(x,y) = \tilde{k}(x,y) - \tilde{k}(x,x_0) - \tilde{k}(y,x_0) + \tilde{k}(x_0,x_0)
\end{equation}
Then $k$ is positive definite if and only if $\tilde{k}$ is conditionally positive definite.
\end{lemma}

\begin{proof}
Fix $n\in\bN$, $c_1,\cdots, c_n\in\bR$, and $x_0, x_1,\cdots,x_n\in\cX$. Let $c_0 = -\sum_{i=1}^n c_i$, then we have
\begin{equation} \label{eq:pd_and_cpd}
\begin{split}
    \sum_{i,j=0}^n c_i c_j \tilde{k}(x_i,x_j) &= \sum_{i,j=1}^n c_i c_j \tilde{k}(x_i,x_j) + 
        \sum_{i=1}^n c_i c_0 \tilde{k}(x_i,x_0) \\
        &\quad +\sum_{j=1}^n c_0 c_j \tilde{k}(x_j,x_0) + 
        c_0 c_0 \tilde{k}(x_0,x_0) \\
        &= \sum_{i,j=1}^n c_i c_j \tilde{k}(x_i,x_j) - 
        \sum_{i,j=1}^n c_i c_j \tilde{k}(x_i,x_0) \\
        &\quad -\sum_{i,j=1}^n c_i c_j \tilde{k}(x_j,x_0) + 
        \sum_{i,j=1}^n c_i c_j\tilde{k}(x_0,x_0) \\
        &= \sum_{i,j=1}^n c_i c_j [\tilde{k}(x,y) - \tilde{k}(x,x_0) - \tilde{k}(y,x_0) + \tilde{k}(x_0,x_0)] \\
        &= \sum_{i,j=1}^n c_i c_j k(x_i,x_j).
\end{split}
\end{equation}
Now, if $\tilde{k}$ is conditionally positive definite, then (\ref{eq:pd_and_cpd}) implies that $\sum_{i,j=1}^n c_i c_j k(x_i,x_j)\geq 0$, so $k$ is positive definite; if $k$ is positive definite, (\ref{eq:pd_and_cpd}) implies that $\sum_{i,j=0}^n c_i c_j \tilde{k}(x_i,x_j \geq 0$ so $\tilde{k}$ is conditionally positive definite. This completes the proof of Lemma~\ref{lemma:pd_and_cpd}.
\end{proof}

\begin{lemma}[Shifted conditionally positive definite] \label{lem:shifted_cpd}
Let $\cX$ be a non-empty set and let $k: \cX\times\cX\to\bR$ be a positive definite kernel, then 
$$\tilde{k}(x,y) = k(x,y) + f(x) + f(y)$$
is a conditionally positive definite kernel for all $f: \cX\to\bR$.
\end{lemma}

\begin{proof}
The proof follows the exact same argument as the proof of Lemma \ref{lemma:pd_and_cpd}.
\end{proof}

\begin{theorem}[Connection between Hilbertian semi-metrics and positive definite kernels]\label{thm:kernel_and_metric} Let $k: \cX\times \cX\to\bR$
and $d: \cX\times\cX\to[0,\infty)$ be functions. If $k$ is a positive definite
kernel and $d$ satisfies $d^2(x,y) = k(x,x) + k(y,y) - 2k(x,y)$,
then $d$ is a Hilbertian semi-metric. On the other hand, for any $x_0\in\cX$,  if $d$ is a
Hilbertian semi-metric and and $k$ satisfies
$k(x,y) = -\frac{1}{2}d^2(x, y) + \frac{1}{2}d^2(x,x_0) +
\frac{1}{2}d^2(x_0,y)$, then $k$ is a pd kernel.
\end{theorem}

The result is due to \citet{schoenberg1938metric}.

\begin{proof}
We start with the first part. Assume that $k$ is a positive definite kernel and $d$ satisfies $d^2(x,y) = k(x,x) + k(y,y) - 2k(x,y)$. Then, $d$ is indeed a semi-metric by the following arguments:
\begin{enumerate}
\item[(a)] $d(x,x) = \sqrt{k(x,x) + k(x,x) - 2k(x,x)} = 0$,
\item[(b)] $d(x,y) = d(y,x)$, and since $k$ is positive definite, let $c_1 = 1$, $c_2=-1$, $x_1 = x$, and $x_2 = y$,
\begin{equation*}
    \begin{split}
        0\leq \sum_{i,j=1}^n c_i c_j k(x_i,x_j) &= k(x_1, x_1) - k(x_1,x_2) - k(x_2, x_1) + k(x_2,x_2) \\ 
        &= k(x,x) + k(y,y) - 2k(x,y) \\
        & = d(x,y)
    \end{split}
\end{equation*}
\item[(c)] Since $k$ is a positive definite kernel, there exists a feature map $\phi_k$ from $\cX$ to an RKHS $\cH_k$, and we have 
\begin{equation*}
    \begin{split}
        ||\phi_k(x) - \phi_k(y)||^2_{\cH_k} &= \langle \phi_k(x) - \phi_k(y), \phi_k(x) - \phi_k(y) \rangle_{\cH_k} \\
        &= \langle \phi_k(x), \phi_k(x)\rangle_{\cH_k}  + 
        \langle \phi_k(y), \phi_k(y)\rangle_{\cH_k} - 
        2 \langle \phi_k(x), \phi_k(y)\rangle_{\cH_k} \\
        &= k(x,x) + k(y,y) - 2k(x,y) \\
        &= d^2(x,y)
    \end{split}
\end{equation*}
Therefore, $d(x,z) \leq d(x, y) + d(y,z)$ follows from the triangle inequality of a norm.
\end{enumerate}
To show $d$ is also Hilbertian, take any $n\in\bN$, any $x_1,\cdots, x_n \in\cX$, and any $c_1,\cdots, c_n\in\bR$, we have 
\begin{equation*}
\begin{split}
    \sum_{i,j=1}^n c_i c_j d(x_i,x_j) &= \sum_{i=1}^n c_i k(x_i,x_i) \sum_{j=1}^n c_j + \sum_{j=1}^n c_i k(x_j,x_j) \sum_{i=1}^n c_i \\
    &\quad - 2\sum_{i,j=1}^n c_i c_j k(x_i,x_j) \\
    &= -2\sum_{i,j=1}^n c_i c_j k(x_i,x_j) \leq 0 \quad \text{(since $k$ is positive definite)}.
\end{split}
\end{equation*}
This proves the first part of the theorem.

For the second part, assume that $d$ is a Hilbertian semi-metric and $k$ satisfies $k(x,y) = -\frac{1}{2}d^2(x, y) + \frac{1}{2}d^2(x,x_0) +
\frac{1}{2}d^2(x_0,y)$. Then, since $d$ is Hilbertian, $-d^2$ satisfies the requirement of a conditionally positive definite kernel (with the additional property that $-d^2(x,x) = 0$). Hence, by Lemma \ref{lemma:pd_and_cpd}, $k$ is indeed positive definite. This completes the proof of Theorem~\ref{thm:kernel_and_metric}.
\end{proof}
  
\subsection{Dimensionality reduction and visualization with kernels}\label{app:kernel_mds}

One important benefit of using the kernel approach is that we can leverage the kernels for dimensionality reduction and visualization, so that one can identify outliers in the data and further investigate them. In this section, we provide a short introduction on how to use kernels for multi-dimensional scaling and connect it to kernel PCA \citep{scholkopf2002learning}.

Kernel methods project the compositional data into a (potentially) high-dimensional RKHS $\cH_k$, which we now want to project into the low dimensional Euclidean space $\bR^{\ell}$ (with $\ell \ll p$) such that the lower dimensional representation preserves information that helps separate the observations of different traits in the RKHS. That is, given observations $x_1,\cdots, x_n \in\bS^{p-1}$ and a kernel $k$, we would like to define a map $\Phi: \cH_k\to\bR^{\ell}$ such that 
\begin{equation*}
    \sum_{i,j=1}^n\|\langle k(x_i, \cdot), k(x_j,\cdot)\rangle_{\cH_k} - \langle\Phi(k(x_i,\cdot)), \Phi(k(x_j,\cdot))\rangle_{\bR^{\ell}}\|^2
\end{equation*}
is minimized. In matrix notation, this corresponds to solving
\begin{equation*}
    \argmin_{Z\in\bR^{n\times \ell}} \|K - ZZ^{\top}\|^2,
\end{equation*}
where the rows of $Z$ are $z_i = \Phi(k(x_i,\cdot)) \in \bR^{\ell}$ for all $i\in\{1\cdots,n\}$ and $K\in\bR^{n\times n}$ is the kernel Gram-matrix.
This is similar to the classical multidimensional scaling (MDS) but measuring the similarity in the RKHS instead of in Euclidean space. By the Eckart-Young theorem \citep{eckart1936approximation}, this minimization problem can be solved via the eigendecomposition of the matrix $K = V\Sigma V^{\top}$, and the optimal solution is $$Z_{\text{opt}} = (V_1,\ldots, V_{\ell})(\Sigma_{:\ell})^{\frac{1}{2}},$$ where $V_{1},\ldots,V_{\ell}$ are the first $\ell$ columns of $V$ and $\Sigma_{:\ell}$ is the upper-left ($\ell\times\ell$)-submatrix of $\Sigma$. The optimal projection $\Phi_{\text{opt}}$ is then given for all $f\in\cH_k$ by
\begin{equation}
    \label{eq:kernel_projection}
    \Phi_{\text{opt}}(f)=(\Sigma_{:\ell})^{-\frac{1}{2}}(V_1,\ldots, V_{\ell})^{\top}\begin{pmatrix} \langle f, k(x_1,\cdot)\rangle_{\cH_k} \\ \vdots \\ \langle f, k(x_n,\cdot)\rangle_{\cH_k}
    \end{pmatrix}.
\end{equation}
This in particular allows to project a new observations $w\in\simp^{p-1}$ with the same projection that is $w\mapsto \Phi_{\text{opt}}(k(w,\cdot))$.

The projection in \eqref{eq:kernel_projection} depends on the origin of the RKHS $\cH_k$. To remove this dependence, it may therefore be desirable to consider a centered version of the optimal projection. This can be achieved by considering the RKHS $\tilde{\cH}_k$ consisting of the functions $\tilde{f}(\cdot) = f(\cdot)- \frac{1}{n}\sum_{i=1}^n k(x_i,\cdot)$ with $f\in\cH_k$. To compute the optimal centered projection \eqref{eq:kernel_projection} for the RKHS $\tilde{\cH}_k$, we only need to perform double centering on the kernel matrix $K$, i.e., $\tilde{K} = HKH$, where $H = I-\frac{1}{n}\mathbf{1}\mathbf{1}^T$ and replace $k(x,\cdot)$ by $\tilde{k}(x,\cdot)=k(x,\cdot)- \frac{1}{n}\sum_{i=1}^n k(x_i,\cdot)$. With the centering step, this procedure is equivalent to kernel PCA \citep{scholkopf2002learning}. The steps to obtain the lower-dimensional representation in matrix form are given in Algorithm~\ref{alg:dim_redu}.

\begin{center}

\begin{minipage}{.7\linewidth}
\begin{algorithm}[H]
\caption{Dimensionality reduction with kernels}\label{alg:dim_redu}
\begin{algorithmic}[1]
\Require{Training data $X_1,\ldots,X_n\in\simp^{p-1}$, visualization data $X^{\text{new}}_1,\ldots,X^{\text{new}}_m\in\simp^{p-1}$} (can be same as training data), kernel function $k$, dimension $l\in\{1,\ldots,p\}$, indicator whether to use centering $\text{CenterK}\in\{\texttt{True}, \texttt{False}\}$
\Ensure{$l$-dimensional representation $Z=(Z_1,\ldots,Z_m)^{\top}\in\bR^{m\times l}$}

\Statex
\State \textit{\# Define centering function}
\Function{CenterKernelMatrix}{$K$, $\tilde{K}$}
\State $K^{\text{center}} \gets \tilde{K} - \frac{1}{n}\mathbf{1}\mathbf{1}^T K - \frac{1}{n} \tilde{K}\mathbf{1}\mathbf{1}^T + 
    \frac{1}{n^2}\mathbf{1}\mathbf{1}^T K \mathbf{1}\mathbf{1}^T$
\State \Return $K^{\text{center}}$
\EndFunction
\Statex

\State \textit{\# Compute kernel matrix for training data}
\For{$i,j=1,\cdots,n$}
    \State $K_{ij} \gets k(X_i, X_j)$
\EndFor
\Statex
\State \textit{\# Compute kernel matrix for visualization data}
\For{$i=1,\cdots,m$ and $j=1,\cdots,n$}
    \State $K^{\text{new}}_{ij} \gets k(X^{\text{new}}_i, X_j)$
\EndFor
\Statex

\State \textit{\# Center kernel matrices}
\If{$\text{CenterK}$}
    \State $K^{\text{new}} \gets \text{CenterKernelMatrix}(K, K^{\text{new}})$
    \State $K \gets \text{CenterKernelMatrix}(K, K)$
\EndIf
\Statex

\State \textit{\# Compute $l$-dimensional representation}
\State $V, \Sigma \gets \text{eigenvalue decomposition of }K$
\State $Z \gets K^{\text{new}}(V_1,\ldots,V_l)(\Sigma_{:l})^{-\frac{1}{2}}$
\Statex

\State \Return $Z$

\end{algorithmic}
\end{algorithm}
\end{minipage}
\end{center}

\FloatBarrier
  
\subsubsection{Compositionally adjusted coordinate-wise contribution to each principle component}
  
Given the optimal projection function $\Phi_{\text{opt}}$, define the function $F:\simp^{p-1}\rightarrow\mathbb{R}^{\ell}$ for all $x\in\simp^{p-1}$ by $F(x)=\Phi_{\text{opt}}(k(x,\cdot))$. We then call the components $F^1,\ldots,F^{\ell}$ the principle components. Our goal is now to understand how each principle component is affected by changes in the different components of its arguments. For this, fix a principle component $r\in\{1,\ldots,\ell\}$ and consider for each $j\in\{1,\ldots,p\}$ the quantities
\begin{equation*}
    \E[F^r(\psi_j(X, c))-F^r(X)],
\end{equation*}
where $c\in(0,1)$ and $\psi_j$ the perturbation defined in App.~\ref{sec:add_defs}. This is very similar in spirit as the CFI but with the derivative replaced by a difference and measures how much a perturbation of size $c$ in the $j$-th component effects the value of the $r$-th principle component. It is easily estimated by
\begin{equation*}
    \frac{1}{n}\sum_{i=1}^nF^r(\psi_j(X_i, c))-F^r(X_i).
\end{equation*}


\section{Proofs}\label{apd:proofs}

\subsection{Proof of Proposition~\ref{prop:cfi_cpd_log_contrast_model}} \label{apd:cfi_cpd_log_contrast_model_proof}

\begin{proof}
We start with the CFI. Fix $j\in\{1,\ldots,p\}$ and $x\in\simp^{p-1}$, then we can compute the derivative using the chain rule and the explicit form of the perturbation $\psi$ as follows
\begin{align}
    \tfrac{d}{dc}f(\psi_j(x, c))
    &=\left\langle \nabla f(\psi_j(x, c)), \tfrac{d}{dc}\psi_j(x, c)\right\rangle \nonumber\\
    &=\left\langle\nabla f(\psi_j(x, c)), \tfrac{d}{dc} s_c (x^1,\cdots,x^{j-1}, cx^j, x^{j+1},\cdots,x^p)^{\top}\right\rangle\nonumber\\
    &=\left\langle\nabla f(\psi_j(x, c)), \tfrac{d}{dc} \frac{1}{\textstyle\sum_{\ell\neq j}^p x^{\ell}+cx^j}
    (x^1,\cdots,x^{j-1}, cx^j, x^{j+1},\cdots,x^p)^{\top}\right\rangle\nonumber\\
    &=\Bigg\langle\nabla f(\psi_j(x, c)), \nonumber \\ 
    &\quad \frac{-x^j}{(\textstyle\sum_{\ell\neq j}^p x^{\ell}+cx^j)^2}\left(x^1,\cdots,x^{j-1},cx^jx^j- x^j(\textstyle\sum_{\ell\neq j}^p x^{\ell}+cx^j), x^{j+1},\cdots,x^p\right)^{\top}\Bigg\rangle.\nonumber
\end{align}
Evaluating, the derivative at $c=1$ leads to
\begin{equation}
\label{eq:CFI_part1}
   \tfrac{d}{dc}f(\psi_j(x, c))\vert_{c=1}
   =\left\langle\nabla f(x), x^j(e_j- x)\right\rangle,
\end{equation}
where we used that $\psi_j(x, 1)=x$. Moreover, the gradient of $f$ in the case of the log-contrast model is given by
\begin{equation}
\label{eq:CFI_part2}
    \nabla f(x)=\left(\frac{\beta_1}{x^1},\ldots, \frac{\beta_p}{x^p}\right)^{\top}.
\end{equation}
Combining \eqref{eq:CFI_part1} and \eqref{eq:CFI_part2} together with the constraint $\sum_{k=1}^p\beta_k=0$ implies that
\begin{equation*}
    \tfrac{d}{dc}f(\psi_j(x, c))\vert_{c=1}=-x^j\sum_{k\neq j}\beta_k + \beta^j(1-x^j)=\beta_j.
\end{equation*}
Hence, taking the expectation leads to
\begin{equation*}
    I^j_j=\bE[\tfrac{d}{dc}f(\psi_j(X, c))\vert_{c=1}]=\beta_j,
\end{equation*}
which proves the first part of the proposition.

Next, we show the result for the CPD. Fix $j\in\{1,\ldots,p\}$ and $z\in[0,1]$. Then $S_f^j(z)$ for the log-contrast model can be computed as follows
\begin{align*}
    S_f^j(z)
    &=\bE[f(\phi_j(X,z))]-\bE[f(X)]\\
    &=\sum_{\ell=1}^p\beta_{\ell}\bE[\log(\phi_j(X,z)^{\ell})]-\bE[f(X)]\\
    &=\sum_{\ell\neq j}^p\beta_{\ell}\bE[\log(sX^{\ell})]+\beta_j\log(z)-\bE[f(X)]\\
    &=\beta_j\log(z)+\sum_{\ell\neq j}^p\beta_{\ell}\bE[\log(s)]+\sum_{\ell\neq j}^p\beta_{\ell} \bE[\log(X^{\ell})]-\bE[f(X)],
\end{align*}
where $s=(1-z)/(\sum_{\ell\neq j}^p X^{\ell})$. Using $\beta^j=-\sum_{\ell\neq j}^p\beta_{\ell}$ (which follows from the log-contrast model constraint on $\beta$) we can simplify this further and get
\begin{align*}
    S_f^j(z)
    &=\beta_j\log(z)+\sum_{\ell\neq j}^p\beta_{\ell}\bE[\log(1-z)]-\sum_{\ell\neq j}^p\beta_{\ell}\bE[\log(\textstyle\sum_{k\neq j}^p X^{k})]+\sum_{\ell\neq j}^p\beta_{\ell} \bE[\log(X^{\ell})]-\bE[f(X)]\\
    &=\beta_j\log(z)-\beta_{j}\bE[\log(1-z)]+\beta_{j}\bE[\log(\textstyle\sum_{k\neq j}^p X^{k})]+\sum_{\ell\neq j}^p\beta_{\ell} \bE[\log(X^{\ell})]-\bE[f(X)]\\
    &=\beta_j\log\left(\frac{z}{1-z}\right)+\beta_{j}\bE[\log(\textstyle\sum_{k\neq j}^p X^{k})]+\sum_{\ell\neq j}^p\beta_{\ell} \bE[\log(X^{\ell})]-\sum_{\ell=1}^p\beta_{\ell}\bE[\log(X^{\ell})]\\
    &=\beta_j\log\left(\frac{z}{1-z}\right)+c,
\end{align*}
with $c=\beta_{j}\bE[\log(\textstyle\sum_{k\neq j}^p X^{k})]+\sum_{\ell\neq j}^p\beta_{\ell} \bE[\log(X^{\ell})]-\sum_{\ell=1}^p\beta_{\ell}\bE[\log(X^{\ell})].$ Finally, assume $\beta^j=0$, then it holds that
\begin{equation*}
    c=\sum_{\ell\neq j}^p\beta_{\ell} \bE[\log(X^{\ell})]-\sum_{\ell\neq j}^p\beta_{\ell}\bE[\log(X^{\ell})]
    =0.
\end{equation*}
This completes the proof of Proposition~\ref{prop:cfi_cpd_log_contrast_model}.
\end{proof}

\subsection{Proof of Theorem~\ref{thm:consistency}} \label{apd:conistency_proof}

\begin{proof}
We first prove (i). To see this, we apply the triangle inequality to get that
\begin{equation}
\label{eq:basic_ineq1}
    |\hat{I}^j_{\hat{f}_n}-I^j_{f^*}|
    \leq \underbrace{|\hat{I}^j_{\hat{f}_n}-\hat{I}^j_{f^*}|}_{\eqqcolon A_n} + \underbrace{|\hat{I}^j_{f^*}-I^j_{f^*}|}_{\eqqcolon B_n}.
\end{equation}
Next, we consider the two terms $A_n$ and $B_n$ separately. We begin with $A_n$, by using the definition of the CFI together with \eqref{eq:CFI_part1} from the proof of Proposition~\ref{prop:cfi_cpd_log_contrast_model}. This leads to
\begin{align*}
    A_n
    &=\left\vert \frac{1}{n}\sum_{i=1}^n\left(\tfrac{d}{dc}\hat{f}_n(\psi(X_i, c)\vert_{c=1}-\tfrac{d}{dc}f^*(\psi(X_i, c)\vert_{c=1}\right)\right\vert\\
    &=\left\vert \frac{1}{n}\sum_{i=1}^n \left\langle\nabla \hat{f}_n(X_i)-\nabla f^*(X_i), X^j_i(e_j- X_i)\right\rangle\right\vert\\
    &\leq\frac{1}{n}\sum_{i=1}^n \left\vert \left\langle\nabla \hat{f}_n(X_i)-\nabla f^*(X_i), X^j_i(e_j- X_i)\right\rangle\right\vert\\
    &\leq\frac{1}{n}\sum_{i=1}^n \big\|\nabla \hat{f}_n(X_i)-\nabla f^*(X_i)\big\|_2 \big\|X^j_i(e_j- X_i)\big\|_2\\
    &\leq\frac{1}{n}\sum_{i=1}^n \big\|\nabla \hat{f}_n(X_i)-\nabla f^*(X_i)\big\|_2,
\end{align*}
where for the last three steps we used the triangle inequality, the Cauchy-Schwartz inequality and that $\|X^j_i(e_j- X_i)\|_2\leq 1$ since $X_i\in\simp^{p-1}$, respectively. By assumption, it therefore holds that $A_n\rightarrow 0$ in probability as $n\rightarrow\infty$.
For the $B_n$ term, observe that using the same bounds it holds that
\begin{equation*}
    \bE\left[\left(\tfrac{d}{dc}f^*(\psi(X_i, c)\vert_{c=1}\right)^2\right]
    =\bE\left[\left(\left\langle\nabla f^*(X_i), X^j_i(e_j- X_i)\right\rangle\right)^2\right]
    \leq \bE\left[\left\|\nabla f^*(X_i)\right\|_2^2\right].
\end{equation*}
By assumption that $\bE\left[\left\|\nabla f^*(X_i)\right\|_2^2\right]<\infty$ this implies we can apply the weak law of large numbers to get for $n\rightarrow\infty$ that
\begin{equation*}
    \hat{I}^j_{f^*}=\frac{1}{n}\sum_{i=1}^n\tfrac{d}{dc}f^*(\psi(X_i, c)\vert_{c=1}\overset{P}{\longrightarrow} \bE\left[\tfrac{d}{dc}f^*(\psi(X_i, c)\vert_{c=1}\right]=I_{f^*}^j.
\end{equation*}
This immediately implies that $B_n\rightarrow 0$ in probability as $n\rightarrow\infty$. Combining the convergence of $A_n$ and $B_n$ in \eqref{eq:basic_ineq1} completes the proof of (i).

Next, we prove (ii). Fix $j\in\{1,\ldots,p\}$  and $z\in[0,1]$ such that $z/(1-z)\in\operatorname{supp}(X^j/\sum_{\ell\neq j}X^{\ell})$.
By the definition of the perturbation $\phi_j$ we get that
\begin{equation}
    \label{eq:nonextrapolation}
    \phi_j(X,z)=s(X^1,\cdots,X^{j-1}, \tfrac{z}{1-z} \textstyle\sum_{\ell\neq j}X^{\ell}, X^{j+1},\cdots,X^p)
\end{equation}
where $s = (1-z)/(\sum_{\ell\neq j}^p X^{\ell})$.
Next, using the assumption that $\operatorname{supp}(X)=\{x\in\simp^{p-1}\mid x=w/(\sum_j w^j) \text{ with }w\in \operatorname{supp}(X^1)\times\cdots\times \operatorname{supp}(X^p)\}$ and that 
$z/(1-z)\in\operatorname{supp}(X^j/\sum_{\ell\neq j}X^{\ell})$ we get that 
\begin{equation}
    \label{eq:nonextrapolation2}
    \phi_j(X,z)\in\operatorname{supp}(X^j)
\end{equation}
almost surely.

By the triangle inequality it holds that
\begin{equation}
\label{eq:basic_ineq1_part2}
    |\hat{S}^j_{\hat{f}_n}(z)-S^j_{f^*}(z)|
    \leq \underbrace{|\hat{S}^j_{\hat{f}_n}(z)-\hat{S}^j_{f^*}(z)|}_{\eqqcolon C_n} + \underbrace{|\hat{S}^j_{f^*}(z)-S^j_{f^*}(z)|}_{\eqqcolon D_n}.
\end{equation}
We now consider the two terms $C_n$ and $D_n$ separately. First, we apply the triangle inequality to bound the $C_n$ term as follows.
\begin{align*}
    C_n
    &=\left\vert \frac{1}{n}\sum_{i=1}^n(\hat{f}_n(\phi_j(X_i,z))-f^*(\phi_j(X_i,z)))+\frac{1}{n}\sum_{i=1}^n(\hat{f}_n(X_i)-f^*(X_i))\right\vert\\
    &\leq \frac{1}{n}\sum_{i=1}^n\left\vert\hat{f}_n(\phi_j(X_i,z))-f^*(\phi_j(X_i,z))\right\vert+\frac{1}{n}\sum_{i=1}^n\left\vert\hat{f}_n(X_i)-f^*(X_i)\right\vert\\
    &\leq 2\sup_{x\in\operatorname{supp}(X)}\left\vert\hat{f}_n(x)-f^*(x)\right\vert,
\end{align*}
where for the last step we used a supremum bound together with \eqref{eq:nonextrapolation}. Hence, using the assumption that $\sup_{x\in\operatorname{supp}(X)}\vert\hat{f}_n(x)-f^*(x)\vert\overset{P}{\rightarrow}0$ as $n\rightarrow\infty$, we get that $C_n\rightarrow\infty$ in probability as $n\rightarrow\infty$. Similarly, for the $D_n$ term we get that
\begin{align*}
    D_n
    &=\left\vert \frac{1}{n}\sum_{i=1}^n f^*(\phi_j(X_i,z))-\bE[f^*(\phi_j(X_i,z))]+\frac{1}{n}\sum_{i=1}^n f^*(X_i)-\bE[f^*(X_i)]\right\vert\\
    &\leq\left\vert\frac{1}{n}\sum_{i=1}^n f^*(\phi_j(X_i,z))-\bE[f^*(\phi_j(X_i,z))]\right\vert+\left\vert\frac{1}{n}\sum_{i=1}^n f^*(X_i)-\bE[f^*(X_i)]\right\vert.
\end{align*}
Since the $X_1,\ldots,X_n$ and hence $\phi_j(X_1,z),\ldots,\phi_j(X_n,z)$ are i.i.d.\ and bounded we can apply the weak law of large numbers to get that $D_n\rightarrow 0$ in probability as $n\rightarrow\infty$.

Finally, combining the convergence of $C_n$ and $D_n$ with \eqref{eq:basic_ineq1_part2} proves (ii) and hence completes the proof of Theorem~\ref{thm:consistency}.
\end{proof}

\subsection{Proof of Proposition~\ref{prop:weighted_aitchison}}\label{apd:prop:weighted_aitchison}

\begin{proof}
For this proof, we denote by $\simp^{p-1}$ the open instead of the closed simplex.

First, since $k_W$ is a positive definite kernel (see Sec.~\ref{apd:validity_weighted_kernels} for a proof), it holds that the RKHS $\mathcal{H}_{k_W}$ can be expressed as the closure of
\begin{equation*}
    \mathcal{F}\coloneqq\Big\{f:\simp^{p-1}\times\simp^{p-1}\rightarrow\mathbb{R} \,\Big\vert\, \exists n\in\mathbb{N}, z_1,\ldots,z_n\in\simp^{p-1}, \alpha_1,\ldots,\alpha_n\in\mathbb{R}: f(\cdot)=\textstyle\sum_{i=1}^n\alpha_i k_W(z_i,\cdot)\Big\}.
\end{equation*}
We now show that any function in $\mathcal{F}$ has the expression given in the statement of the proposition. Let $f\in\mathcal{F}$ be arbitrary with the expansion
\begin{equation*}
    f(\cdot)=\sum_{i=1}^n\alpha_i k_W(z_i,\cdot).
\end{equation*}
Then, for all $x\in\simp^{p-1}$ it holds that
\begin{align}
    f(x)
    &=\sum_{i=1}^n\alpha_i\sum_{j,\ell=1}^pW_{\ell,j}\log\Big(\frac{z_i^{\ell}}{g(z_i)}\Big)\log\Big(\frac{x^{j}}{g(x)}\Big)\nonumber\\
    &=\sum_{j=1}^p\bigg(\sum_{\ell=1}^p W_{\ell,j}\sum_{i=1}^n\alpha_i\log\Big(\frac{z_i^{\ell}}{g(z_i)}\Big)\bigg)\log\Big(\frac{x^{j}}{g(x)}\Big)\nonumber\\
    &=\sum_{j=1}^p\bigg(\sum_{\ell=1}^p W_{\ell,j}\widetilde{\beta}_{\ell}\bigg)\log\Big(\frac{x^{j}}{g(x)}\Big)\nonumber\\
    &=\sum_{j=1}^p\bigg(\sum_{\ell=1}^p W_{\ell,j}\widetilde{\beta}_{\ell}\bigg)\log\big(x^{j}\big)-\bigg(\sum_{j,\ell=1}^p W_{\ell,j}\widetilde{\beta}_{\ell}\bigg)\log\big(g(x)\big) \nonumber\\
    &=\sum_{j=1}^p\bigg(\sum_{\ell=1}^p W_{\ell,j}\widetilde{\beta}_{\ell}\bigg)\log\big(x^{j}\big)-\bigg(\sum_{\ell=1}^p \widetilde{\beta}_{\ell}\bigg)\log\big(g(x)\big),\label{eq:weighted_aitchison_function}
\end{align}
where in the third line we defined $\widetilde{\beta}_{\ell}\coloneqq \sum_{i=1}^n\alpha_i\log\Big(\frac{z_i^{\ell}}{g(z_i)}\Big)$ and in the last equation we used that $\sum_{j=1}^{p}W_{\ell,j}=1$ for all $\ell\in\{1,\ldots,p\}$ by construction of $W$. Furthermore, we get that
\begin{equation}
\label{eq:aitchison_proof_cond1}
    \sum_{j=1}^p\widetilde{\beta}_j
    =\sum_{i=1}^{n}\alpha_i\bigg(\sum_{j=1}^p\log\big(z_i^j\big)-p\log(g(z_i))\bigg)
    =\sum_{i=1}^{n}\alpha_i\bigg(\sum_{j=1}^p\log\big(z_i^j\big)-\sum_{j=1}^p\log\big(z_i^j\big)\bigg)
    =0.
\end{equation}
Now, combining this with \eqref{eq:weighted_aitchison_function} and setting $\beta_{j}\coloneqq \sum_{\ell=1}^p W_{\ell,j}\widetilde{\beta}_{\ell}$ implies that
\begin{equation*}
    f(x)=\beta^{\top}\log(x),
\end{equation*}
where $\beta$ does not depend on $x$.

It remains to show that $\beta$ satisfies (i) $\sum_{j=1}^p\beta_j=0$ and (ii) for all $\ell\in\{1,\ldots,m\}$ it holds for all $i,j\in P_{\ell}$ that $\beta_i=\beta_j$. For (i), we can use \eqref{eq:aitchison_proof_cond1} and directly compute
\begin{equation*}
    \sum_{j=1}^p\beta_j
    =\sum_{j=1}^p\sum_{\ell=1}^p W_{\ell,j}\widetilde{\beta}_{\ell}
    =\sum_{\ell=1}^p \widetilde{\beta}_{\ell}=0.
\end{equation*}
Finally for (ii), fix $k\in\{1,\ldots,m\}$ and $i,j\in P_{k}$, then it holds that
\begin{equation*}
    \beta_j
    =\sum_{\ell=1}^p W_{\ell,j}\widetilde{\beta}_{\ell}
    =\sum_{\ell=1}^p \sum_{r=1}^{m}\frac{1}{|P_r|}\mathds{1}_{\{\ell,j\in P_r\}}\widetilde{\beta}_{\ell}
    =\sum_{\ell=1}^p \frac{1}{|P_k|}\widetilde{\beta}_{\ell}
    =\sum_{\ell=1}^p \sum_{r=1}^{m}\frac{1}{|P_r|}\mathds{1}_{\{\ell,i\in P_r\}}\widetilde{\beta}_{\ell}
    =\sum_{\ell=1}^p W_{\ell,i}\widetilde{\beta}_{\ell}
    =\beta_i.
\end{equation*}
This completes the proof of Proposition~\ref{prop:weighted_aitchison}.
\end{proof}


\section{List of kernels implemented in KernelBiome}\label{apd:list_kernels}

In this section we summarize the all kernels implemented in \KernelBiome and visualize the metrics and kernels via heatmaps when $p=3$. The reference points are the neutral point $u = (\frac{1}{3}, \frac{1}{3}, \frac{1}{3})$, a vertex $v = (1,0,0)$, a midpoint on a boundary $m = (\frac{1}{2}, \frac{1}{2}, 0)$, and an interior point $z = (\frac{1}{4}, \frac{1}{4}, \frac{1}{2})$ of the simplex. For kernels we omit the neutral point, since $k(x,u) = 0$ for any $x \in\bS^2$, as we centered our kernels at $u$.

\begin{mybox}{Linear metric \& kernel}{blue!20}
\begin{equation*}
    d^2(x,y) = \sum_{j=1}^p (x^j - y^j)^2
\end{equation*}
\begin{equation*}
    k(x,y) = \Big(\sum_{j=1}^p x^j y^j\Big) - \tfrac{1}{p}
\end{equation*}
\centering\includegraphics{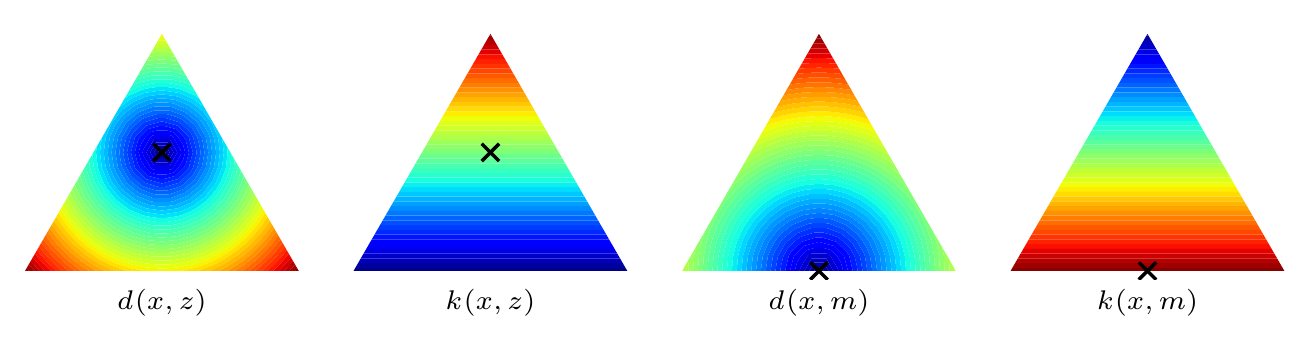}
\end{mybox}

\begin{mybox}{RBF metric \& kernel}{blue!20}
\begin{equation*}
    d^2(x,y) = 2-2\exp\Big(\sum_{j=1}^p \frac{(x^j - y^j)^2}{2\sigma^2}\Big)
\end{equation*}
\begin{equation*}
    k(x,y) = \exp\Big(-\sum_{j=1}^p \frac{(x^j-y^j)^2}{2\sigma^2}\Big)
\end{equation*}
\centering\includegraphics{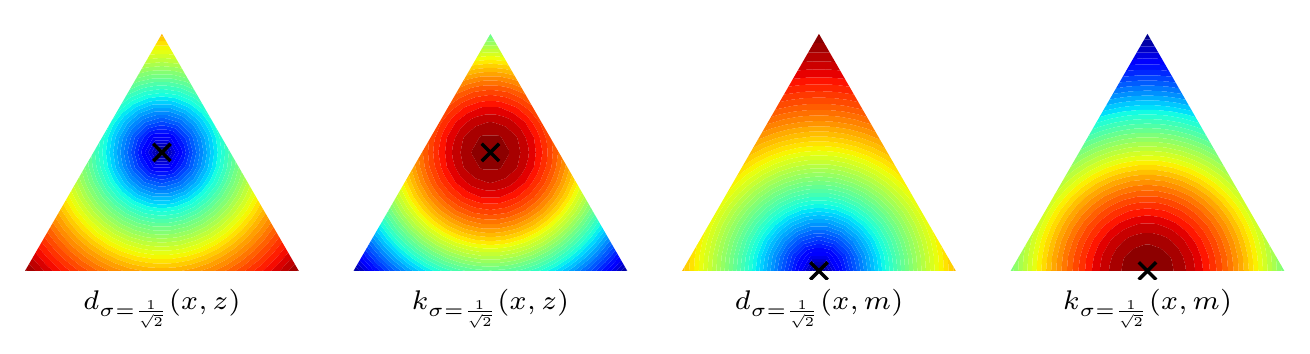}
\end{mybox}

\begin{mybox}{Generalized-JS metric \& kernel ($a < \infty, b \in [0.5, a)$)}{blue!20}
\begin{equation*}
    d^2(x,y) = \frac{ab}{a-b}\sum_{j=1}^p\frac{2^{\frac{1}{b}}\Big[(x^j)^{a}+(y^j)^{a}\Big]^{\frac{1}{a}} 
    - 2^{\frac{1}{a}}\Big[(x^j)^{b}+(y^j)^{b}\Big]^{\frac{1}{b}}}{2^{\frac{1}{a}+\frac{1}{b}}}
\end{equation*}
\begin{equation*}
\begin{split}
    k(x,y) &= -\frac{ab}{a-b}\cdot 2^{-(1+\frac{1}{a}+\frac{1}{b})} \sum_{j=1}^p 
    \Bigg\{2^\frac{1}{b} \Big(\Big[(x^j)^{a} + (y^j)^{a}\Big]^{\frac{1}{a}} - \Big[(x^j)^{a} + (\tfrac{1}{p})^{a}\Big]^{\frac{1}{a}} \\ 
    &\qquad - \Big[(\tfrac{1}{p})^{a} + (y^j)^{a}\Big]^{\frac{1}{a}}\Big) -2^\frac{1}{a} \Big(\Big[(x^j)^{b} + (y^j)^{b}\Big]^{\frac{1}{b}} - \Big[(x^j)^{b} + (\tfrac{1}{p})^{b}\Big]^{\frac{1}{b}} \\ 
    &\qquad - \Big[(\tfrac{1}{p})^{b} + (y^j)^{b}\Big]^{\frac{1}{b}}\Big)\Bigg\}
\end{split}
\end{equation*}
\centering\includegraphics{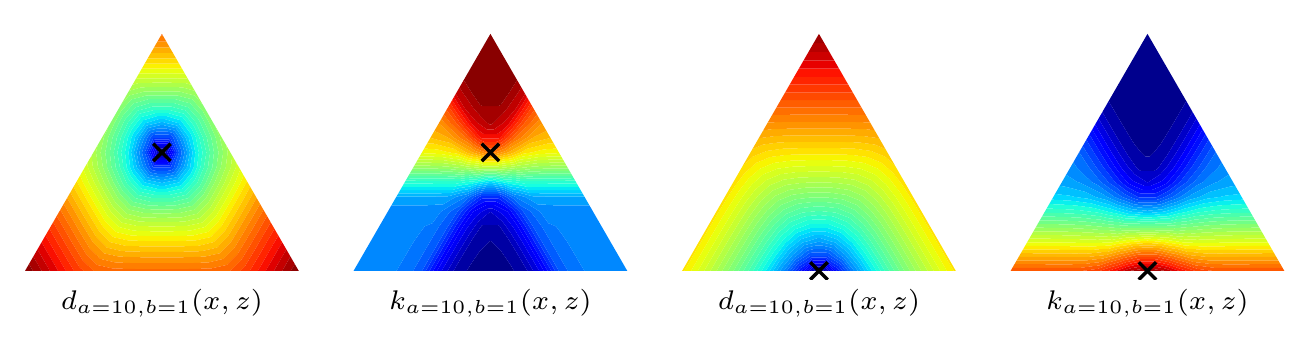}
\end{mybox}

\begin{mybox}{Generalized-JS metric \& kernel ($a \to \infty, b < \infty$)}{blue!20}
\begin{equation*}
    d^2(x,y) = \sum_{j=1}^p b\Bigg\{2^{\frac{1}{b}}\cdot\max\{x^j,y^j\} - \Big[(x^j)^{b} + (y^j)^{b}\Big]^{\frac{1}{b}}\Bigg\}
\end{equation*}
\begin{equation*}
\begin{split}
    k(x,y) &= -\frac{b}{2}\sum_{j=1}^p \Bigg\{
    2^{\frac{1}{b}}\Big(\max\{x^j,y^j\} - \max\{x^j,\tfrac{1}{p}\} - \max\{y^j,\tfrac{1}{p}\}\Big) \\
    &\quad - \Big[(x^j)^{b} + (y^j)^{b}\Big]^{\frac{1}{b}} 
    + \Big[(x^j)^{b} + (\tfrac{1}{p})^{b}\Big]^{\frac{1}{b}}
    + \Big[(y^j)^{b} + (\tfrac{1}{p})^{b}\Big]^{\frac{1}{b}}
    \Bigg\}
\end{split}
\end{equation*}
\centering\includegraphics{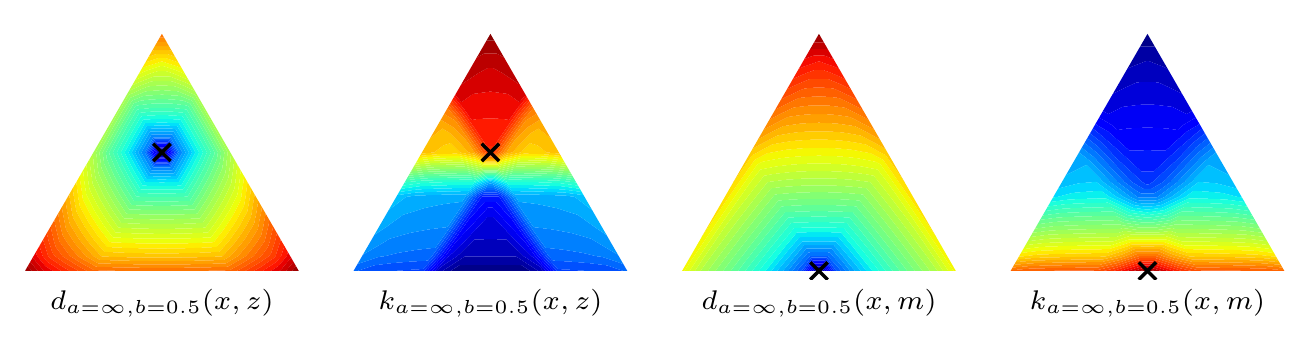}
\end{mybox}

\begin{mybox}{Generalized-JS metric \& kernel ($a < \infty, b \to a$)}{blue!20}
\begin{equation*}
\begin{split}
    d^2(x,y) &= \sum_{j=1}^p\bigg[\frac{(x^j)^{b} + (y^j)^{b}}{2}\bigg]^{\frac{1}{b}} \cdot \bigg[\frac{(x^j)^{b}}{(x^j)^{b}+(y^j)^{b}}\cdot\log\frac{2(x^j)^{b}}{(x^j)^{b}+(y^j)^{b}} \\ 
    &\qquad\qquad +\frac{(y^j)^{b}}{(x^j)^{b}+(y^j)^{b}}\cdot\log\frac{2(y^j)^{b}}{(x^j)^{b}+(y^j)^{b}}\bigg]
\end{split}
\end{equation*}
\begin{equation*}
\begin{split}
    k(x,y) &= -\frac{1}{2^{\frac{1}{b}+1}}\sum_{j=1}^p \Bigg\{
    \Big[(x^j)^{b}+(y^j)^{b}\Big]^{\frac{1}{b}-1} \cdot \Big((x^j)^{b}\cdot\log\Big[2(x^j)^{b}\Big] + (y^j)^{b}\log\Big[2(y^j)^{b}\Big] \\
    &\qquad\qquad\qquad - \Big[(x^j)^{b} + (y^j)^{b}\Big]\cdot\log\Big[(x^j)^{b} + (y^j)^{b}\Big]\Big)\\
    &\qquad\qquad - \Big[(x^j)^{b}+\tfrac{1}{p^{b}}\Big]^{\frac{1}{b}-1}\cdot \Big(-\Big[(x^j)^{b} + (\tfrac{1}{p^{b}})\Big]\cdot\log\Big[(x^j)^{b} + \tfrac{1}{p^{b}}\Big] \\ 
    &\qquad\qquad\qquad + (x^j)^{b}\cdot\log\Big[2(x^j)^{b}\Big] + \tfrac{1}{p^{b}}\cdot\log\Big[\tfrac{2}{p^{b}}\Big]\Big) \\
    &\qquad\qquad - \Big[(y^j)^{b}+\tfrac{1}{p^{b}}\Big]^{\tfrac{1}{b}-1}\cdot\Big(-\Big[(y^j)^{b} + (\tfrac{1}{p^{b}})\Big]\cdot\log\Big[(y^j)^{b} + \tfrac{1}{p^{b}}\Big] \\ 
    &\qquad\qquad\qquad + (y^j)^{b}\cdot\log\Big[2(y^j)^{b}\Big] + \tfrac{1}{p^{b}}\cdot\log\Big[\tfrac{2}{p^{b}}\Big]\Big)
    \Bigg\}
\end{split}
\end{equation*}
\centering\includegraphics{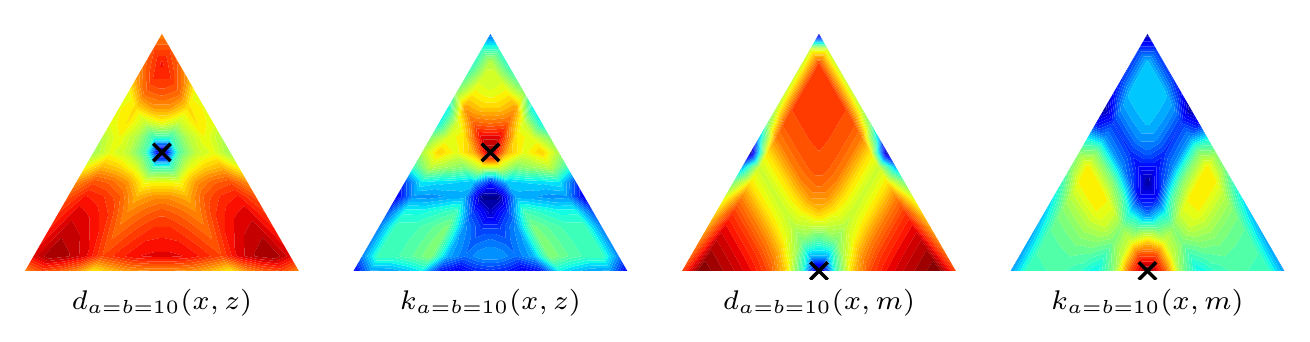}
\end{mybox}

\begin{mybox}{Generalized-JS metric \& kernel ($a = b \to \infty$)}{blue!20}
\begin{equation*}
    d^2(x,y) = \sum_{j=1}^p \max\{x^j,y^j\} \cdot \Big[\log(2)\mathds{1}{\{x^j \neq y^j\}}\Big]
\end{equation*}
\begin{equation*}
\begin{split}
    k(x,y) &= -\frac{\log(2)}{2}\cdot\sum_{j=1}^p \Bigg\{\max\{x^j,y^j\}\cdot\mathds{1}{\{x^j \neq y^j\}} \\ 
    &\qquad -\max\{x^j,\tfrac{1}{p}\}\cdot\mathds{1}{\{x^j \neq \tfrac{1}{p}\}} -\max\{y^j,\tfrac{1}{p}\}\cdot\mathds{1}{\{y^j \neq \tfrac{1}{p}\}}
    \Bigg\}
\end{split}
\end{equation*}
\centering\includegraphics{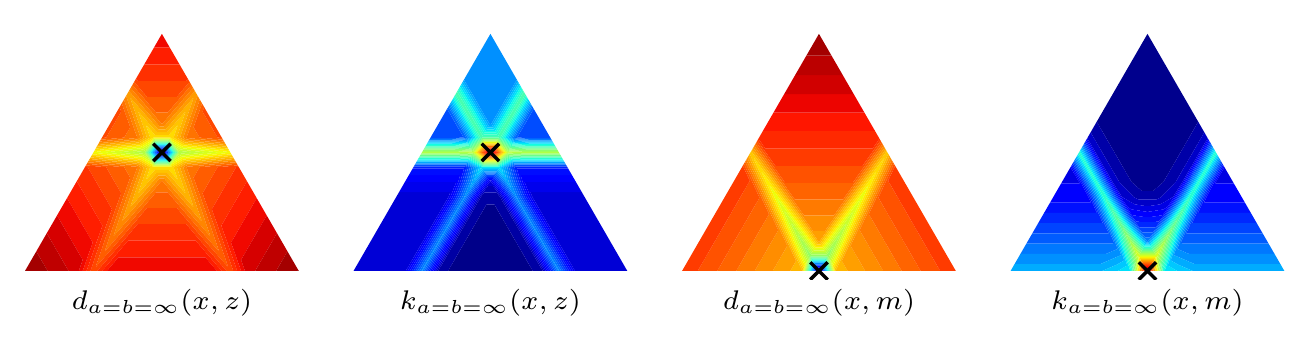}
\end{mybox}

\begin{mybox}{Special Case: Hellinger - Generalized-JS metric \& kernel ($a = 1, b = \frac{1}{2}$)}{blue!20}
\begin{equation*}
    d^2(x,y) 
    = \frac{\sqrt{2}}{2}\sum_{j=1}^p\big(\sqrt{x^j}-\sqrt{y^j}\big)^2
\end{equation*}
\begin{equation*}
    k(x,y) = \frac{\sqrt{2}}{4} + \frac{\sqrt{2}}{4}\sum_{j=1}^p \Bigg\{ \sqrt{x^j y^j}-\frac{\sqrt{x^j}+\sqrt{y^j}}{\sqrt{p}} \Bigg\}
\end{equation*}
This corresponds to $\frac{\sqrt{2}}{2}$ times the \textbf{Hellinger} metric and kernel.

\centering\includegraphics{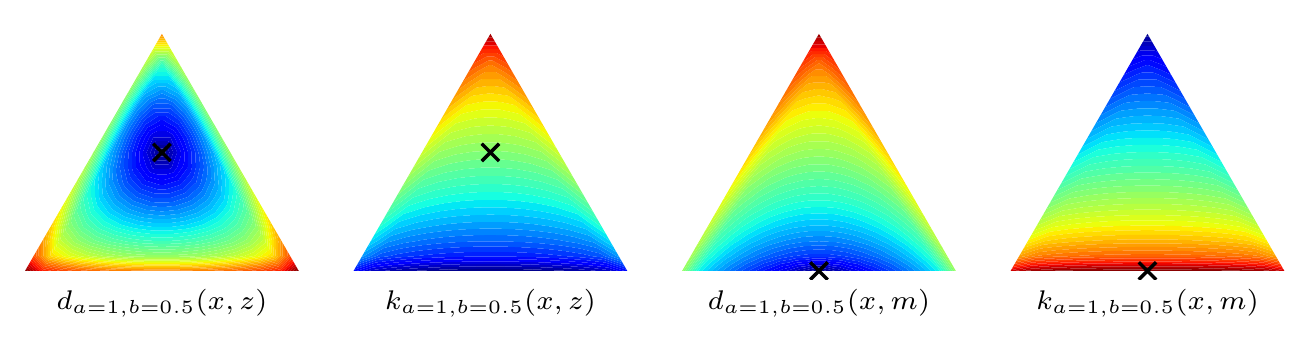}
\end{mybox}

\begin{mybox}{Special Case: Jenson-Shannon - Generalized-JS metric \& kernel ($a = 1, b = 1$)}{blue!20}
\begin{equation*}
\begin{split}
    d^2(x,y) 
    = \frac{1}{2}\sum_{j=1}^p x^j\log\frac{2x^j}{x^j+y^j} + y^j\log\frac{2y^j}{x^j+y^j}
\end{split}
\end{equation*}
\begin{equation*}
    k(x,y) = -\frac{1}{4}\sum_{j=1}^p \Bigg\{x^j\log\frac{x^j + \tfrac{1}{p}}{x^j+y^j} + y^j\log\frac{y^j + \tfrac{1}{p}}{x^j+y^j} - \tfrac{1}{p}\log\frac{4}{p^2(x^j + \tfrac{1}{p})(y^j+\tfrac{1}{p})}\Bigg\}
\end{equation*}
This corresponds to the \textbf{Jenson-Shannon} metric and kernel.

\centering\includegraphics{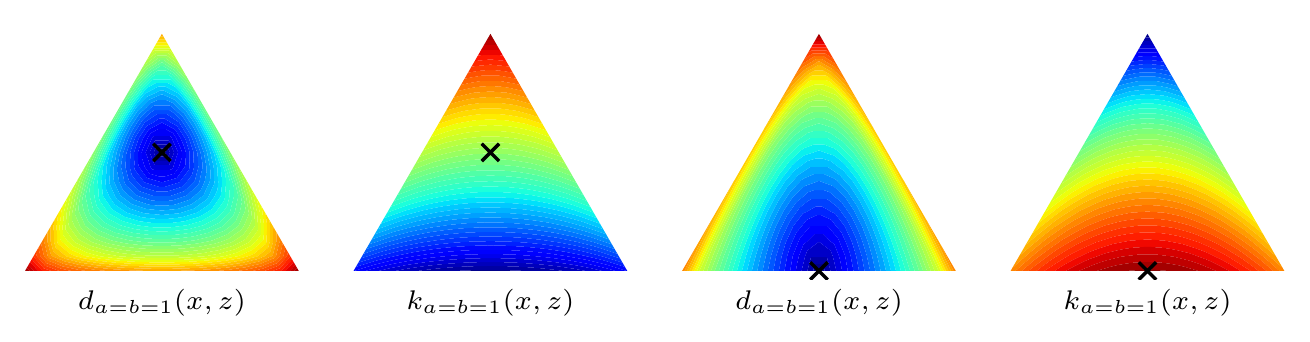}
\end{mybox}

\begin{mybox}{Special Case: Total Variation - Generalized-JS metric \& kernel ($a = \infty, b = 1$)}{blue!20}
\begin{equation*}
\begin{split}
    d^2(x,y) 
    = \sum_{j=1}^p |x^j-y^j|
\end{split}
\end{equation*}
\begin{equation*}
    k(x,y) = -\frac{1}{2}\sum_{j=1}^p \Bigg\{|x^j - y^j|-|x^j-\tfrac{1}{p}|-|y^j-\tfrac{1}{p}|\Bigg\}
\end{equation*}
This corresponds to $2$ times the \textbf{total variation} metric and kernel.

\centering\includegraphics{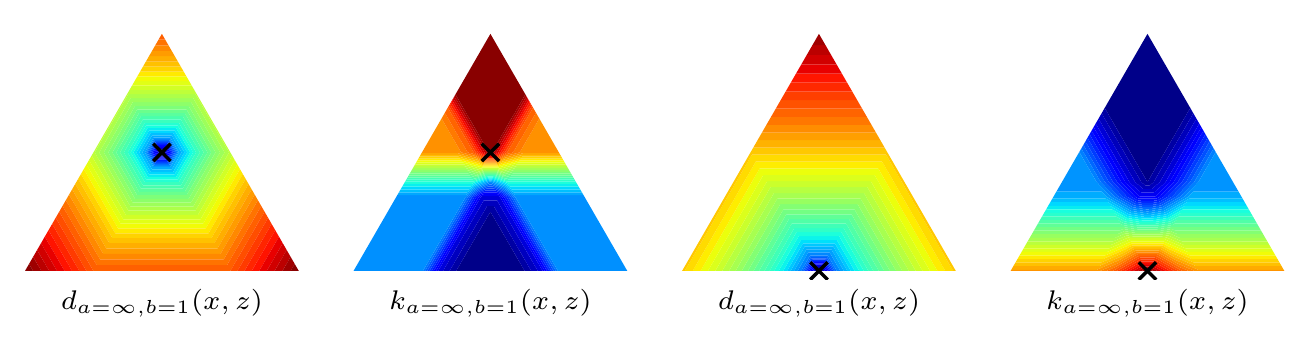}
\end{mybox}

\begin{mybox}{Hilbertian metric \& kernel ($a < \infty, b > -\infty$)}{blue!20}
\begin{equation*}
    d^2(x,y) = \sum_{j=1}^p\frac{2^{\frac{1}{b}}\Big[(x^j)^{a} +  (y^j)^{a}\Big]^{\frac{1}{a}} - 2^{\frac{1}{a}}\Big[(x^j)^{b} + (y^j)^{b}\Big]^{\frac{1}{b}}}{2^{\frac{1}{a}} - 2^{\frac{1}{b}}}
\end{equation*}
\begin{equation*}
\begin{split}
    k(x,y) &= -\frac{1}{2(2^{\frac{1}{a}} - 2^{\frac{1}{b}})}\sum_{j=1}^p \Bigg\{
        2^{\frac{1}{b}}\Big(\Big[(x^j)^{a}+(y^j)^{a}\Big]^{\frac{1}{a}} - \Big[(x^j)^{a}+(\tfrac{1}{p})^{a}\Big]^{\frac{1}{a}} \\ 
        &\qquad - \Big[(y^j)^{a}+(\tfrac{1}{p})^{a}\Big]^{\frac{1}{a}}\Big) - 2^{\frac{1}{a}}\Big(\Big[(x^j)^{b}+(y^j)^{b}\Big]^{\frac{1}{b}} - \Big[(x^j)^{b}+(\tfrac{1}{p})^{b}\Big]^{\frac{1}{b}} \\ 
        &\qquad - \Big[(y^j)^{b}+(\tfrac{1}{p})^{b}\Big]^{\frac{1}{b}}\Big)
    \Bigg\}
\end{split}
\end{equation*}

\centering\includegraphics{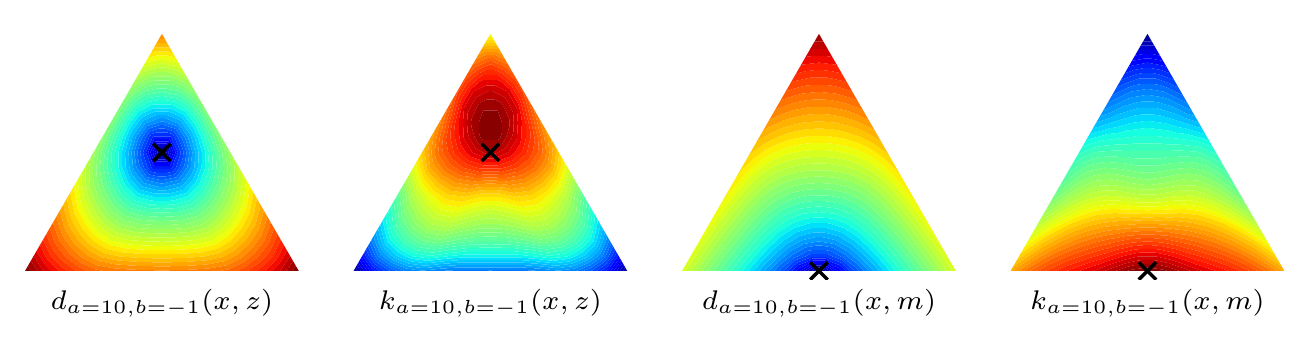}
\end{mybox}

\begin{mybox}{Hilbertian metric \& kernel ($a \to \infty, b > -\infty$)}{blue!20}
\begin{equation*}
    d^2(x,y) = \sum_{j=1}^p b\Bigg\{2^{\frac{1}{b}}\cdot\max\{x^j,y^j\} - \Big[(x^j)^{b} + (y^j)^{b}\Big]^{\frac{1}{b}}\Bigg\}
\end{equation*}
\begin{equation*}
\begin{split}
    k(x,y) &= -\frac{1}{2(1-2^{\frac{1}{b}})}\sum_{j=1}^p \bigg\{
    2^{\frac{1}{b}}\Big(\max\{x^j,y^j\}-\max\{x^j,\tfrac{1}{p}\}-\max\{y^j,\tfrac{1}{p}\}\Big) \\
    &\quad - \Big[(x^j)^{b} + (y^j)^{b}\Big]^{\frac{1}{b}}
    + \Big[(x^j)^{b} + (\tfrac{1}{p})^{b}\Big]^{\frac{1}{b}}
    + \Big[(y^j)^{b} + (\tfrac{1}{p})^{b}\Big]^{\frac{1}{\beta}} \bigg\}
\end{split}
\end{equation*}

\centering\includegraphics{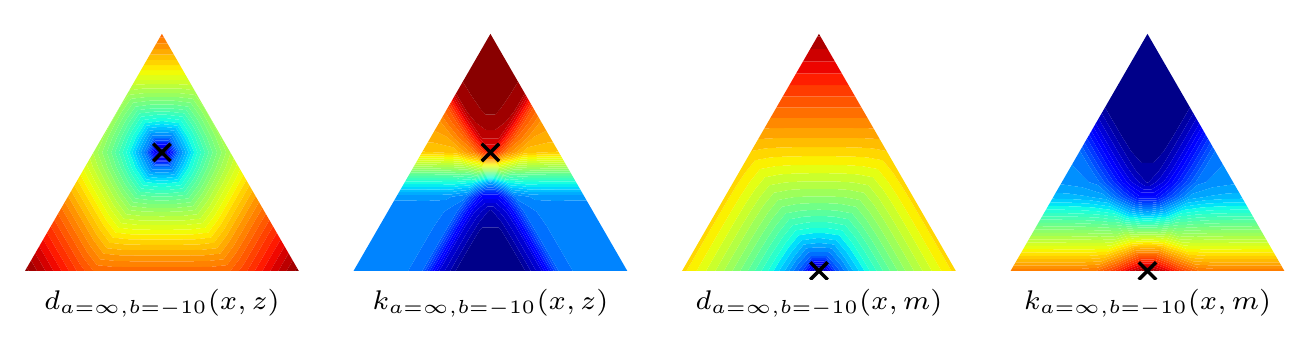}
\end{mybox}

\begin{mybox}{Hilbertian metric \& kernel ($a < \infty, b \to -\infty$)}{blue!20}
\begin{equation*}
    d^2(x,y) = \frac{1}{2^{\frac{1}{a}}-1}\sum_{j=1}^p \Bigg\{
    \Big[(x^j)^{a} + (y^j)^{a}\Big]^{\frac{1}{a}}
    - 2^{\frac{1}{a}}\cdot\min\{x^j,y^j\} 
    \Bigg\}
\end{equation*}
\begin{equation*}
\begin{split}
    k(x,y) &= -\frac{1}{2(2^{\frac{1}{a}}-1)}\sum_{j=1}^p \Bigg\{
        \Big[(x^j)^{a} + (y^j)^{a}\Big]^{\frac{1}{a}}
        - \Big[(x^j)^{a} + (\tfrac{1}{p})^{a}\Big]^{\frac{1}{a}}
        - \Big[(y^j)^{a} + (\tfrac{1}{p})^{a}\Big]^{\frac{1}{a}} \\
        &\qquad - 2^{\frac{1}{a}}\Big[\min\{x^j,y^j\} - \min\{x^j,\tfrac{1}{p}\}  - \min\{y^j,\tfrac{1}{p}\}\Big]
    \Bigg\}
\end{split}
\end{equation*}

\centering\includegraphics{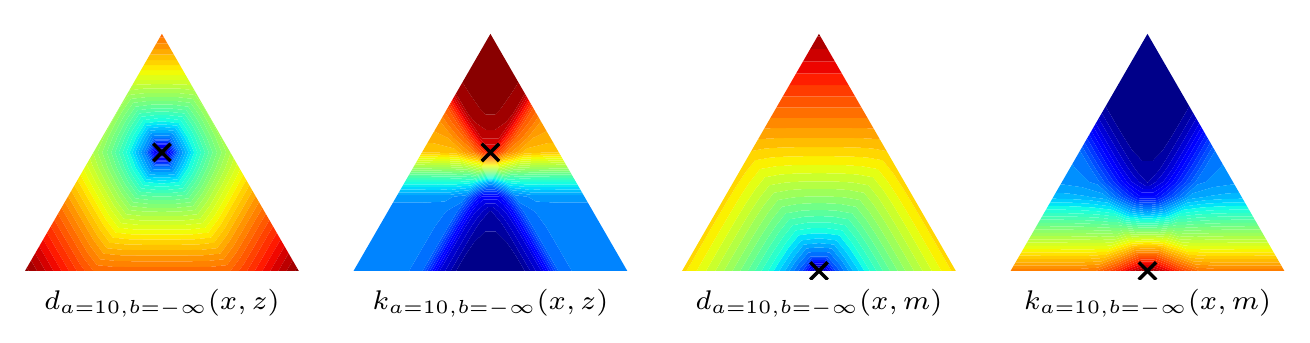}
\end{mybox}

\begin{mybox}{Special Case: Chi-square - Hilbertian metric \& kernel ($a = 1, b = -1$)}{blue!20}
\begin{equation*}
\begin{split}
    d^2(x,y) 
    = \frac{1}{3}\sum_{j=1}^p\frac{(x^j-y^j)^2}{x^j+y^j}
\end{split}
\end{equation*}
\begin{equation*}
    k(x,y) = -\frac{1}{6}\sum_{j=1}^p \Bigg\{\frac{(x^j-y^j)^2}{x^j+y^j} - \frac{(x^j-\tfrac{1}{p})^2}{x^j + \tfrac{1}{p}} - \frac{(y^j-\tfrac{1}{p})^2}{y^j + \tfrac{1}{p}}\Bigg\}
\end{equation*}
This corresponds to $\frac{1}{3}$ of the \textbf{chi-square} metric and kernel.

\centering\includegraphics{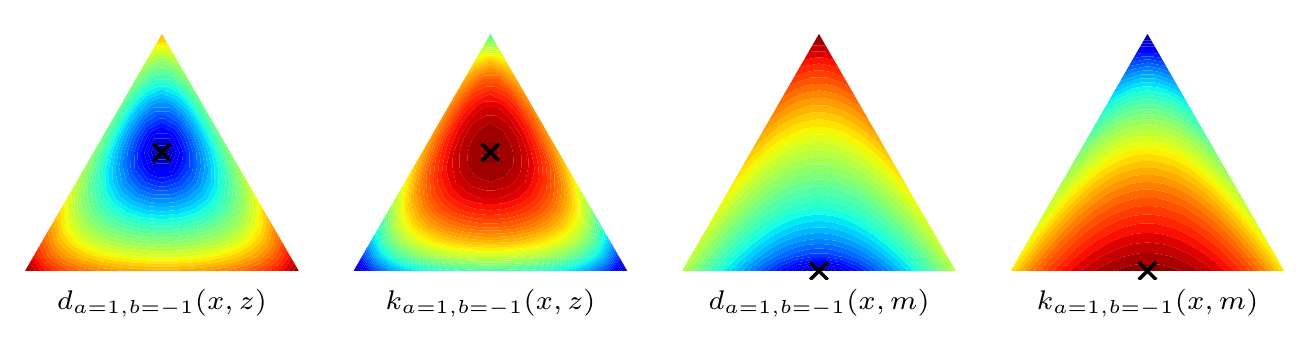}
\end{mybox}

\begin{mybox}{Special Case: Total Variation - Hilbertian metric \& kernel ($a = 1, b = -\infty$)}{blue!20}
\begin{equation*}
\begin{split}
    d^2(x,y) 
    = \frac{1}{2}\sum_{j=1}^p |x^j-y^j|
\end{split}
\end{equation*}
\begin{equation*}
    k(x,y) = -\frac{1}{4}\sum_{j=1}^p \Bigg\{|x^j - y^j|-|x^j-\tfrac{1}{p}|-|y^j-\tfrac{1}{p}|\Bigg\}
\end{equation*}
This corresponds to the \textbf{total variation} metric and kernel.

\centering\includegraphics{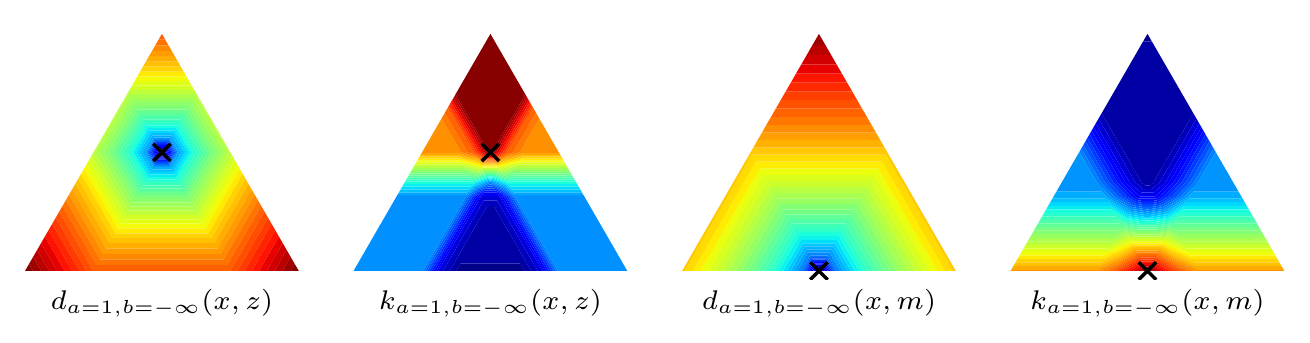}
\end{mybox}

\begin{mybox}{Aitchison metric \& kernel}{blue!20}
\begin{equation*}
\begin{split}
    d^2(x,y) 
    = \sum_{j=1}^p \Big(\log\frac{x^j+c}{g(x+c)} - \log\frac{y^j+c}{g(y+c)}\Big)^2
\end{split}
\end{equation*}
\begin{equation*}
\begin{split}
    k(x,y) = \sum_{j=1}^p\log\frac{x^j+c}{g(x+c)}\log\frac{y^j+c}{g(y+c)}
\end{split}
\end{equation*}
where $g(x) = \sqrt[p]{\prod_{j=1}^p} x^j$ is the geometric mean of $x$.

\centering\includegraphics{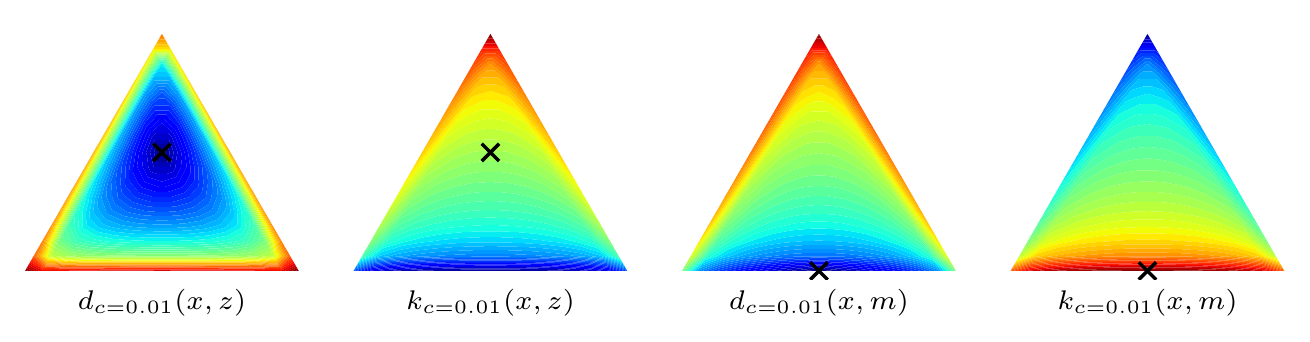}
\end{mybox}

\begin{mybox}{Aitchison-RBF metric \& kernel}{blue!20}
\begin{equation*}
    d^2(x,y) = 2 - 2\exp\Big(-\frac{1}{2\sigma^2}\sum_{j=1}^p\Big[\log\frac{x^j+c}{g(x+c)} - \log\frac{y^j+c}{g(y+c)}\Big]^2\Big)
\end{equation*}
\begin{equation*}
    k(x,y) = \exp\Big(-\frac{1}{2\sigma^2}\sum_{j=1}^p\Big[\log\frac{x^j+c}{g(x+c)} - \log\frac{y^j+c}{g(y+c)}\Big]^2\Big)
\end{equation*}
where $g(x) = \sqrt[p]{\prod_{j=1}^p} x^j$ is the geometric mean of $x$.

\centering\includegraphics{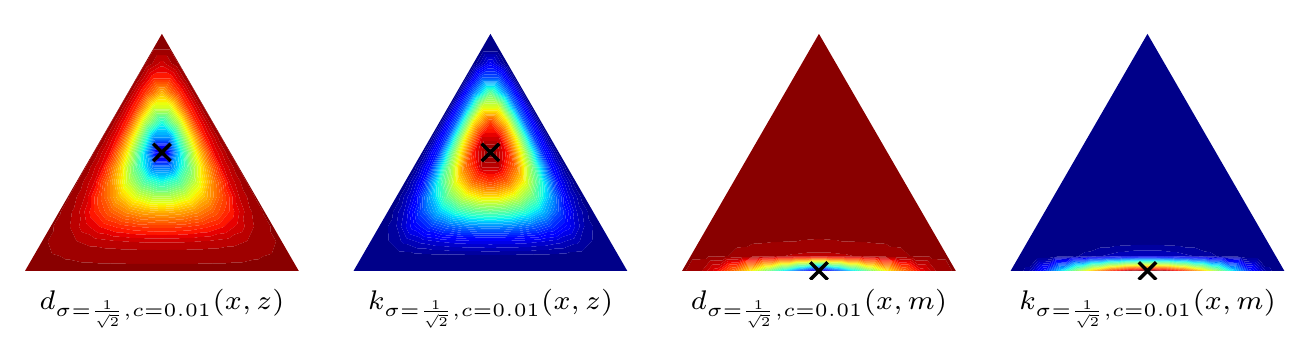}
\end{mybox}

\begin{mybox}{Heat diffusion metric \& kernel}{blue!20}
\begin{equation*}
    d^2(x,y) = 2\cdot(4\pi t)^{-\frac{p}{2}}\cdot \Big[1-\exp\Big(-\tfrac{1}{t}\arccos^2(\sum_{j=1}^p\sqrt{x^j y^j})\Big)\Big]
\end{equation*}
\begin{equation*}
    k(x,y) = (4\pi t)^{-p/2} \cdot\exp\Big(-\tfrac{1}{t}\arccos^2(\sum_{j=1}^p\sqrt{x^j y^j})\Big)
\end{equation*}

\centering\includegraphics{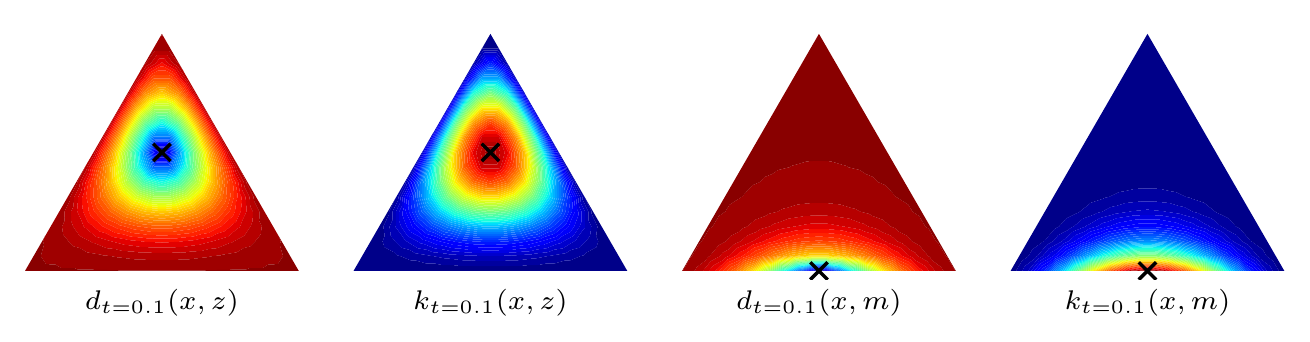}
\end{mybox}

\subsection{List of weighted kernels}\label{apd:list_weighted_kernels}

As discussed in Sec.~\ref{apd:summary_kernels_and_metrics}, all kernels can also be modified to include a weight matrix $W\in\bR^{p\times p}$. Below, we list the explicit forms of all weighted kernels as they are implemented in \KernelBiome package. As before, let $g(x) = \sqrt[p]{\prod_{j=1}^p} x^j$ be the geometric mean of $x$.

\noindent\textbf{Linear kernel}
\begin{equation*}
    k(x,y) = \sum_{j, \ell=1}^p W_{j, \ell} \Big (x^j y^\ell - \tfrac{x^j}{p}-\tfrac{y^\ell}{p} +\tfrac{1}{p^2}\Big)
\end{equation*}
\textbf{RBF kernel}
\begin{equation*}
    k(x,y) = \exp\Big(-\sum_{j, \ell=1}^p \frac{W_{j, \ell}(x^j-y^\ell)^2}{2\sigma^2}\Big)
\end{equation*}
\textbf{Generalized-JS kernel ($a< \infty, b \in [0.5, a]$)}
\begin{equation*}
\begin{split}
    k(x,y) &= -\frac{ab}{a-b}\cdot 2^{-(1+\frac{1}{a}+\frac{1}{b})} \sum_{j, \ell=1}^p W_{j, \ell}
    \Bigg\{2^\frac{1}{b} \Big(\Big[(x^j)^{a} + (y^\ell)^{a}\Big]^{\frac{1}{a}} - \Big[(x^j)^{a} + (\tfrac{1}{p})^{a}\Big]^{\frac{1}{a}} \\ 
    &\quad - \Big[(\tfrac{1}{p})^{a} + (y^\ell)^{a}\Big]^{\frac{1}{a}}\Big) -2^\frac{1}{a} \Big(\Big[(x^j)^{b} + (y^\ell)^{b}\Big]^{\frac{1}{b}} - \Big[(x^j)^{b} + (\tfrac{1}{p})^{b}\Big]^{\frac{1}{b}} \\ 
    &\quad - \Big[(\tfrac{1}{p})^{b} + (y^\ell)^{b}\Big]^{\frac{1}{b}}\Big)\Bigg\}
\end{split}
\end{equation*}
\textbf{Generalized-JS kernel ($a \to \infty, b < \infty$)}
\begin{equation*}
\begin{split}
    k(x,y) &= -\frac{b}{2}\sum_{j, \ell=1}^p W_{j, \ell} \Bigg\{
    2^{\frac{1}{b}}\Big(\max\{x^j,y^\ell\} - \max\{x^j,\tfrac{1}{p}\} - \max\{y^\ell,\tfrac{1}{p}\}\Big) \\
    &\quad - \Big[(x^j)^{b} + (y^\ell)^{b}\Big]^{\frac{1}{b}} 
    + \Big[(x^j)^{b} + (\tfrac{1}{p})^{b}\Big]^{\frac{1}{b}}
    + \Big[(y^\ell)^{b} + (\tfrac{1}{p})^{b}\Big]^{\frac{1}{b}}
    \Bigg\}
\end{split}
\end{equation*}
\textbf{Generalized-JS kernel ($a < \infty, b \to a$)}
\begin{equation*}
\begin{split}
    k(x,y) &= -\frac{1}{2^{\frac{1}{b}+1}}\sum_{j, \ell=1}^p W_{j, \ell} \Bigg\{\Big[(x^j)^{b}+(y^\ell)^{b}\Big]^{\frac{1}{b}-1} \cdot \Big((x^j)^{b}\cdot\log\Big[2(x^j)^{b}\Big] + (y^\ell)^{b}\log\Big[2(y^\ell)^{b}\Big] \\
    &\qquad\qquad - \Big[(x^j)^{b} + (y^\ell)^{b}\Big]\cdot\log\Big[(x^j)^{b} + (y^\ell)^{b}\Big]\Big)\\
    &\qquad - \Big[(x^j)^{b}+\tfrac{1}{p^{b}}\Big]^{\frac{1}{b}-1}\cdot
    \Big(-\Big[(x^j)^{b} + (\tfrac{1}{p^{b}})\Big]\cdot\log\Big[(x^j)^{b} + \tfrac{1}{p^{b}}\Big] \\ 
    &\qquad\qquad + (x^j)^{b}\cdot\log\Big[2(x^j)^{b}\Big] + \tfrac{1}{p^{b}}\cdot\log\Big[\tfrac{2}{p^{b}}\Big]\Big) \\
    &\qquad - \Big[(y^\ell)^{b}+\tfrac{1}{p^{b}}\Big]^{\frac{1}{b}-1}\cdot\Big(-\Big[(y^\ell)^{b} + (\tfrac{1}{p^{b}})\Big]\cdot\log\Big[(y^\ell)^{b} + \tfrac{1}{p^{b}}\Big] \\
    &\qquad\qquad + (y^\ell)^{b}\cdot\log\Big[2(y^\ell)^{b}\Big] + \tfrac{1}{p^{b}}\cdot\log\Big[\tfrac{2}{p^{b}}\Big]\Big)
    \Bigg\}
\end{split}
\end{equation*}
\textbf{Generalized-JS kernel ($a = b \to \infty$)}
\begin{equation*}
\begin{split}
    k(x,y) &= -\frac{\log(2)}{2}\cdot\sum_{j, \ell=1}^p W_{j, \ell} \Bigg\{\max\{x^j,y^\ell\}\cdot\mathds{1}{\{x^j \neq y^\ell\}} \\ 
    &\quad -\max\{x^j,\tfrac{1}{p}\}\cdot\mathds{1}{\{x^j \neq \tfrac{1}{p}\}} -\max\{y^\ell,\tfrac{1}{p}\}\cdot\mathds{1}{\{y^\ell \neq \tfrac{1}{p}\}}
    \Bigg\}
\end{split}
\end{equation*}
\textbf{Hilbertian kernel ($a < \infty, b > - \infty$)}
\begin{equation*}
\begin{split}
    k(x,y) &= -\frac{1}{2(2^{\frac{1}{a}} - 2^{\frac{1}{b}})}\sum_{j, \ell=1}^p W_{j, \ell} \Bigg\{
        2^{\frac{1}{b}}\Big(\Big[(x^j)^{a}+(y^\ell)^{a}\Big]^{\frac{1}{a}} - \Big[(x^j)^{a}+(\tfrac{1}{p})^{a}\Big]^{\frac{1}{a}} \\ 
        &\quad - \Big[(y^\ell)^{a}+(\tfrac{1}{p})^{a}\Big]^{\frac{1}{a}}\Big) - 2^{\frac{1}{a}}\Big(\Big[(x^j)^{b}+(y^\ell)^{b}\Big]^{\frac{1}{b}} - \Big[(x^j)^{b}+(\tfrac{1}{p})^{b}\Big]^{\frac{1}{b}} \\ 
        &\quad - \Big[(y^\ell)^{b}+(\tfrac{1}{p})^{b}\Big]^{\frac{1}{b}}\Big)
    \Bigg\}
\end{split}
\end{equation*}
\textbf{Hilbertian kernel ($a \to \infty, b > - \infty$)}
\begin{equation*}
\begin{split}
    k(x,y) &= -\frac{1}{2(1-2^{\frac{1}{b}})}\sum_{j, \ell=1}^p W_{j, \ell}\bigg\{
    2^{\frac{1}{b}}\Big(\max\{x^j,y^\ell\}-\max\{x^j,\tfrac{1}{p}\}-\max\{y^\ell,\tfrac{1}{p}\}\Big) \\
    &\quad - \Big[(x^j)^{b} + (y^\ell)^{b}\Big]^{\frac{1}{b}}
    + \Big[(x^j)^{b} + (\tfrac{1}{p})^{b}\Big]^{\frac{1}{b}}
    + \Big[(y^\ell)^{b} + (\tfrac{1}{p})^{b}\Big]^{\frac{1}{b}} \bigg\}
\end{split}
\end{equation*}
\textbf{Hilbertian kernel ($a < \infty, b \to - \infty$)}
\begin{equation*}
\begin{split}
    k(x,y) &= -\frac{1}{2(2^{\frac{1}{a}}-1)}\sum_{j, \ell=1}^p W_{j, \ell}\Bigg\{
        \Big[(x^j)^{a} + (y^\ell)^{a}\Big]^{\frac{1}{a}}
        - \Big[(x^j)^{a} + (\tfrac{1}{p})^{a}\Big]^{\frac{1}{a}}
        - \Big[(y^\ell)^{a} + (\tfrac{1}{p})^{a}\Big]^{\frac{1}{a}} \\
        &\quad - 2^{\frac{1}{a}}\Big[\min\{x^j,y^\ell\} - \min\{x^j,\tfrac{1}{p}\}  - \min\{y^\ell,\tfrac{1}{p}\}\Big]
    \Bigg\}
\end{split}
\end{equation*}
\textbf{Aitchison kernel}
\begin{equation*}
    k(x,y) = \sum_{j, \ell=1}^p W_{j, \ell}\log\frac{x^j+c}{g(x+c)}\log\frac{y^\ell+c}{g(y+c)}
\end{equation*}
\textbf{Aitchison RBF kernel}
\begin{equation*}
    k(x,y) = \exp\Big(-\frac{1}{2\sigma^2}\sum_{j, \ell=1}^p W_{j, \ell}\Big[\log\frac{x^j+c}{g(x+c)} - \log\frac{y^\ell+c}{g(y+c)}\Big]^2\Big)
\end{equation*}
\textbf{Heat diffusion kernel}
\begin{equation*}
    k(x,y) = (4\pi t)^{-p/2} \cdot\exp\bigg(-\tfrac{1}{t}\arccos^2\Big(\sum_{j, \ell=1}^p W_{j, \ell}\sqrt{x^j y^\ell}\Big)\bigg)
\end{equation*}
\end{appendices}

\clearpage

\vskip 0.2in
\bibliographystyle{plainnat}
\bibliography{bibliography} 

\end{document}